%% file: manuscript.tex
\documentclass[11pt,letterpaper]{article}
\usepackage{preprint}
\usepackage[round,authoryear]{natbib}

\usepackage{amsmath,amssymb,amsthm}
\usepackage{graphicx,booktabs,longtable,tabularx,array,float,algpseudocode,placeins}
\usepackage{xcolor}
\usepackage{microtype}
\usepackage[colorlinks=true,linkcolor=blue,citecolor=blue,urlcolor=blue]{hyperref}
\newtheorem{theorem}{Theorem}
\newtheorem{lemma}{Lemma}
\newtheorem{proposition}{Proposition}
\newtheorem{appendixproposition}[proposition]{Proposition}
\theoremstyle{definition}
\newtheorem{definition}{Definition}
\newtheorem{assumption}{Assumption}
\newcommand{\R}{\mathbb R}
\newcommand{\E}{\mathbb E}
\newcommand{\Prb}{\mathbb P}
\newcommand{\Var}{\operatorname{Var}}
\newcommand{\Cov}{\operatorname{Cov}}

\newcommand{\vecop}{\operatorname{vec}}
\newcommand{\indep}{\mathrel{\perp\!\!\!\perp}}
\newcommand{\Sset}{\mathcal S}
\newcommand{\Shat}{\widehat{\mathcal S}}
\newcommand{\xstar}{\boldsymbol x^\star}
\newcommand{\xobs}{\boldsymbol x}
\newcommand{\xpilot}{\widehat{\boldsymbol x}}
\newcommand{\Wbar}{\bar W^\star}
\newcommand{\ybar}{\bar{\boldsymbol y}^\star}
\newcommand{\ytest}{\bar{\boldsymbol y}^{\mathrm{test}}}
\newcommand{\Wt}{\widetilde W}
\newcommand{\Zt}{\widetilde Z}
\newcommand{\xistar}{\boldsymbol\xi^\star}
\newcommand{\Mt}{\widetilde M_N}
\newcommand{\zp}{\widetilde z^\perp_{j,\Sset}}
\newcommand{\wperp}{\widetilde w^\perp_{j,\Sset}}
\newcommand{\cc}{c^\star_{j,\Sset}}
\newcommand{\chat}{\widehat c_{j,\Sset}}
\newcommand{\bb}{\xi^\star_{j,\Sset}}
\newcommand{\ee}{e_{j,\Sset}}
\newcommand{\ttt}{\widehat t_{j,\Sset}}
\newcommand{\sss}{\widehat s_{j,\Sset}}
\newcommand{\boot}{\mathrm{boot}}
\newcommand{\Bboot}{B_{\mathrm{boot}}}
\newcommand{\Wstar}{W^\star}
\newcommand{\ystar}{\boldsymbol y^\star}
\newcommand{\xihat}{\boldsymbol\xi}
\newcommand{\Nset}{\Sset\setminus\{j\}}
\newcommand{\Mstar}{M^\star_{\Nset}}

\floatstyle{ruled}
\newfloat{algorithm}{tbp}{loa}
\floatname{algorithm}{Algorithm}

\allowdisplaybreaks[1]
\title{PSI-SINDy: Post-Selection Inference for Sparse Identification of Nonlinear Dynamics}
\author{Ashraful Islam, Shuichi Nishino, Tomohiro Shiraishi, Ichiro Takeuchi}
\hypersetup{pdftitle={PSI-SINDy: Post-Selection Inference for Sparse Identification of Nonlinear Dynamics},pdfauthor={Ashraful Islam, Shuichi Nishino, Tomohiro Shiraishi, Ichiro Takeuchi},pdfsubject={Post-selection inference for sparse identification of nonlinear dynamics}}

\begin{document}
\begin{center}
{\Large\bfseries PSI-SINDy: Post-Selection Inference for Sparse Identification of Nonlinear Dynamics\par}
\vspace{12pt}
{\normalsize Ashraful Islam\textsuperscript{1}, Shuichi Nishino\textsuperscript{2},
Tomohiro Shiraishi\textsuperscript{2}, Ichiro Takeuchi\textsuperscript{2,$\dagger$}\par}
\vspace{6pt}
{\small\textsuperscript{1}RIKEN\qquad\textsuperscript{2}Nagoya University\par}
\vspace{12pt}
\end{center}
\begingroup
\renewcommand{\thefootnote}{\fnsymbol{footnote}}
\footnotetext[2]{Corresponding author: Ichiro Takeuchi. Email: \href{mailto:takeuchi.ichiro.n6@f.mail.nagoya-u.ac.jp}{\nolinkurl{takeuchi.ichiro.n6@f.mail.nagoya-u.ac.jp}}.}
\endgroup

\begin{abstract}
\input{abstract}
\end{abstract}

\noindent\textbf{Keywords:} Sparse identification of nonlinear dynamics; post-selection inference; dynamical systems; measurement error; data thinning.

\input{sec1}
\input{sec2}
\input{sec3}
\input{sec4}
\input{sec5}
\input{app}

\end{document}

%% file: abstract.tex
Sparse identification of nonlinear dynamics (SINDy) is a data-driven framework for discovering governing dynamics from time-series data by identifying a sparse subset of candidate dynamical terms from a prespecified library.
In this work, we develop a statistical inference framework for quantifying the reliability of dynamical terms selected by SINDy through hypothesis tests and confidence intervals.
A key difficulty is that using the same noisy trajectory for both selecting dynamical terms and assessing their statistical significance can introduce selection bias.
Post-selection inference provides a principled framework for addressing such bias, and we propose PSI-SINDy, a post-selection inference method tailored to SINDy.
Direct application of existing post-selection inference techniques is challenging because SINDy involves measurement error in the candidate terms and shared noise between the response and design.
To address these challenges, PSI-SINDy uses data thinning to decompose a single observed trajectory into four mutually independent views with distinct roles in selection and inference.
This construction enables inference for selected dynamical terms while accounting not only for selection bias but also for measurement-error and shared noise effects.
We establish the theoretical validity of PSI-SINDy under stated conditions and evaluate its performance through numerical experiments on simulated and experimental dynamical-system data.

%% file: sec1.tex
\section{Introduction}
\label{sec1:introduction}

Sparse identification of nonlinear dynamics (SINDy) \citep{brunton2016} discovers dynamical systems from time-series observations by selecting sparse combinations of terms from a prespecified library.
The resulting models provide interpretable descriptions of the dynamics and identify terms relevant to the observed system.

To assess the statistical reliability of terms selected from noisy data, a natural goal is to construct hypothesis tests and confidence intervals.
However, reusing the trajectory for inference while treating the selected model as fixed can introduce selection bias \citep{taylor2015,berk2013}.
This is a form of data double dipping \citep{kriegeskorte2009}: the same data determine the hypotheses and provide the evidence used to test them.

Post-selection inference (PSI) provides a framework for inference after data-dependent model selection \citep{lee2016,tibshirani2018,tian2018}.
We propose \emph{PSI-SINDy} to construct hypothesis tests and confidence intervals for selected dynamical terms while accounting for their selection from a noisy trajectory.

Applying existing PSI methods to SINDy requires additional care.
Unlike regression with an error-free design, SINDy's candidate terms are evaluated from noisy states, introducing measurement error into the design.
The response and design also share noise from the same trajectory.
These classical inference problems \citep{fuller1987,wendy2023} must therefore be addressed jointly with data-dependent selection.

The key idea is to use \emph{data thinning} \citep{neufeld2024} to allocate information from one observed trajectory across four mutually independent views, $V_1$, $V_2$, $V_3$, and $V_4$.
We use $V_1$ for term selection, $V_2$ for constructing the testing response, $V_3$ for constructing a corrected design, and $V_4$ for constructing an instrument, an independent corrected estimate of the same noise-free design (Sections~\ref{sec:problemsetup} and~\ref{sec3:method}).
Their independence separates selection from inference and the testing response from the noisy design and instrument.

We consider two related inferential targets.
Our primary target asks what coefficient each term would have in the best linear approximation within the selected model if the noise-free trajectory were available.
Because that trajectory is unobserved, we introduce an auxiliary coefficient using the corrected design and instrument from independent views.
Under the assumptions in Section~\ref{sec3:method}, this auxiliary target admits finite-sample-valid conditional tests and confidence intervals.
The targets generally differ because the corrected dynamical terms still contain random measurement-error fluctuations.
An asymptotic characterization of this difference enables asymptotically valid tests and confidence intervals for the primary target.
Section~\ref{sec:problemsetup} defines both targets and the weak-form SINDy formulation.

Numerical experiments evaluate Type~I error, con\-fi\-dence-interval coverage and width in finite samples, and the effect of accounting for uncertainty in the difference between the auxiliary and noise-free targets.
We consider simulated and experimental settings, including larger candidate libraries and multistate dynamical systems.

\paragraph{Contributions.}
\label{sec1:contributions}

\begin{itemize}

\item We propose PSI-SINDy, a post-selection inference framework for quantifying the statistical reliability of dynamical terms selected by SINDy from a noisy trajectory.

\item We develop a four-view data-thinning construction that separates selection, testing response, corrected design, and instrument, thereby addressing selection bias, shared response--design noise, and measurement error in the design.

\item We establish finite-sample-valid inference for an auxiliary realized target and asymptotically valid inference for the primary noise-free target through a theoretical bridge between the two, and evaluate the resulting procedure in numerical experiments.

\end{itemize}

\paragraph{Related Work.}
\label{sec1:related}

SINDy discovers ordinary differential equations and extends to partial differential equations \citep{brunton2016,rudy2017}.
Weak-form variants use integral relations to avoid numerical differentiation in both settings \citep{schaeffer2017,messenger2021,messengerpde2021}.
For differential equations whose form is specified in advance, orthogonality-condition methods estimate unknown parameters and construct confidence intervals using large-sample approximations \citep{brunel2014}.
They first estimate a smooth trajectory nonparametrically from noisy observations and then fit the equation through integrated residual conditions.
Weak-form estimation of nonlinear dynamics (WENDy) jointly accounts for response and design noise for a prescribed differential equation \citep{wendy2023}; its numerical calibration has been studied across noise distributions and levels \citep{chawla2026}.
Continuum-data analysis establishes weak-form SINDy (WSINDy) consistency under specified library, noise, and filtering conditions \citep{consistency2025}.

Uncertainty quantification in SINDy takes several directions.
UQ-SINDy uses sparse Bayesian inference to quantify posterior coefficient uncertainty and term-inclusion probabilities \citep{hirsh2022}.
Ensemble-SINDy uses bootstrap aggregation to assess term inclusion, coefficient variability, and forecast uncertainty \citep{fasel2022}; \citet{gao2023} provide bootstrap guarantees in a sparse linear-regression setting with support-recovery conditions.
Conformal methods combined with Ensemble-SINDy address time-series prediction, library-term importance, and coefficient uncertainty \citep{fasel2025}.
These methods do not establish hypothesis tests and confidence intervals conditional on the selected model while jointly accounting for design measurement error and shared response--design noise.
PSI-SINDy addresses this joint problem for the selected noise-free approximation without requiring correct support recovery.

Post-selection inference provides valid inference for data-dependent hypotheses.
Early work addressed linear-regression selection through simultaneous protection over submodels \citep{berk2013} and conditional inference after lasso and related procedures \citep{lee2016,tibshirani2018}.
Subsequent work has improved statistical power and computational flexibility; for example, parametric programming reduces unnecessary conditioning for the generalized lasso \citep{le2022more}.
Conditional inference also covers k-means clustering \citep{chen2023clustering} and regions selected by Transformer attention \citep{shiraishi2026statistical}.
Randomized selective inference introduces additional randomization into selection \citep{tian2018}.

A complementary approach separates information for selection and inference.
Classical sample splitting uses disjoint observations, leaving less information for each task \citep{wasserman2009,rinaldo2019}.
Gaussian randomized splitting \citep{rasines2023} and data fission \citep{leiner2025} construct statistically tractable components from the same data.
Data thinning decomposes observations from suitable distributions into mutually independent views \citep{neufeld2024}.

%% file: sec2.tex
\renewcommand{\Mt}{\widetilde M_{\Nset}}
\section{Problem Setup}
\label{sec:problemsetup}

\noindent
We formulate the inference problem by specifying the observation model and weak-form SINDy setup, identifying the statistical challenges, and defining the coefficient targets and hypotheses.

\subsection{Weak-form SINDy and model selection}
\label{sec:weaksetup}

\paragraph{Noisy time-series observations and underlying dynamics.}
We observe a scalar time series $x_1,\ldots,x_M$ at $M\ge2$ uniformly spaced times $t_1<\cdots<t_M$, with $\Delta t=(t_M-t_1)/(M-1)$. For a continuously differentiable noise-free trajectory $x^\star(t)$ on $[t_1,t_M]$, set $x_m^\star=x^\star(t_m)$ and treat $\xstar=(x_1^\star,\ldots,x_M^\star)^\top$ as fixed. We assume
\begin{equation}
x_m=x_m^\star+\varepsilon_m,\qquad
\varepsilon_m\overset{\mathrm{iid}}{\sim}N(0,\sigma_X^2),
\quad m\in[M],
\label{eq:measurement}
\end{equation}
where $\sigma_X^2>0$ and $[n]=\{1,\ldots,n\}$ for a positive integer $n$. Throughout our theoretical development, we treat $\sigma_X^2$ as known; in practice, it may be estimated from an independent data set.
The underlying dynamics satisfy
\begin{equation}
\dot x^\star(t)=f\bigl(x^\star(t)\bigr),
\label{eq:dynamics}
\end{equation}
where $f:\R\to\R$ is unknown and the dot denotes time differentiation.

\paragraph{Weak-form SINDy.}
Given prespecified candidate functions $\theta_j:\R\to\R$, $j\in[P]$, SINDy \citep{brunton2016} seeks a sparse vector $\xihat=(\xi_1,\ldots,\xi_P)^\top\in\R^P$ such that
\begin{equation}
f(x)\approx\sum_{j\in[P]}\xi_j\theta_j(x).
\label{eq:libraryapprox}
\end{equation}
For example, $\theta_j(x)=x^{j-1}$ with $P=4$ gives $(1,x,x^2,x^3)$. To avoid amplifying noise through numerical differentiation, weak-form SINDy integrates the dynamics against prespecified smooth weights $\psi_1,\ldots,\psi_K$, supported inside $(t_1,t_M)$ and vanishing at their support endpoints \citep{schaeffer2017,messenger2021,wendy2023}. The noise-free weak response $\ystar\in\R^K$ and design $\Wstar\in\R^{K\times P}$ are
\begin{equation}
y_k^\star =-\sum_{m=1}^M\psi'_k(t_m)x_m^\star\,\Delta t, k\in[K],\qquad W_{kj}^\star =\sum_{m=1}^M\psi_k(t_m)\theta_j(x_m^\star)\,\Delta t, k\in[K],\ j\in[P].
\label{eq:weakgrid}
\end{equation}
At the finite-grid level, SINDy seeks a sparse $\xihat$ with $\ystar\approx\Wstar\xihat$; Appendix~\ref{app:finitegrid} gives the derivation and distinguishes model approximation from discretization error.

\paragraph{Feature selection.}
Because $\ystar$ and $\Wstar$ are unobserved, basic SINDy substitutes $x_m$ for $x_m^\star$ in \eqref{eq:weakgrid}, obtaining $\widehat{\boldsymbol y}\in\R^K$ and $\widehat W\in\R^{K\times P}$. Sparse regression, such as the lasso or sequentially thresholded least squares, selects candidate functions $\Shat\subseteq[P]$. In PSI-SINDy, selection and its data-dependent hyperparameter choices use only the selection view constructed in Section~\ref{sec3:construction}. Subsection~\ref{sec:targets} defines each selected coefficient after adjusting for the other selected terms.

\subsection{Statistical challenges for inference after SINDy}
\label{sec:challenges}

Inference for the selected coefficients must address the following three challenges jointly.

\begin{enumerate}
\item \textbf{Data-dependent model selection.}
The coefficient targets depend on the selected set $\Shat$. Reusing the selection data for inference while treating this set as fixed can invalidate tests and confidence intervals \citep{taylor2015}.

\item \textbf{Dependence between response and design errors.}
The empirical response $\widehat{\boldsymbol y}$ and design $\widehat W$ are constructed from the same noisy observations. Their errors are therefore generally dependent, which must be accounted for in inference \citep{wendy2023}.

\item \textbf{Measurement error in the design.}
Replacing the unobserved noise-free design $\Wstar$ with $\widehat W$ creates an errors-in-variables (EIV) problem \citep{fuller1987}. Nonlinear features can also introduce mean bias; for example, $\E[x_m^2\mid\xstar]=(x_m^\star)^2+\sigma_X^2$. Inference must account for both this bias and the variability of the noisy design.
\end{enumerate}

\subsection{Inferential targets and hypotheses}
\label{sec:targets}

For a realization $\Sset$ of $\Shat$, we consider two inferential targets. The first, primary target is the best linear approximation (BLA) coefficient in the noise-free weak regression restricted to $\Sset$. The second is an auxiliary realized instrumental-variable (IV) coefficient admitting exact conditional tests and confidence intervals under Section~\ref{sec3:method}'s assumptions. These targets generally differ in finite samples. Their relationship provides the bridge from exact inference for the second target to asymptotically valid inference for the first.

\paragraph{Selected-model best linear approximation (BLA) coefficient.}
Let $\Sigma_{\mathrm{test}}\in\R^{K\times K}$ be the covariance of the unwhitened testing response constructed in Section~\ref{sec3:construction}. Shared measurement errors can correlate its components. We assume $\Sigma_{\mathrm{test}}$ is positive definite and define the symmetric whitening matrix $Q=\Sigma_{\mathrm{test}}^{-1/2}$, $\ybar=Q\ystar$ and $\Wbar=Q\Wstar$. A design subscript $\Sset$ denotes its columns indexed by $\Sset$.

\begin{definition}[BLA coefficient for the selected model]
Assume $(W^\star_\Sset)^\top\Sigma_{\mathrm{test}}^{-1}W^\star_\Sset$ is invertible. The covariance-weighted selected-model BLA coefficient is
\begin{equation}
\xistar_\Sset
=\underset{\boldsymbol\beta\in\R^{|\Sset|}}{\arg\min}
\|\ybar-\bar W^\star_\Sset\boldsymbol\beta\|_2^2,
\label{eq:selectedbla}
\end{equation}
where $|\Sset|$ is the number of selected features and $\|\cdot\|_2$ is the Euclidean norm.
\end{definition}
The target allows model misspecification and generally changes with $\Sset$ \citep{berk2013,lee2016}. Write $I_n$ for the $n\times n$ identity matrix, $A^+$ for the Moore--Penrose pseudoinverse, and $A_{:,j}$ for column $j$ of $A$. For $j\in\Sset$, retaining its original library label, adjust for the other selected terms using
\[
\Mstar=I_K-\bar W^\star_{\Nset}
\bigl(\bar W^\star_{\Nset}\bigr)^+,
\]
with $\Mstar=I_K$ if $\Nset$ is empty. Let $y^{\star\perp}_{j,\Sset}=\Mstar\ybar$ and $w^{\star\perp}_{j,\Sset}=\Mstar\bar W^\star_{:,j}$. The Frisch--Waugh--Lovell representation \citep{lovell1963} is
\begin{equation}
\xi^\star_{j,\Sset}
=\frac{(w^{\star\perp}_{j,\Sset})^\top y^{\star\perp}_{j,\Sset}}
{(w^{\star\perp}_{j,\Sset})^\top w^{\star\perp}_{j,\Sset}}.
\label{eq:selectedbla-fwl}
\end{equation}
Definition~1's rank assumption ensures a positive denominator; Appendix~\ref{full:app:normal} gives the derivation.

\paragraph{Realized instrumental-variable (IV) coefficient.}
We introduce the second, auxiliary target to obtain exact conditional inference and an asymptotic connection to the primary BLA. The noise-free design is unobserved, so we construct corrected quantities from independent views. Gaussian correction removes polynomial feature-mean bias \citep{stefanski1989}, but leaves random design error. We therefore pair a corrected design $\Wt\in\R^{K\times P}$ with an independent corrected estimate $\Zt\in\R^{K\times P}$ of the same noise-free design, used as an instrument. In whitened coordinates,
\begin{equation}
\E[\Wt\mid\xstar]=\E[\Zt\mid\xstar]=\Wbar,
\qquad \Wt\indep\Zt\mid\xstar.
\label{eq:independentdesigns}
\end{equation}
Before residualization, their expected cross-product is the whitened noise-free Gram matrix: $\E[\Zt^\top\Wt\mid\xstar]=(\Wbar)^\top\Wbar$. Section~\ref{sec3:construction} constructs these matrices for fixed-degree polynomial libraries under the Gaussian measurement model.

To adjust term $j$ for the other selected terms, we apply the residualization idea in \eqref{eq:selectedbla-fwl} to the corrected design. For $j\in\Sset$, set
\[
\Mt=I_K-\widetilde W_{\Nset}
\bigl(\widetilde W_{\Nset}\bigr)^+,
\]
with $\Mt=I_K$ if $\Nset$ is empty. Set $\widetilde y^\perp_{j,\Sset}=\Mt\ybar$, $\widetilde w^\perp_{j,\Sset}=\Mt\widetilde W_{:,j}$, and $\widetilde z^\perp_{j,\Sset}=\Mt\widetilde Z_{:,j}$. The residualized response uses the noise-free mean and is random through $\Mt$. The same estimated matrix residualizes all three quantities; independence in \eqref{eq:independentdesigns} applies before residualization.

\begin{definition}[Realized IV coefficient for the selected model]
\leavevmode\par\nobreak\noindent
Suppose $(\widetilde z^\perp_{j,\Sset})^\top\widetilde w^\perp_{j,\Sset}\ne\nobreak0$. The realized IV coefficient is
\begin{equation}
c^\star_{j,\Sset}
=\frac{(\widetilde z^\perp_{j,\Sset})^\top\widetilde y^\perp_{j,\Sset}}
{(\widetilde z^\perp_{j,\Sset})^\top\widetilde w^\perp_{j,\Sset}},
\label{eq:realizediv}
\end{equation}
the unique $c\in\R$ satisfying
$(\widetilde z^\perp_{j,\Sset})^\top(\widetilde y^\perp_{j,\Sset}-c\widetilde w^\perp_{j,\Sset})=0$.
\end{definition}
Equations~\eqref{eq:selectedbla-fwl} and~\eqref{eq:realizediv} have parallel forms, but use different residualizers; the IV coefficient also uses an instrument in the inner products. Even for fixed $\Sset$, the IV target is random through the corrected design and instrument. Conditioning on the selected set and both matrices fixes the target for exact inference.

\paragraph{Post-selection hypotheses.}
For each $j\in\Shat$, whenever the corresponding target is defined, we consider
\begin{equation}
H^{\mathrm{IV}}_{0,j}(c_0):c^\star_{j,\Shat}=c_0,
\qquad
H^{\mathrm{BLA}}_{0,j}(c_0):\xi^\star_{j,\Shat}=c_0,
\label{eq:hypotheses}
\end{equation}
where $c_0\in\R$ is specified. Setting $c_0=0$ tests whether the relevant coefficient is zero within the selected model. A zero selected-BLA coefficient does not by itself establish absence of the term from the underlying governing equation.
\renewcommand{\Mt}{\widetilde M_N}

%% file: sec3.tex
\section{PSI-SINDy}
\label{sec3:method}

The key idea of PSI-SINDy is to use data thinning to decompose a single noisy trajectory into four mutually independent views with distinct roles in selection and inference. This construction first enables finite-sample exact conditional inference for the second, auxiliary IV target. We then establish an asymptotic bridge from this auxiliary target to the first, primary selected-BLA target, leading to asymptotically valid inference for the latter.

\subsection{Four-view Gaussian data thinning}
\label{sec3:construction}

Let $V_1$, $V_2$, $V_3$, and $V_4$ denote the four independent views obtained by Gaussian data thinning \citep{neufeld2024}. Prespecified fractions $a_i>0$ with $\sum_{i=1}^4a_i=1$ give $V_i\sim N(\xstar,s_i^2I_M)$, where $s_i^2=\sigma_X^2/a_i$. The four views are assigned distinct roles so that selection and inference use independent information. Model selection and data-dependent tuning of the selector use $V_1$ alone. Any auxiliary selection randomization is independent of all four views. Let $\mathcal A$ denote the sigma-field generated by $V_1$ and this randomization.

We construct the weak testing response $\ytest$ from $V_2$ in the whitened coordinates of Section~2. Applying Gaussian measurement-error correction \citep{stefanski1989} to the polynomial features evaluated on $V_3$ and $V_4$ yields the corrected design $\Wt$ and instrument $\Zt$, respectively, satisfying \eqref{eq:independentdesigns}. The independence of the four views is crucial for inference. Let $\mathcal H=\sigma(\mathcal A,\Wt,\Zt)$ denote the information from selection, the corrected design, and the instrument. Conditioning on $\mathcal H$ fixes these quantities while leaving the testing response with independent Gaussian noise:
\begin{equation}
\ytest=\ybar+\varepsilon^{\mathrm{test}},~\varepsilon^{\mathrm{test}}\sim N(0,I_K),~\varepsilon^{\mathrm{test}}\indep\mathcal H.
\label{sec3:eq:gaussian}
\end{equation}
This conditional Gaussian property enables the finite-sample exact inference developed in Section~\ref{sec3:exact}, while the remaining random errors in the corrected design and instrument enter the IV--BLA gap studied in Section~\ref{sec3:bridge}. Appendix~\ref{sec3:app:construction} gives the thinning construction, weak maps, and correction formulas.

\subsection{Exact conditional inference for the realized IV coefficient}
\label{sec3:exact}

The conditional Gaussian property in \eqref{sec3:eq:gaussian} enables exact inference for the auxiliary IV target. For a selected term $j\in\Sset$, define $d_{j,\Sset}=(\zp)^\top\wperp$ using the residualized quantities from Section~\ref{sec:targets}. When $d_{j,\Sset}\ne0$, replacing the noise-free response in \eqref{eq:realizediv} with the testing response gives the estimator and standard error
\begin{equation}
\chat=\frac{(\zp)^\top\ytest}{d_{j,\Sset}},
\qquad
\ee=\frac{\|\zp\|_2}{|d_{j,\Sset}|}.
\label{sec3:eq:estimator}
\end{equation}
Conditional on $\mathcal H$, the target $\cc$ and $\ee$ are fixed and $\chat$ is linear in the Gaussian testing response, which gives Theorem~\ref{sec3:thm:exact}.
For $p\in(0,1)$, let $z_p=\Phi^{-1}(p)$, where $\Phi$ is the standard normal distribution function.

\begin{theorem}[Exact conditional IV inference]
\label{sec3:thm:exact}
Assume the measurement model~\eqref{eq:measurement} with known $\sigma_X^2$ and the four-view construction of Section~\ref{sec3:construction}, so that \eqref{sec3:eq:gaussian} holds (Lemma~\ref{sec3:lem:views}). Whenever $d_{j,\Sset}\ne0$, the standardized estimation error has the conditional distribution
\begin{equation}
\frac{\chat-\cc}{\ee}\ \bigg|\ \mathcal H\sim N(0,1).
\label{sec3:eq:exactpivot}
\end{equation}
For $0<\alpha<1$, define
\begin{equation}
I^{\mathrm{IV}}_{j,\Sset}(\alpha) =[\chat-z_{1-\alpha/2}\ee,\ \chat+z_{1-\alpha/2}\ee],\qquad p^{\mathrm{IV}}_{j,\Sset}(c_0) =2\Phi\!\left(-\frac{|\chat-c_0|}{\ee}\right).
\label{sec3:eq:ivci}
\end{equation}
Then
$\Prb\{\cc\in I^{\mathrm{IV}}_{j,\Sset}(\alpha)\mid\mathcal H\}=1-\alpha$.
The $p$-value evaluated at $\cc$ is uniform on $(0,1)$ conditional on $\mathcal H$.
\end{theorem}

When $d_{j,\Sset}=0$, the realized IV coefficient is unidentified and no IV test or interval is reported; additional numerical checks preserve the conditional guarantee for available IV inference if they depend only on $\mathcal H$.

\subsection{Relating the IV and BLA coefficients}
\label{sec3:bridge}

An exact IV interval need not cover the primary selected-BLA coefficient: the estimation error for that target contains an additional component. When both targets are defined,
\[
\chat-\bb
=
\underbrace{\chat-\cc}_{\text{response error}}
+
\underbrace{\cc-\bb}_{\text{IV--BLA difference (gap)}}.
\]
Theorem~\ref{sec3:thm:exact} characterizes the response error. To obtain inference for the primary BLA target, it remains to quantify the IV--BLA difference. If the corrected design and instrument were equal to the noise-free design, $\Wt=\Zt=\Wbar$, the two targets would coincide. The remaining difference therefore arises from the random errors in the corrected design and instrument. We quantify their effect by expanding in $E_3=\Wt-\Wbar$ and $E_4=\Zt-\Wbar$. We use $\Rightarrow$ and $\to_p$ for convergence in distribution and probability, respectively.

To characterize the IV--BLA difference asymptotically, we consider the following fixed-dimensional dense-grid regime.

\begin{assumption}[Fixed-dimensional weak-design limit]
\label{sec3:ass:limit}
The observation grid becomes dense on a fixed time interval as $M\to\infty$. The dimensions $K$ and $P$, noise variance $\sigma_X^2$, thinning fractions, and weak weight functions remain fixed. There exist deterministic limits $A\in\R^{K\times P}$ and $b\in\R^K$ such that
\[
A_M:=\Wbar/\sqrt M\longrightarrow A,\qquad
b_M:=\ybar/\sqrt M\longrightarrow b.
\]
The submatrix $A_\Sset$, containing the columns of $A$ indexed by $\Sset$, has full column rank. The errors $(E_3,E_4)$ converge jointly in distribution to independent centered Gaussian matrices $(\mathcal E_3,\mathcal E_4)$ with finite covariance matrices $\Omega_3,\Omega_4$ of their columnwise vectorizations.
\end{assumption}

For fixed-degree polynomial libraries, Proposition~\ref{sec3:prop:weaklimit} gives sufficient trajectory and weak-weight conditions for these limits; full column rank of $A_\Sset$ remains a separate assumption.

The first-order gap combines instrument error, the error in adjusting for the other selected terms, and corrected-design error. Only the last remains when the limiting selected-model approximation residual is zero. For $j\in\Sset$, let $q>0$ be the squared norm of column $j$ of $A$ after projecting out the other selected columns, and let $\tau_{j,\Sset}\ge0$ denote the standard deviation of the Gaussian limit of $\sqrt M(\cc-\bb)$. Appendix~\ref{sec3:app:bridge} derives the expansion and its variance \eqref{sec3:eq:tauquadratic}.

\begin{theorem}[Selected-model IV--BLA bridge]
\label{sec3:thm:bridge}
Under Assumption~\ref{sec3:ass:limit}, for each fixed selected set $\Sset$ and coordinate $j\in\Sset$,
\begin{equation}
\sqrt M(\cc-\bb)\Rightarrow N(0,\tau_{j,\Sset}^2),\qquad M\ee^2\to_p\kappa_{j,\Sset}^2:=1/q>0.
\label{sec3:eq:bridgelimit}
\end{equation}
The probability that the IV coefficient is unidentified tends to zero. These conclusions also hold conditional on $\Shat=\Sset$ whenever that event has positive probability along the sequence considered.
\end{theorem}

The gap contributes uncertainty at the response error's $M^{-1/2}$ scale when $\tau_{j,\Sset}>0$ and is negligible at this scale when $\tau_{j,\Sset}=0$; the centering in \eqref{sec3:eq:bridgelimit} remains the finite-grid BLA of Section~2.

\subsection{Asymptotic inference for the selected BLA}
\label{sec3:bla}

We now combine the two uncertainty components identified in Section~\ref{sec3:bridge} to obtain inference for the primary BLA target. The response error is quantified by the exact conditional standard error $\ee$ from Section~\ref{sec3:exact}, while the IV--BLA gap contributes an additional asymptotic uncertainty of scale $\tau_{j,\Sset}/\sqrt M$. Because their limiting Gaussian components are independent (Theorems~\ref{sec3:thm:exact} and~\ref{sec3:thm:bridge}), their variances add. With $\ttt\ge0$ estimating $\tau_{j,\Sset}/\sqrt M$, the combined standard error, confidence interval, and two-sided $p$-value are
\begin{equation}
\sss=\sqrt{\ee^2+\ttt^2},\qquad I^{\mathrm{BLA}}_{j,\Sset}(\alpha) =\chat\pm z_{1-\alpha/2}\sss,\qquad p^{\mathrm{BLA}}_{j,\Sset}(c_0) =2\Phi\!\left(-\frac{|\chat-c_0|}{\sss}\right).
\label{sec3:eq:blaci}
\end{equation}

The response scale $\ee$ is available directly from \eqref{sec3:eq:estimator}, whereas the IV--BLA gap scale depends on the unknown noise-free trajectory through the remaining uncertainty in the corrected design and instrument. We therefore estimate this additional scale using a parametric bootstrap \citep{davison1997} around a pilot trajectory $\xpilot$. The pilot is obtained by local cubic smoothing of the allocation-weighted average $(a_2V_2+a_3V_3+a_4V_4)/(1-a_1)$, using a prespecified bandwidth. It therefore uses only views independent of the selection view.

Each draw fixes the selected support and the pilot's whitened weak response, regenerates the corrected design and instrument, and recomputes residualization and the IV coefficient. Fixing the response isolates gap uncertainty, since $\ee$ already accounts for testing-response noise.

The gap scale $\ttt$ is the interquartile range of the retained IV coefficients divided by $2\Phi^{-1}(3/4)$. This normal-reference scale does not require the unfiltered finite-sample IV ratio to have a finite variance. Appendix~\ref{sec3:app:bootstrap} and Algorithm~\ref{sec3:alg:bootstrap} specify the construction and retention checks.

Under the conditions of Proposition~\ref{sec3:prop:scale}, $\sqrt M\,\ttt\to_p\tau_{j,\Sset}$ as $M$ and the number of bootstrap draws $\Bboot$ tend to infinity, including when $\tau_{j,\Sset}=0$.

\begin{theorem}[Asymptotically valid selected-BLA inference]
\label{sec3:thm:bla}
Fix a selected set $\Sset$ and coordinate $j\in\Sset$ such that $\Prb(\Shat=\Sset)>0$ along the sequence considered. Assume the conditions of Theorems~\ref{sec3:thm:exact} and~\ref{sec3:thm:bridge} and Proposition~\ref{sec3:prop:scale} hold, and the observed and pilot numerical checks pass with probability tending to one for this support. Conditional on $\Shat=\Sset$,
\begin{equation}
\frac{\chat-\bb}{\sss}\Rightarrow N(0,1).
\label{sec3:eq:blapivot}
\end{equation}
Consequently, for every fixed $0<\alpha<1$,
\[
\Prb\{\bb\in I^{\mathrm{BLA}}_{j,\Sset}(\alpha)\mid\Shat=\Sset\}\to1-\alpha,
\]
and the $p$-value evaluated at the true selected BLA coefficient converges in distribution to a uniform random variable on $(0,1)$.
\end{theorem}

BLA validity conditions on the selected support and allows the design and instrument to vary; exact IV validity conditions on $\mathcal H$. The BLA result assumes dense observations of a scalar trajectory, a fixed polynomial library, and fixed weak weights. The pilot may use $V_2$, since scale consistency does not require independence from the testing response.

IV inference remains available when design and instrument checks pass, even if the gap scale is unavailable. Failed BLA checks yield the interval $\R$ and $p=1$; their probabilities vanish under the stated asymptotic conditions. Within the selected family, Bonferroni intervals and Bonferroni or Holm tests provide simultaneous coverage and familywise error control, respectively, conditionally in finite samples for available IV inference and asymptotically for BLA under the respective conditions (Appendix~\ref{sec3:app:coverage}). Algorithms~\ref{sec3:alg:inference}--\ref{sec3:alg:bootstrap} specify reporting, and Appendix~\ref{sec3:proofs} contains the proofs.

%% file: sec4.tex
\section{Experiments}
\label{sec4:experiments}

\paragraph{Methods and evaluation metrics.}
\label{sec4:design}
We assess Type~I error, reference-null power, coverage, and median interval width at nominal level $0.05$. Type~I error is the rejection rate when testing each selected support's true BLA or realized IV coefficient; coverage is the complementary rate for the corresponding $95\%$ intervals on the same cases. Appendix~\ref{sec4:app:counts} specifies target-specific denominators, unavailable-inference handling, and pointwise Wilson Monte Carlo intervals.

Naive reuses the selection response while retaining the corrected design and independent instrument. Model-count Bonferroni adjusts Naive $p$-values over eligible models. Their definitions and guarantees are in Appendix~\ref{sec4:app:bon}.

\begin{figure*}[!t]
\centering\includegraphics[width=0.86\textwidth]{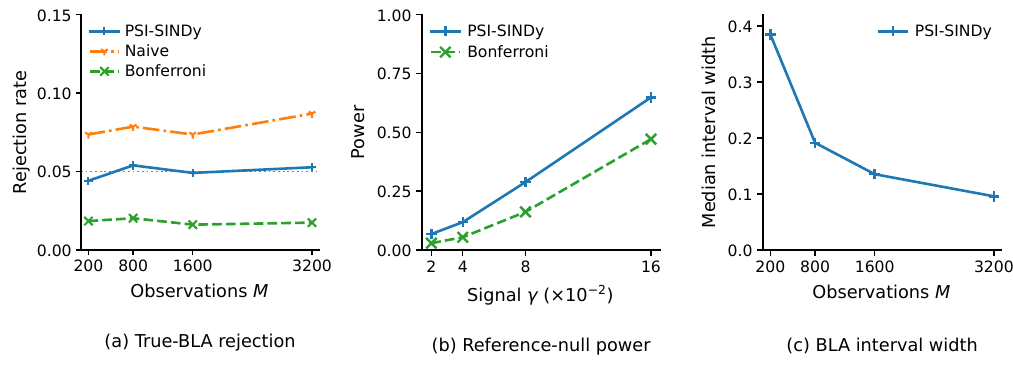}
\caption{Primary selected-BLA inference conditional on selecting $x^2$ ($\sigma_X=0.06$, $\delta=0$). (a) Rejection at each support's true BLA, $\gamma=0$; dotted line at 0.05. (b) Power against the reference null \eqref{sec4:eq:null}, $M=800$. (c) Median width of PSI-SINDy's 95\% BLA intervals, $\gamma=0$; counts and MC intervals in Tables~\ref{emp:tab:primary}--\ref{sec4:tab:power}.}
\label{fig:main-inference}
\end{figure*}

\paragraph{Simulation setup.}
We simulate noisy observations from the scalar dynamics
\begin{equation}
\dot x^\star(t)=-0.25-0.35x^\star(t)+\gamma\{x^\star(t)\}^2+\delta\sin\{2x^\star(t)\}, \qquad x^\star(0)=1.2,\qquad t\in[0,8].
\label{sec4:eq:ode}
\end{equation}
Primary simulations use independent Gaussian noise with known $\sigma_X=0.06$, the library $(1,x,x^2,x^3)$, $K=12$ weak equations, and equal thinning fractions. Sequentially thresholded least squares (STLSQ) selects corrected, whitened features on $V_1$, retaining the intercept, with a threshold fixed using only selected-support counts from separate pilot simulations.

Each primary setting has 5,000 independently generated datasets. Calibration uses $\gamma=\delta=0$ and $M\in\{200,800,1600,3200\}$; power uses $\delta=0$, $M=800$, and $\gamma\in\{0.02,0.04,0.08,0.16\}$. For the quadratic coefficient ($j=3$), BLA calibration uses each support's true coefficient as the null value. All methods use the same datasets selecting this term. Rejection, coverage, and power are proportions over these datasets, retaining each support's target. Bonferroni counts the four models containing the intercept and quadratic term.

For the power comparison, we test \begin{equation} H_{0,\Sset}:\quad \xi^\star_{3,\Sset}=\xi^{\star,(0)}_{3,\Sset}, \label{sec4:eq:null} \end{equation} where $\xi^{\star,(0)}_{3,\Sset}$ is the same subset's quadratic BLA on the $\gamma=0$ reference trajectory; omitted relevant terms can make it nonzero. Every selected quadratic coefficient received a finite PSI-SINDy interval in the primary simulations. Appendix~\ref{sec4:appendix} gives the protocol, selected counts, and pointwise Monte Carlo (MC) intervals \citep{morris2019}.

\subsection{Type I error, power, coverage, and interval width}
\label{sec4:bla}

Primary PSI-SINDy rejection rates at the true selected BLA are 4.41\%--5.39\% (Figure~\ref{fig:main-inference}), versus 7.36\%--8.69\% for Naive and 1.61\%--2.02\% for Bonferroni. Its 95\% intervals have 94.61\%--95.59\% coverage on these datasets, and all four MC rejection intervals contain 5\%. As $M$ increases from 200 to 3200, the median BLA interval width decreases from 0.3846 to 0.0957.

Across the signal grid, power for \eqref{sec4:eq:null} rises from 6.75\% to 64.85\% for PSI-SINDy and from 2.85\% to 47.10\% for Bonferroni; quadratic-term selection frequency rises from 70.80\% to 97.92\%. Higher power is specific to this protocol: a supplementary comparison reverses the ordering (Appendix~\ref{full:app:comparison}). Neither comparison establishes superiority at matched conditional Type~I error.

Pooled calibration does not establish calibration within every selected model: one of eight prespecified support groups in the supplementary study shows conditional BLA undercoverage (Appendix~\ref{app:support-calibration}).

\paragraph{Comparison with coefficient-bootstrap methods.}
We compare PSI-SINDy with Ensemble-SINDy data-bagging adapters \citep{fasel2022} and a residual-bootstrap adapter motivated by \citet{gao2023}. At $M=800$, $\sigma_X=0.06$, and $\delta=0$, we test the structural null $H_0:\gamma=0$ on the same 1,382 of 2,000 datasets in which PSI-SINDy selects $x^2$. Rejection is $4.70\%$ for PSI-SINDy, $8.10\%$ for the longer-smoothing ensemble, and $20.48\%$ for the residual adapter; the latter two MC intervals exclude $5\%$. The five-point ensemble has no observed rejections but about 15 times PSI-SINDy's median interval width. This empirical comparison concerns the specified adapters and structural target. Appendix~\ref{app:external-comparison} gives implementation details, coverage, widths, and a setting with a different ordering.

\subsection{Effect of measurement-error adjustments}
\label{sec4:main-bridge}\label{sec4:bridge}
\label{full:sec:supplementary}

\noindent
Table~\ref{tab:measurement-ablations} assesses feature correction and gap adjustment with matched observations, selection, views, and bootstrap randomness at the same true selected BLA; removing the gap scale keeps the estimate, whereas removing feature correction recomputes the design, instrument, estimate, and standard error.

Omitting the gap scale increases rejection for all three coefficients at both noise levels. Removing feature correction raises high-noise intercept rejection from 5.53\% to 11.06\%, but leaves the linear and quadratic tests unchanged: the included intercept absorbs the constant subtracted from the quadratic feature, an effect specific to the selected basis.

We also examine the accuracy of the asymptotic bridge at the sample sizes used. Primary bridge diagnostics show a decreasing median normalized linearization remainder with $M$ and a median estimated-to-oracle scale ratio close to one (Appendix~\ref{app:bridge-accuracy}). In a separate sine-drift study, rejection at the respective true coefficients is 12.18\% for BLA and 5.15\% for IV inference, conditional on selection of the quadratic term (Appendix~\ref{app:port-misspec-power}). BLA failure coincides with a larger bridge remainder (Table~\ref{tab:bridge-regimes}), illustrating that near-nominal IV calibration does not ensure finite-sample BLA calibration.

\begin{table}[!htbp]
\setlength{\tabcolsep}{3pt}
\caption{Rejection at the true selected BLA (\%; nominal 5\%) in the matched polynomial study: $\gamma=0.20$, $\delta=0$, $M=800$, $\Sset=\{1,2,3\}$. $N_\Sset$ counts datasets selecting this support. Full uses both adjustments. Paired differences and MC intervals are in Table~\ref{tab:paired-ablation-uncertainty}.}
\label{tab:measurement-ablations}
\begin{center}\small
\begin{tabular}{@{}rcrccc@{}}\toprule
$\sigma_X$ & Term & $N_\Sset$ & Full & No gap & \shortstack{No feature\\correction}\\\midrule
0.06 & $1$ & 812 & 5.42 & 7.51 & 6.28\\
0.06 & $x$ & 812 & 5.30 & 7.51 & $5.30^\dagger$\\
0.06 & $x^2$ & 812 & 5.42 & 6.28 & $5.42^\dagger$\\
\addlinespace[2pt]
0.12 & $1$ & 452 & 5.53 & 7.30 & 11.06\\
0.12 & $x$ & 452 & 6.86 & 8.85 & $6.86^\dagger$\\
0.12 & $x^2$ & 452 & 5.53 & 6.86 & $5.53^\dagger$\\
\bottomrule\end{tabular}
\par\vspace{2pt}{\footnotesize $\dagger$ For the selected basis $\{1,x,x^2\}$, the intercept absorbs the quadratic correction, leaving the $x$ and $x^2$ tests unchanged.}
\end{center}
\end{table}

\subsection{Larger libraries and multistate systems}
\label{sec4:extensions}\label{ext:sec:largerp}

\label{ext:sec:systems}
We examine libraries with up to $P=20$ candidates and four multistate systems. Larger-library inference remains informative within sparse selected models; in a separate study with unregularized selection, intervals widen substantially (Appendix~\ref{ext:app:largerp}). This does not establish growing-dimensional validity. Multistate experiments extend beyond the scalar BLA theorem: gap adjustment improves coverage, but nominal coverage is not uniformly attained (Appendix~\ref{ext:app:systems}).

\subsection{Supercapacitor discharge data}
\label{full:sec:real}

We analyze eight supercapacitor discharge traces \citep{hanschek2026,hanschekdata} using a cubic voltage library, $M=1600$ processed observations per device, and $K=20$ weak equations. Device~1 retains constant, linear, and cubic terms; Table~\ref{real:tab:processedcoef} reports estimates and nominal 95\% intervals. Estimated noise variance and unresolved temporal dependence place this application outside the theoretical guarantees. Appendix~\ref{full:app:real} specifies preprocessing and settings.

Simulations based on fixed, smoothed traces assess sensitivity to temporal dependence while inference assumes independent noise. With the true marginal noise scale supplied, selected-coefficient BLA coverage, averaged equally across devices, is 95.23\% under independent noise and 81.86\% under correlated noise (Table~\ref{real:tab:semisynthetic}). These simulations show sensitivity to dependence; they do not validate measured-trace coverage.

%% file: sec5.tex
\section{Conclusion}
\label{sec5:discussion}

We proposed PSI-SINDy for statistical inference after selecting dynamical terms from a noisy trajectory. Four-view data thinning and measurement-error correction address selection bias, shared response--design noise, and design measurement error. It gives finite-sample exact conditional inference for the auxiliary realized IV coefficient and, through an asymptotic bridge, asymptotically valid inference for the primary selected-model BLA coefficient. In the primary simulations, its Type~I error and coverage are compatible with the nominal level, whereas naive reuse of the selection response inflates rejection. Omitting the gap adjustment also increases rejection in the matched ablation.

The selected-model BLA is a least-squares approximation to the noise-free weak response and need not equal a governing-equation coefficient. Its inferential guarantee is asymptotic and currently established for scalar trajectories, fixed polynomial libraries, and independent Gaussian noise of known variance; in some supplementary settings, BLA intervals undercover in finite samples even where exact IV inference remains near nominal. Finite-sample calibration and extensions to multistate systems, broader libraries, dependent noise, and estimated noise variance remain open.

%% file: app.tex
\FloatBarrier
\label{main:end}
\input{ai_use_statement}

\input{bibliography}
\clearpage
\appendix
\section*{Appendices}
\numberwithin{equation}{section}
\numberwithin{theorem}{section}
\numberwithin{lemma}{section}
\numberwithin{assumption}{section}
\numberwithin{proposition}{section}
\section{Finite-grid weak-form formulation}
\label{app:formulation}

Throughout, $\R$ denotes the real numbers and $A^\top$ the transpose of a vector or matrix $A$. We use $\Prb$, $\E$, $\Var$ and $\Cov$ for probability, expectation, variance and covariance, respectively, and $\indep$ for statistical independence. The notation $N(\mu,\Sigma)$ denotes a Gaussian distribution with mean $\mu$ and covariance $\Sigma$ (variance in the scalar case), and iid denotes independent and identically distributed.

\subsection{Continuous-time dynamics and finite-grid approximation}
\label{app:finitegrid}

The weak-form literature often refers to the functions $\psi_k$ as test functions; here we call them weight functions to distinguish them from statistical tests.

We separate approximation by the selected library from numerical quadrature. For a continuously differentiable trajectory $x^\star$, assume that the feature integrands $\psi_k(t)\theta_j(x^\star(t))$ are integrable and define the continuous weak quantities over an interval of length $T=t_M-t_1$ by
\[
y^{\mathrm{cont}}_k=-\int_{t_1}^{t_M}\psi'_k(t)x^\star(t)\,dt,
\qquad
W^{\mathrm{cont}}_{kj}=\int_{t_1}^{t_M}
\psi_k(t)\theta_j(x^\star(t))\,dt.
\]
For the full-library approximation in \eqref{eq:libraryapprox}, multiplying \eqref{eq:dynamics} by $\psi_k$ and integrating by parts gives
\begin{equation}
-\int_{t_1}^{t_M}\psi'_k(t)x^\star(t)\,dt
\approx\sum_{j\in[P]}\xi_j
\int_{t_1}^{t_M}\psi_k(t)\theta_j\bigl(x^\star(t)\bigr)\,dt.
\label{eq:weakcontinuous}
\end{equation}

For a selected set $\Sset$ and any coefficient vector $\boldsymbol\beta\in\R^{|\Sset|}$, integration by parts and \eqref{eq:dynamics} give
\begin{equation}
y^{\mathrm{cont}}_k-\sum_{j\in\Sset}W^{\mathrm{cont}}_{kj}\beta_j
=\int_{t_1}^{t_M}\psi_k(t)
\left[f(x^\star(t))-\sum_{j\in\Sset}\beta_j\theta_j(x^\star(t))\right]dt.
\label{eq:continuousresidual}
\end{equation}
The boundary term is zero because the weights have support inside $(t_1,t_M)$. Here $\beta_j$ denotes the coefficient associated with original feature label $j$.

Suppose that $g_{k,0}(t)=-\psi'_k(t)x^\star(t)$ and
$g_{k,j}(t)=\psi_k(t)\theta_j(x^\star(t))$ are continuously differentiable. Their endpoint values are zero because the smooth weights are supported inside the observation interval. For any such integrand $g$, the full-grid sum equals the left Riemann sum because $g(t_M)=0$, and
\begin{equation}
\begin{aligned}
\left|\Delta t\sum_{m=1}^M g(t_m)-\int_{t_1}^{t_M}g(t)\,dt\right|
&\le \sum_{m=1}^{M-1}\int_{t_m}^{t_{m+1}}
|g(t_m)-g(t)|\,dt\\
&\le \frac{T\Delta t}{2}\,\|g'\|_\infty.
\end{aligned}
\label{eq:quadraturebound}
\end{equation}
The last step uses $|g(t_m)-g(t)|\le\|g'\|_\infty(t-t_m)$, where
$\|g'\|_\infty=\sup_{t\in[t_1,t_M]}|g'(t)|$.
Let $\boldsymbol e_y=\ystar-\boldsymbol y^{\mathrm{cont}}$ and
$E_W=W^\star-W^{\mathrm{cont}}$ be the quadrature errors. Then
\[
\ystar-W^\star_\Sset\boldsymbol\beta
=\bigl(\boldsymbol y^{\mathrm{cont}}-W^{\mathrm{cont}}_\Sset\boldsymbol\beta\bigr)
+\boldsymbol e_y-E_{W,\Sset}\boldsymbol\beta.
\]
Equation~\eqref{eq:continuousresidual} identifies the continuous approximation residual, while \eqref{eq:quadraturebound} bounds each quadrature-error component. The selected BLA in Section~2 is defined using the finite-grid quantities themselves. Its finite-grid interpretation therefore does not require the continuous approximation residual or the quadrature error to be zero.

\subsection{Normal equations and residualization within the selected model}
\label{full:app:normal}
Since $Q^\top Q=\Sigma_{\mathrm{test}}^{-1}$, the whitened objective in \eqref{eq:selectedbla} equals its covariance-weighted form:
\[
\|\ybar-\bar W^\star_\Sset\boldsymbol\beta\|_2^2
=(\ystar-W^\star_\Sset\boldsymbol\beta)^\top
\Sigma_{\mathrm{test}}^{-1}
(\ystar-W^\star_\Sset\boldsymbol\beta).
\]

Let $X=\Wbar_\Sset$, $\mu=\ybar$, and assume $X$ has full column rank. The selected BLA minimizes $\|\mu-X\boldsymbol\beta\|_2^2$ over $\boldsymbol\beta\in\R^{|\Sset|}$. Its unique minimizer is $(X^\top X)^{-1}X^\top\mu$, and its residual $\rho_\Sset=\mu-X\xistar_\Sset$ satisfies $X^\top\rho_\Sset=0$. These normal equations are used in the IV--BLA bridge derivation in Appendix~\ref{sec3:app:bridge}.

For $j\in\Sset$, write $X_j=\bar W^\star_{:,j}$ and $N=\Nset$, and let $X_N$ contain the columns with original feature labels in $N$. The matrix $R_N=I_K-X_NX_N^+$ is the orthogonal projector onto the orthogonal complement of the nuisance column space. Equivalently, write $\Mstar=R_N$ and define the noise-free residualized response and regressor by $y^{\star\perp}_{j,\Sset}=R_N\mu$ and $w^{\star\perp}_{j,\Sset}=R_NX_j$.

Minimizing first over the nuisance coefficients leaves the objective $\|R_N\mu-R_NX_j b\|_2^2$ for a scalar coefficient $b\in\R$. Full column rank gives $X_j^\top R_NX_j=\|R_NX_j\|_2^2>0$, so differentiation yields the Frisch--Waugh--Lovell representation \citep{lovell1963}:
\begin{equation}
\xi^\star_{j,\Sset}
=\frac{X_j^\top R_N\mu}{X_j^\top R_NX_j}
=\frac{(w^{\star\perp}_{j,\Sset})^\top y^{\star\perp}_{j,\Sset}}
{(w^{\star\perp}_{j,\Sset})^\top w^{\star\perp}_{j,\Sset}}.
\label{eq:fwl}
\end{equation}
This is the main-text representation \eqref{eq:selectedbla-fwl}. If $N$ is empty, $R_N=\Mstar=I_K$. Thus, $\xi^\star_{j,\Sset}$ is the coefficient of term $j$, retaining its original library label, after adjusting for the other selected terms. The coefficient can change when the selected adjustment set changes.
\setcounter{equation}{0}

\section{PSI-SINDy construction and proofs}
\label{sec3:proofs}

The proofs follow the construction in Section~\ref{sec3:method}. We first establish independence of the measurement views and unbiasedness of the corrected features. The conditional Gaussian law then gives auxiliary IV inference. Weak-design limits and differentiation of the nuisance projector yield the IV--BLA bridge, after which pilot and bootstrap consistency establish the primary BLA result.

\subsection{Independent views and polynomial correction}
\label{sec3:app:construction}

Write $\xobs=(x_1,\ldots,x_M)^\top$. For the fixed fractions in Section~\ref{sec3:construction}, independently of the observations draw mutually independent $g_i\sim N(0,a_i\sigma_X^2I_M)$. With $g=\sum_i g_i$, define
\begin{equation}
V_i=\xobs+\frac{g_i-a_i g}{a_i},
\qquad s_i^2=\frac{\sigma_X^2}{a_i},
\qquad i=1,\ldots,4.
\label{sec3:eq:thinning}
\end{equation}
Here $s_i>0$ is the noise standard deviation in view $i$.
\begin{lemma}[Independent measurement views]
\label{sec3:lem:views}
The vectors $V_i$ are mutually independent with $V_i\sim N(\xstar,s_i^2I_M)$, and $\sum_i a_iV_i=\xobs$.
\end{lemma}
This independence is with respect to measurement noise and thinning randomness, with the noise-free trajectory fixed. It does not hold after conditioning on the original observed vector $\xobs$.

\begin{proof}[Proof of Lemma~\ref{sec3:lem:views}]
Put $G_i=g_i-a_i g$. Independence of the $g_i$ gives
\[
\Cov(G_i,G_\ell)
=\sigma_X^2(a_i\mathbf1_{\{i=\ell\}}-a_i a_\ell)I_M,
\]
where $\mathbf1_E$ is the indicator of an event $E$. The original measurement error is independent of the $G_i$. Hence
\[
\Cov(V_i-\xstar,V_\ell-\xstar)
=\sigma_X^2I_M+\frac{\Cov(G_i,G_\ell)}{a_i a_\ell}
=\mathbf1_{\{i=\ell\}}\frac{\sigma_X^2}{a_i}I_M.
\]
The views are jointly Gaussian as affine functions of Gaussian vectors. Their block-diagonal covariance implies mutual independence. Their means equal $\xstar$, and $\sum_iG_i=0$ proves the reconstruction identity.
\end{proof}

\paragraph{Known variance and practical substitution.}
The exact guarantees assume that the variance used for thinning equals $\sigma_X^2$. Substituting an estimate, including one obtained from independent data, need not preserve independence of the thinned views \citep[Section~2.3]{neufeld2024} and is not covered by the present guarantees. The discussion in Section~\ref{sec:weaksetup} therefore distinguishes a practical variance estimate from the known variance assumed in the theoretical analysis.

\paragraph{Weak maps, whitening, and corrected features.}
The weak sums from Section~2 use the deterministic $K\times M$ matrices
\begin{equation}
(D_M)_{km}=-\psi'_k(t_m)\Delta t,\qquad (H_M)_{km}=\psi_k(t_m)\Delta t,\qquad k\in[K],\quad m\in[M].
\label{sec3:eq:weakmaps}
\end{equation}
The testing response before whitening is $D_MV_2$. Its covariance determines the whitening map introduced in Section~2:
\begin{equation}
\Sigma_{\mathrm{test}}=s_2^2D_MD_M^\top,\qquad Q=\Sigma_{\mathrm{test}}^{-1/2},\qquad C_M=QD_M,\qquad B_M=QH_M.
\label{sec3:eq:whitening}
\end{equation}
We assume that $\Sigma_{\mathrm{test}}$ is positive definite, so $s_2^2C_MC_M^\top=I_K$.
For $v\in\R^M$, let $\Theta(v)\in\R^{M\times P}$ have entries $\Theta(v)_{mj}=\theta_j(v_m)$. The noise-free quantities are then $\ybar=C_M\xstar$ and $\Wbar=B_M\Theta(\xstar)$, as in Section~2.

Independent views remove the shared noise between the testing response and features, but nonlinear feature evaluations can still have biased means. For the polynomial library $\theta_j(u)=u^{j-1}$, Gaussian measurement-error correction provides unbiased evaluations \citep{stefanski1989}. For a scalar argument $v$ and noise standard deviation $s>0$, define
\begin{equation}
\mathcal C_{0,s}(v)=1,\qquad \mathcal C_{1,s}(v)=v,\qquad \mathcal C_{d+1,s}(v) =v\mathcal C_{d,s}(v)-ds^2\mathcal C_{d-1,s}(v), \qquad d\ge1.
\label{sec3:eq:correction}
\end{equation}
Here $d$ denotes polynomial degree, so feature $j$ uses degree $j-1$. The correction satisfies
$\E\{\mathcal C_{d,s}(u+sG)\}=u^d$ for $G\sim N(0,1)$; a proof is given below. For a cubic library, the corrected evaluations are $(1,v,v^2-s^2,v^3-3s^2v)$.

Let $\mathcal C_s(v)$ be the $M\times P$ matrix with entries $\mathcal C_{j-1,s}(v_m)$. We construct the three inference quantities as
\begin{equation}
\ytest=C_MV_2,\qquad \Wt=B_M\mathcal C_{s_3}(V_3),\qquad \Zt=B_M\mathcal C_{s_4}(V_4).
\label{sec3:eq:inferencedata}
\end{equation}
The corrected design $\Wt$ and instrument $\Zt$ are independent and both have expectation $\Wbar$. Correction removes their mean bias, but they still fluctuate around this noise-free design.

For a real generating parameter $t$, the polynomial recursion in \eqref{sec3:eq:correction} has generating function
\[
\exp(tv-s^2t^2/2)=\sum_{d=0}^{\infty}\mathcal C_{d,s}(v)t^d/d!.
\]
Evaluating at $t=0$ gives the constant term; differentiating with respect to $t$ and equating coefficients gives the recursion and the linear term. For $G\sim N(0,1)$, the Gaussian moment-generating function gives
\[
\E\exp\{t(u+sG)-s^2t^2/2\}=e^{tu}.
\]
Taking the $d$th derivative at zero yields $\E\mathcal C_{d,s}(u+sG)=u^d$. The interchange with expectation is justified by finite Gaussian exponential moments on a neighborhood of zero.

For the subsequent covariance calculation, let $d,e$ be nonnegative polynomial degrees and define the probabilists' Hermite polynomials by
$\exp(tG-t^2/2)=\sum_{r\ge0}\mathrm{He}_r(G)t^r/r!$.
Multiplication of generating functions gives
\[
\mathcal C_{d,s}(u+sG)
=\sum_{r=0}^d\binom dr u^{d-r}s^r\mathrm{He}_r(G).
\]
For real generating parameters $t$ and $v$,
$\E\{\exp(tG-t^2/2)\exp(vG-v^2/2)\}=e^{tv}$.
Equating coefficients shows
$\E\{\mathrm{He}_r(G)\mathrm{He}_\ell(G)\}=r!\mathbf1_{\{r=\ell\}}$.
Consequently,
\begin{equation}
\Cov\{\mathcal C_{d,s}(u+sG),\mathcal C_{e,s}(u+sG)\}
=\sum_{r=1}^{\min(d,e)}
\binom dr\binom er r!u^{d+e-2r}s^{2r}.
\label{sec3:eq:featurecov}
\end{equation}
At fixed degrees, bounded $u$, and fixed $s$, all fixed moments of the centered corrected evaluations are uniformly bounded, since they are polynomials of a Gaussian variable with bounded coefficients.

The maps $B_M,C_M$ are deterministic. The correction identity and Lemma~\ref{sec3:lem:views} therefore give $\E\Wt=\E\Zt=\Wbar$ and independence of the two matrices. The response $\ytest=C_MV_2$ has mean $\ybar$ and covariance $s_2^2C_MC_M^\top=I_K$ and is independent of $\mathcal H$. This proves \eqref{sec3:eq:gaussian}.

\subsection{Exact conditional IV inference}
\label{sec3:app:exact}

For a fixed selected set $\Sset$ and $j\in\Sset$, write $N=\Sset\setminus\{j\}$. The residualizer from Section~2 is $\Mt=I_K-\Wt_N\Wt_N^+$, with $\Mt=I_K$ for empty $N$. Thus $\zp=\Mt\Zt_{:,j}$ and $\wperp=\Mt\Wt_{:,j}$. Since $\Mt$ is symmetric and idempotent, $\Mt\zp=\zp$ and \eqref{eq:realizediv} becomes
\[
\cc=\frac{(\zp)^\top\ybar}{d_{j,\Sset}}.
\]

\begin{proof}[Proof of Theorem~\ref{sec3:thm:exact}]
For this proof abbreviate $\zp$ by $z$ and $d_{j,\Sset}$ by $d$.
Both are $\mathcal H$-measurable. Substituting \eqref{sec3:eq:gaussian} into \eqref{sec3:eq:estimator} gives
\[
\chat-\cc=\frac{z^\top\varepsilon^{\mathrm{test}}}{d},
\qquad
U:=\frac{\chat-\cc}{\ee}
=\operatorname{sign}(d)\frac{z^\top\varepsilon^{\mathrm{test}}}{\|z\|_2}.
\]
Here $\operatorname{sign}(d)=d/|d|$ for $d\ne0$. Conditional on $\mathcal H$, the direction is fixed and has unit Euclidean norm. The last expression is therefore $N(0,1)$. Since $U\mid\mathcal H\sim N(0,1)$, $\Prb\{|U|\le z_{1-\alpha/2}\mid\mathcal H\}=1-\alpha$, which proves conditional coverage.
For $0<u<1$,
\[
\Prb\{2\Phi(-|U|)\le u\mid\mathcal H\}
=\Prb\{|U|\ge z_{1-u/2}\mid\mathcal H\}=u,
\]
which proves conditional uniformity.
\end{proof}

This argument requires independence of the testing noise from $\mathcal H$ and a nonzero IV denominator. Unbiased feature correction is needed for the BLA bridge, but not for this conditional Gaussian identity. When identification holds almost surely, the tower property also gives exact IV coverage conditional on $\Shat=\Sset$. Otherwise the exact statement applies conditional on the identified realizations.

\subsection{Sufficient conditions for the weak-design limit}
\label{sec3:app:weaklimit}

Write $\psi(t)=(\psi_1(t),\ldots,\psi_K(t))^\top$, and use $\vecop$ for columnwise vectorization.

\begin{appendixproposition}[Polynomial weak-design limit]
\label{sec3:prop:weaklimit}
Let the endpoints of the observation interval be fixed, with $T=t_M-t_1>0$. Suppose $x^\star$ and the prespecified weights $\psi_k$ are continuously differentiable and the polynomial library has fixed degree. If
$G_D=\int_{t_1}^{t_M}\psi'(t)\psi'(t)^\top dt$ is positive definite, define
$L_D=(s_2^2TG_D)^{-1/2}$. The limits in Assumption~\ref{sec3:ass:limit}, apart from its separate requirement on the rank of the selected design, hold with
\begin{equation}
\begin{aligned}
A&=L_D\int_{t_1}^{t_M}\psi(t)\theta(x^\star(t))^\top dt,\qquad
b=-L_D\int_{t_1}^{t_M}\psi'(t)x^\star(t)\,dt,\\
\Omega_i&=T\int_{t_1}^{t_M}\Gamma_i(x^\star(t))\otimes
\{L_D\psi(t)\psi(t)^\top L_D^\top\}\,dt,\qquad i=3,4.
\end{aligned}
\label{sec3:eq:weaklimits}
\end{equation}
Here $\theta(u)=(1,u,\ldots,u^{P-1})^\top$,
$\otimes$ is the Kronecker product, and $\Gamma_i(u)$ is the $P\times P$ covariance matrix from \eqref{sec3:eq:featurecov}, with $s=s_i$ and degrees $d=j-1$, $e=\ell-1$ in entry $(j,\ell)$. The matrices $\Omega_i$ may be singular.
\end{appendixproposition}

\begin{proof}
Riemann-sum convergence and $M\Delta t\to T$ give
\[
MD_MD_M^\top
=M(\Delta t)^2\sum_{m=1}^M\psi'(t_m)\psi'(t_m)^\top
\longrightarrow TG_D.
\]
Continuity of the symmetric inverse square root on positive-definite matrices gives
$Q/\sqrt M\to L_D$. Combining this limit with the weak sums gives the limits of $\Wbar/\sqrt M$ and $\ybar/\sqrt M$ in \eqref{sec3:eq:weaklimits}. Positive definiteness also ensures that the finite-grid whitening is well defined for all sufficiently large $M$.

Let $w_{Mm}=Q\Delta t\,\psi(t_m)$ be column $m$ of $B_M$. Then
$\max_m\|w_{Mm}\|_2=O(M^{-1/2})$ and
$\sum_m\|w_{Mm}\|_2^2=O(1)$.
For view $i$, let $h_{im}\in\R^P$ be the centered vector of corrected feature evaluations at time $t_m$. We have
\[
E_i=\sum_{m=1}^M w_{Mm}h_{im}^\top,\qquad
\Cov\{\vecop(E_i)\}
=\sum_{m=1}^M\Gamma_i(x^\star(t_m))\otimes(w_{Mm}w_{Mm}^\top).
\]
The latter sum converges to $\Omega_i$ by Riemann-sum convergence. For any fixed linear contrast of $\vecop(E_i)$, the independent centered summands have total $(2+\delta)$th absolute moment $O(M^{-\delta/2})$, for any fixed $\delta>0$, by the polynomial Gaussian moment bounds above. When its limiting variance is positive, the Lyapunov central limit theorem applies. When that variance is zero, Chebyshev's inequality gives convergence in probability to zero. The Cram\'er--Wold criterion yields the Gaussian matrix limit, and independence of views 3 and 4 gives independence of their limits.
\end{proof}

This argument also gives bounded row sums of absolute entries in $B_M/\sqrt M$ and $C_M/\sqrt M$. These bounds will transfer uniform pilot consistency to the normalized weak quantities. Neither smoothness nor the correction identity implies full rank of $A_\Sset$; that identification condition must be imposed for each support under consideration.

\subsection{IV--BLA expansion and bridge proof}
\label{sec3:app:bridge}

Write $N=\Sset\setminus\{j\}$.
The selected BLA residual is
$\rho_\Sset=\ybar-\Wbar_\Sset\xistar_\Sset$, with
$(\Wbar_\Sset)^\top\rho_\Sset=0$.
The IV moment gives the exact identity
\begin{equation}
\cc-\bb
=\frac{(\zp)^\top\rho_\Sset
-(\zp)^\top E_{3,\Sset}\xistar_\Sset}{d_{j,\Sset}}.
\label{sec3:eq:gap}
\end{equation}
The numerator contains the selected-model approximation residual and corrected-design error, both evaluated against the estimated residualized instrument. We allow $\rho_\Sset\ne0$: the selected model need not represent the noise-free weak response exactly.

To describe the first-order difference, let $\beta$ and $r$ be the BLA coefficient and approximation residual for the limiting pair $(A_\Sset,b)$. Let $R_0$ be the orthogonal projector onto the orthogonal complement of the limiting nuisance column space, let $v$ be the residualized tested column, and let $q$ be its squared norm. The vector $\boldsymbol\pi$ gives the projection coefficients of the original tested column $A_{:,j}$ on the nuisance columns:
\begin{equation}
\begin{aligned}
\beta&=(A_\Sset^\top A_\Sset)^{-1}A_\Sset^\top b,
&r&=b-A_\Sset\beta,\\
R_0&=I_K-A_NA_N^+,
&v&=R_0A_{:,j},\\
q&=v^\top v>0,
&\boldsymbol\pi&=A_N^+A_{:,j}.
\end{aligned}
\label{sec3:eq:limitobjects}
\end{equation}
The coordinate labels of $\beta$ retain the original feature indices, and empty nuisance terms are omitted. Define
\begin{equation}
L_{j,\Sset} =\frac{r^\top((\mathcal E_4)_{:,j}-\mathcal E_{3,N}\boldsymbol\pi) -v^\top\mathcal E_{3,\Sset}\beta}{q},\qquad \tau_{j,\Sset}^2=\Var(L_{j,\Sset}).
\label{sec3:eq:linearization}
\end{equation}
The expansion separates uncertainty from the instrument, the adjustment for the other selected terms, and the corrected design. The latter two contributions involve the same design error $\mathcal E_3$ and can be correlated; their covariance is included in the limiting variance $\Var(L_{j,\Sset})$.

If $r=0$, the instrument and nuisance-projector terms in \eqref{sec3:eq:linearization} vanish, although corrected-design uncertainty can remain. The bridge is centered at the finite-grid target $\bb$, rather than the limiting coefficient $\beta_j$, so it requires no additional rate for the convergence of $A_M$ and $b_M$.

For this subsection abbreviate $\widetilde M_N$ by $R$, $\zp$ by $z$, and $d_{j,\Sset}$ by $d$. The selected BLA normal equations give $(\Wbar_\Sset)^\top\rho_\Sset=0$. Appendix~\ref{full:app:normal} derives these equations and the residualization formula.

Because $z^\top\Wt_N=0$, substituting
$\ybar=(\Wt_\Sset-E_{3,\Sset})\xistar_\Sset+\rho_\Sset$ yields
\[
z^\top\ybar
=d\bb-z^\top E_{3,\Sset}\xistar_\Sset+z^\top\rho_\Sset.
\]
Division by $d$ proves \eqref{sec3:eq:gap}. At a noise-free pair $\Wt=\Zt=\Wbar$, residualization and the normal equations give $\cc=\bb$, including when $\rho_\Sset\ne0$.

We next differentiate the nuisance projector. This is a full-column-rank instance of projector differentiation \citep{golub1973}; the calculation is included here.
For a full-column-rank matrix $F$ and perturbation $D$, put $G=F^\top F$ and $\Pi(F)=FG^{-1}F^\top$. For $F(u)=F+uD$, a dot denotes differentiation at $u=0$. The product rule and differentiation of $GG^{-1}=I$ give
\[
\dot G=D^\top F+F^\top D,\qquad
\dot{(G^{-1})}=-G^{-1}\dot GG^{-1}.
\]
Substitution in $\dot\Pi$ and collection of terms yield
\[
\dot\Pi=(I-\Pi)DF^++(F^+)^\top D^\top(I-\Pi).
\]
For fixed direction matrices $D_3,D_4\in\R^{K\times P}$ and scalar perturbation parameter $u$, perturb the normalized design and instrument to $A+uD_3$ and $A+uD_4$ with $b$ fixed. In this differentiation, $R$, $z$ and $c$ denote the residualizer, residualized instrument and IV ratio formed from these normalized matrices. Then
\begin{equation}
\dot R=-R_0D_{3,N}A_N^+-(A_N^+)^\top D_{3,N}^\top R_0.
\label{sec3:eq:projectorderivative}
\end{equation}
For empty $N$, $\dot R=0$.
At the unperturbed point $z=v$, the derivative of the residualized instrument is
\[
\dot z
=R_0(D_{4,:,j}-D_{3,N}\boldsymbol\pi)
-(A_N^+)^\top D_{3,N}^\top v.
\]
Here $D_{i,:,j}$ denotes column $j$ of the direction matrix $D_i$. Since $v^\top A_{:,j}=q$ and $v^\top b=\beta_jq$, the quotient rule gives
\[
\dot c=\frac{\dot z^\top(b-\beta_jA_{:,j})-\beta_jv^\top D_{3,:,j}}{q}.
\]
Differentiating $z(u)^\top(A_N+uD_{3,N})=0$ at $u=0$ gives
$\dot z^\top A_N=-v^\top D_{3,N}$.
The normal equations imply $r^\top A_N=0$ and $R_0r=r$, so
$\dot z^\top r=r^\top(D_{4,:,j}-D_{3,N}\boldsymbol\pi)$.
Substituting $b-\beta_jA_{:,j}=A_N\beta_N+r$ proves
\begin{equation}
\dot c
=\frac{r^\top(D_{4,:,j}-D_{3,N}\boldsymbol\pi)-v^\top D_{3,\Sset}\beta}{q}.
\label{sec3:eq:targetderivative}
\end{equation}

\begin{proof}[Proof of Theorem~\ref{sec3:thm:bridge}]
The IV ratio is unchanged when its design, instrument, and response mean are all divided by $\sqrt M$. At the moving deterministic baseline $(A_M,A_M,b_M)$, its value is exactly $\bb$, by the noise-free identity above.

Full rank of $A_\Sset$ and $q>0$ give a neighborhood of $(A,A,b)$ on which the nuisance inverse and ratio denominator are nonsingular. The ratio is twice continuously differentiable there, with bounded second derivatives on a smaller closed neighborhood. The normalized perturbations are $E_3/\sqrt M$ and $E_4/\sqrt M$, both $O_p(M^{-1/2})$. We use the multivariate delta-method principle \citep[Chapter~3]{vandervaart1998} at this moving deterministic baseline. Taylor's theorem at $(A_M,A_M,b_M)$ gives a remainder $O_p(M^{-1})$. After multiplication by $\sqrt M$, this remainder is $o_p(1)$. The derivative coefficients converge to those in \eqref{sec3:eq:targetderivative}; applying the joint Gaussian limit proves the first claim in \eqref{sec3:eq:bridgelimit}.

Continuity of residualization gives $z/\sqrt M\to_p v$ and $d/M\to_p q>0$. This proves asymptotic identification, and
\[
M\ee^2
=\frac{\|z/\sqrt M\|_2^2}{(d/M)^2}\to_p\frac1q.
\]
For each fixed support, these quantities are functions of views 3 and 4 and deterministic problem quantities. Their law is unchanged by conditioning on the independent event $\Shat=\Sset$ when that event has positive probability. This proves the conditional statements without a uniform positive lower bound on the selection probability.
\end{proof}

For an explicit covariance expression, let $e_j^{(\Sset)}\in\R^{|\Sset|}$ have value one at the selected coordinate corresponding to $j$ and zero elsewhere. Let $\boldsymbol\pi^{\mathrm{emb}}\in\R^{|\Sset|}$ contain $\boldsymbol\pi$ at the nuisance coordinates and zero at $j$. Put
$J_3=-v\beta^\top-r(\boldsymbol\pi^{\mathrm{emb}})^\top$ and $J_4=r(e_j^{(\Sset)})^\top$.
If $\Omega_{i,\Sset}$ is the covariance submatrix for selected columns, then
\begin{equation}
\tau_{j,\Sset}^2
=\frac{\vecop(J_3)^\top\Omega_{3,\Sset}\vecop(J_3)
+\vecop(J_4)^\top\Omega_{4,\Sset}\vecop(J_4)}{q^2}.
\label{sec3:eq:tauquadratic}
\end{equation}
This expression uses the covariance of the limiting Gaussian matrices. No finite-sample variance of the random-denominator ratio is required.

\subsection{Pilot consistency and the bootstrap gap scale}
\label{sec3:app:bootstrap}

The pilot is constructed from
\begin{equation}
\overline V=\frac{a_2V_2+a_3V_3+a_4V_4}{1-a_1},
\label{sec3:eq:pilotview}
\end{equation}
using local cubic smoothing with a prespecified bandwidth $h_M$.

By independence of the views,
\[
\Var(\overline V)
=\frac{\sum_{i=2}^4 a_i^2s_i^2}{(1-a_1)^2}I_M
=\frac{\sigma_X^2}{1-a_1}I_M.
\]
Its mean is $\xstar$, and it is independent of $\mathcal A$. Write this variance factor as $\sigma_{\mathrm{pil}}^2$.

\begin{assumption}[Conditions for the local cubic pilot]
\label{sec3:ass:pilot}
The trajectory $x^\star$ has four continuous derivatives on the fixed interval. The deterministic fitting windows have radius of order $h_M$ and contain order $Mh_M$ observations, including at the boundaries. Their cubic Gram matrices, normalized by the window counts with coordinates scaled by $h_M$, have uniformly bounded inverses. The bandwidth satisfies
$h_M\to0$ and $Mh_M/\log M\to\infty$.
\end{assumption}

This construction gives
\begin{equation}
\|\xpilot-\xstar\|_\infty
=O(h_M^4)+O_p\!\left(\sqrt{\frac{\log M}{Mh_M}}\right)=o_p(1).
\label{sec3:eq:pilotr}
\end{equation}
Here $\|\cdot\|_\infty$ is the largest absolute coordinate.

Under Assumption~\ref{sec3:ass:pilot}, the local cubic fit at $t_m$ is a linear smoother $\widehat x_m=\sum_\ell w_{m\ell}\overline V_\ell$. The bounded inverse Gram matrices and rescaled coordinates give, uniformly in $m$,
\[
\sum_\ell|w_{m\ell}|\le C,\qquad
\sum_\ell w_{m\ell}^2\le \frac{C}{Mh_M}
\]
for a finite constant $C$. Cubic reproduction and a fourth-order Taylor remainder give bias $O(h_M^4)$. Each noise term is Gaussian with variance at most $C\sigma_{\mathrm{pil}}^2/(Mh_M)$. A Gaussian tail bound followed by a union bound therefore gives
\[
\Prb\!\left\{\max_m\left|\sum_\ell w_{m\ell}
(\overline V_\ell-x^\star_\ell)\right|>u\right\}
\le 2M\exp\{-cMh_Mu^2/\sigma_{\mathrm{pil}}^2\}
\]
for a constant $c>0$. Taking $u$ to be a sufficiently large multiple of
$\sqrt{\log M/(Mh_M)}$ proves \eqref{sec3:eq:pilotr}. One-sided boundary windows satisfy the same argument under the stated Gram condition.

We resample only the feature views because $\ee$ already accounts for testing-response noise. Each draw recomputes the adjustment for the other selected terms while keeping the response mean fixed. We use a parametric bootstrap of the known measurement model \citep{davison1997}, with superscript $\boot$ denoting bootstrap quantities. Hold the selected support fixed and condition on the inference views. Generate independent
$V_i^\boot\sim N(\xpilot,s_i^2I_M)$ for $i=3,4$, correct their features, and form $\Wt^\boot$ and $\Zt^\boot$ as in \eqref{sec3:eq:inferencedata}.
The resulting nuisance adjustment is
\[
\widetilde M_N^\boot=I_K-\Wt_N^\boot(\Wt_N^\boot)^+,\qquad
z_{j,\Sset}^\boot=\widetilde M_N^\boot\Zt^\boot_{:,j}.
\]
Holding the response mean at $\widehat{\bar{\boldsymbol y}}=C_M\xpilot$, calculate
\begin{equation}
c_{j,\Sset}^\boot
=\frac{(z_{j,\Sset}^\boot)^\top\widehat{\bar{\boldsymbol y}}}
{(z_{j,\Sset}^\boot)^\top\Wt^\boot_{:,j}}
\label{sec3:eq:bootratio}
\end{equation}
whenever the denominator is nonzero.

From $\Bboot$ draws independent conditional on the inference views, with the selected support held fixed, we compute $\ttt$ by dividing the retained IV coefficients' empirical interquartile range by the standard normal interquartile range. Algorithm~\ref{sec3:alg:bootstrap} specifies the retention checks and the empirical-quantile calculation.

Define the pilot weak design $\widehat{\bar W}=B_M\Theta(\xpilot)$.
When the pilot selected design has full column rank, let
\[
\widehat\xi^{\mathrm{pil}}_{j,\Sset}
=\left[
(\widehat{\bar W}_\Sset^\top\widehat{\bar W}_\Sset)^{-1}
\widehat{\bar W}_\Sset^\top\widehat{\bar{\boldsymbol y}}
\right]_j,
\]
where $[\cdot]_j$ denotes the coordinate associated with original feature $j$.
Let $\mathcal D=\sigma(V_2,V_3,V_4)$.
We use $\Rightarrow^\boot$ for weak convergence under bootstrap randomness conditional on $\mathcal D$, in probability with respect to the original experiment, with the support held fixed.
For the convergence arguments only, assign zero to an undefined pilot BLA or bootstrap coefficient. The calculations below show that pilot-rank failure has probability tending to zero and that the conditional probability of bootstrap nonidentification tends to zero in probability; numerical reporting remains as in Algorithm~\ref{sec3:alg:bootstrap}.

We first establish
\begin{equation}
\sqrt M(c^\boot_{j,\Sset}-\widehat\xi^{\mathrm{pil}}_{j,\Sset})
\Rightarrow^\boot N(0,\tau_{j,\Sset}^2)
\quad\text{in probability}.
\label{sec3:eq:bootlaw}
\end{equation}
Uniform pilot consistency implies $\max_m|\widehat x_m|=O_p(1)$. Fixed-degree polynomials are Lipschitz on bounded intervals. The normalized row-sum bounds from Appendix~\ref{sec3:app:weaklimit} then imply
$\widehat{\bar W}/\sqrt M\to_p A$ and
$\widehat{\bar{\boldsymbol y}}/\sqrt M\to_p b$.
The pilot selected design is consequently full rank with probability tending to one.

Conditional on $\mathcal D$, the centered bootstrap feature errors are independent across times, with covariances $\Gamma_i(\widehat x_m)$. Formula~\eqref{sec3:eq:featurecov} shows that the entries of these covariance matrices are polynomials. Uniform pilot consistency therefore makes their differences from $\Gamma_i(x^\star_m)$ uniformly $o_p(1)$ on bounded pilot events. Since $\sum_m\|w_{Mm}\|_2^2=O(1)$, the bootstrap weak-error covariances converge in probability to $\Omega_i$. The same fixed-degree moment bounds give a conditional Lyapunov bound $O(M^{-\delta/2})$ on those events.

To justify conditional convergence, take any subsequence and then a further subsequence on which uniform pilot consistency and convergence of the pilot baselines and covariance limits hold almost surely. For every data realization in this probability-one event, the triangular-array argument in Proposition~\ref{sec3:prop:weaklimit} gives the Gaussian bootstrap error limit. The Taylor expansion in Appendix~\ref{sec3:app:bridge}, now at the pilot baseline, gives \eqref{sec3:eq:bootlaw}. The subsequence characterization of convergence in probability proves the assertion for the original sequence. The same positive limiting denominator $q$ gives bootstrap identification with conditional probability tending to one.

\begin{proposition}[Consistency of the bootstrap gap scale]
\label{sec3:prop:scale}
Suppose the polynomial weak-design conditions of Proposition~\ref{sec3:prop:weaklimit}, the selected-design rank condition in Assumption~\ref{sec3:ass:limit}, and Assumption~\ref{sec3:ass:pilot} hold. If $\Bboot\to\infty$ and any filtering has conditional acceptance probability tending to one, then
\begin{equation}
\sqrt M\,\ttt\to_p\tau_{j,\Sset}.
\label{sec3:eq:scaleconsistency}
\end{equation}
This includes $\tau_{j,\Sset}=0$. The conclusion holds conditional on every positive-probability event $\Shat=\Sset$ for the fixed support.
\end{proposition}

\begin{proof}[Proof of Proposition~\ref{sec3:prop:scale}]
Write $n$ for the retained count. The conditional expected fraction of discarded draws tends to zero, so Markov's inequality gives $n/\Bboot\to_p1$. Thus $n\to_p\infty$ because $\Bboot\to\infty$. On $n>0$, the empirical distribution functions before and after filtering differ by at most $1-n/\Bboot$ at every point.

For a nondegenerate limiting distribution, the argument uses the usual quantile-consistency principle \citep[Chapter~21]{vandervaart1998}; we give the details and handle the degenerate case separately. For $p\in\{1/4,3/4\}$ and $\tau=\tau_{j,\Sset}>0$, put $a_p=\tau z_p$. For any $\epsilon>0$, the limiting distribution function is strictly below $p$ at $a_p-\epsilon$ and strictly above $p$ at $a_p+\epsilon$. The unfiltered conditional bootstrap probabilities satisfy the same strict inequalities with probability tending to one by \eqref{sec3:eq:bootlaw}. Conditional on $\mathcal D$, the empirical proportions from the $\Bboot$ independent draws have variance at most $1/(4\Bboot)$. Chebyshev's inequality and the bound on the effect of filtering place the retained empirical $p$-quantile between the two points with probability tending to one.

If $\tau=0$, weak convergence is to a point mass at zero. For every $\epsilon>0$, the probabilities below $-\epsilon$ and below $+\epsilon$ tend to zero and one, respectively. The same empirical-proportion argument shows that both quartiles tend to zero. Thus the quantile conclusion remains valid in this case.

Subtracting the pilot BLA from every draw leaves the IQR unchanged. The two scaled empirical quartiles converge to $\tau z_{1/4}$ and $\tau z_{3/4}$, so their difference divided by $2z_{3/4}$ converges to $\tau$. This proves \eqref{sec3:eq:scaleconsistency}.

For each fixed support, the entire pilot and scale calculation uses only $\mathcal D$ and bootstrap randomization independent of $\mathcal A$. Its law and failure probabilities are unchanged by conditioning on any positive-probability event $\Shat=\Sset$. This proves the conditional assertion.
\end{proof}

One sufficient filtering rule retains draws with full selected and nuisance rank and absolute denominator greater than $\delta_M$, where $\delta_M\ge0$ and $\delta_M/M\to0$. If numerical rank uses the smallest-to-largest singular-value ratio, a relative tolerance tending to zero suffices. A fixed tolerance instead needs to lie strictly below the corresponding limiting selected and nonempty nuisance ratios. These conditions give acceptance probability tending to one, since normalized selected designs converge to a full-rank limit and normalized denominators converge to $q>0$. They also make the observed and pilot rank checks pass with probability tending to one: the normalized pilot design converges to $A_\Sset$, and the normalized augmented nuisance/instrument matrix converges to a column permutation of $A_\Sset$. For an observed relative-denominator tolerance $\eta_M$, $\limsup_M\eta_M<1$ suffices, since $|d_{j,\Sset}|/(\|\zp\|_2\|\wperp\|_2)\to_p1$. A fixed minimum retained count is then reached with probability tending to one. Together, these conditions satisfy the numerical-check hypothesis of Theorem~\ref{sec3:thm:bla}.

A zero IQR is not treated as failure: on identified realizations, the positive response scale still makes the combined standard error positive. In the constant-only selected model $\Sset=\{1\}$, provided its noise-free weak column is nonzero, the corrected design and instrument equal the noise-free constant column exactly, so the IV and BLA targets coincide and we set the gap scale to zero directly.

\subsection{Selected-BLA coverage and simultaneous inference}
\label{sec3:app:coverage}

\begin{proof}[Proof of Theorem~\ref{sec3:thm:bla}]
Write $G_M=\sqrt M(\cc-\bb)$ and let $U$ be the exact pivot in \eqref{sec3:eq:exactpivot}. On identified realizations,
\[
\sqrt M(\chat-\bb)=\sqrt M\,\ee U+G_M.
\]
Both $\ee$ and $G_M$ are $\mathcal H$-measurable. Conditional on $\mathcal H$, $U$ is standard normal. Integrating over $\mathcal H$ conditional on $\Shat=\Sset$ gives the characteristic function, with $i^2=-1$,
\[
\E\!\left[\exp\{itG_M-t^2M\ee^2/2\}\mid\Shat=\Sset\right]+o(1),
\qquad t\in\R.
\]
The $o(1)$ allows an arbitrary extension on nonidentification events, whose probabilities vanish. The bounded factor
$\exp(-t^2M\ee^2/2)$ converges in probability, hence in mean, to
$\exp(-t^2\kappa_{j,\Sset}^2/2)$.
The bridge limit for $G_M$ therefore gives
\[
\sqrt M(\chat-\bb)\Rightarrow
N(0,\kappa_{j,\Sset}^2+\tau_{j,\Sset}^2)
\quad\text{conditional on }\Shat=\Sset.
\]
Moreover,
$M\sss^2\to_p\kappa_{j,\Sset}^2+\tau_{j,\Sset}^2>0$
by Proposition~\ref{sec3:prop:scale}. Slutsky's theorem proves \eqref{sec3:eq:blapivot}, with any extension of the pivot to unavailable realizations. The limiting normal distribution is continuous at the critical values, which gives coverage. For $0<u<1$ and $T_M=(\chat-\bb)/\sss$, the identity $\{2\Phi(-|T_M|)\le u\}=\{|T_M|\ge z_{1-u/2}\}$ gives the $p$-value statement. The theorem's numerical-check hypothesis and the retained-count argument above make failure probabilities vanish. Reporting $\R$ and $p=1$ on those events therefore leaves these limits unchanged.
\end{proof}

The proof uses bounded characteristic functions and requires neither finite-sample second moments of the ratio nor independence between the estimated scale and the numerator. The bridge variance $\tau_{j,\Sset}^2$ is the variance of the limiting linearization, not necessarily $M\Var(\cc-\bb)$ at finite $M$.

For a nonempty selected family of size $m_\Sset=|\Sset|$, let $p_j$ be the applicable IV or BLA $p$-value at its specified null value. For multiplicity adjustment, assign $p_j=1$ whenever inference is unavailable; undefined IV targets are excluded from the true-null set.
Bonferroni testing uses $p_j\le\alpha/m_\Sset$, and simultaneous intervals use $z_{1-\alpha/(2m_\Sset)}$ in place of $z_{1-\alpha/2}$. This factor counts all selected hypotheses; any unavailable hypotheses are left unrejected. The familywise error rate (FWER) is at most $\alpha$ for the identified IV family conditional on $\mathcal H$, and asymptotically at most $\alpha$ for a fixed BLA family conditional on $\Shat=\Sset$. Simultaneous IV coverage concerns only the available intervals, an $\mathcal H$-measurable subfamily; no coverage is claimed for an undefined IV target. BLA intervals include the $\R$ fallback on failures.

For simultaneous inference, let $\Sset_0$ be the true-null indices in a tested family of size $m_\Sset$. The union bound gives
\[
\Prb\{\exists j\in\Sset_0:p_j\le\alpha/m_\Sset\mid\mathcal H\}
\le\sum_{j\in\Sset_0}\Prb\{p_j\le\alpha/m_\Sset\mid\mathcal H\}
\le|\Sset_0|\alpha/m_\Sset\le\alpha
\]
for identified IV targets. For BLA targets under the fixed-support conditional law, each summand is $\alpha/m_\Sset+o(1)$. Since $m_\Sset$ is fixed, the limit superior is at most $\alpha$. Applying the union bound over available IV intervals, or over all selected BLA intervals, gives simultaneous coverage at least $1-\alpha$ in the corresponding exact or asymptotic sense. No independence across coordinate $p$-values is needed.

For testing this same selected family, Holm's step-down procedure \citep{holm1979} is an alternative. Order the $p$-values as $p_{(1)}\le\cdots\le p_{(m_\Sset)}$, and reject in sequence while $p_{(k)}\le\alpha/(m_\Sset-k+1)$, stopping at the first failed comparison.

If the family contains $m_0>0$ true nulls, its first rejected true null must have $p_j\le\alpha/m_0$. A union bound over those $m_0$ indices gives error probability at most $\alpha$ conditional on $\mathcal H$ for identified IV targets, and $\alpha+o(1)$ conditional on $\Shat=\Sset$ for a fixed BLA family. If $m_0=0$, a familywise error is impossible.

The fixed-support argument also permits pooling coverage over a fixed finite collection of eligible selected supports, provided the conditions hold for each support and each interval is assessed against that support's own target. The largest fixed-support approximation error then tends to zero, so any probability-weighted average does as well. A Type~I error experiment with null value zero additionally requires a zero target on every support included in that null experiment.
\subsection{Inference and bootstrap algorithms}
\label{sec3:app:algorithms}

Algorithm~\ref{sec3:alg:inference} combines independent selection, IV inference and BLA inference; Algorithm~\ref{sec3:alg:bootstrap} estimates the gap scale required by the last step. Both use known measurement variance, fixed thinning fractions and a prespecified polynomial library. The observation times are equally spaced, with $\Delta t=(t_M-t_1)/(M-1)$. We write $\Phi$ for the standard normal distribution function, $z_p=\Phi^{-1}(p)$, and $c_0\in\R$ for a null coefficient value. The vector $\mathbf1_M$ contains $M$ ones.

\begin{algorithm}[!htbp]
\caption[PSI-SINDy: selection and coefficient inference]{\phantomsection\label{sec3:alg:inference}PSI-SINDy: selection and coefficient inference}
\small
\begin{algorithmic}[1]
\Require Observations $\xobs=(x_1,\ldots,x_M)^\top$ at equally spaced times $t_1,\ldots,t_M$; known variance $\sigma_X^2$; weights $\psi_k$; library $\theta_j(u)=u^{j-1}$, $j\in[P]$; selection rule; fractions $a_i>0$ summing to one; level $\alpha$; bootstrap settings.
\Ensure Selected support, coefficient estimates, coefficientwise IV and BLA intervals and $p$-value functions of $c_0$, and availability indicators.
\State Independently of $\xobs$, draw mutually independent $g_i\sim N(0,a_i\sigma_X^2I_M)$, $i=1,\ldots,4$; set $g=\sum_i g_i$.
\State Set $V_i=\xobs+(g_i-a_i g)/a_i$ and $s_i=\sigma_X/\sqrt{a_i}$.
\State Set $(D_M)_{km}=-\psi'_k(t_m)\Delta t$ and $(H_M)_{km}=\psi_k(t_m)\Delta t$.
\State Apply the prespecified selection and tuning rule to $V_1$, using only $V_1$ and known quantities for any preprocessing; obtain $\Sset$.
\State \textbf{if} $\Sset$ is empty, \textbf{return} no coefficient hypotheses.
\State Set $Q=(s_2^2D_MD_M^\top)^{-1/2}$, $C_M=QD_M$, $B_M=QH_M$; require positive-definite response covariance.
\State Use $\mathcal C_{0,s}(v)=1$, $\mathcal C_{1,s}(v)=v$ and
  $\mathcal C_{d+1,s}(v)=v\mathcal C_{d,s}(v)-ds^2\mathcal C_{d-1,s}(v)$.
\State Set $\mathcal C_s(V)_{mj}=\mathcal C_{j-1,s}(V_m)$ and
  $\ytest=C_MV_2$, $\Wt=B_M\mathcal C_{s_3}(V_3)$, $\Zt=B_M\mathcal C_{s_4}(V_4)$.
\State Obtain gap scales $\widehat t_{j,\Sset}$ and their availability indicators from Algorithm~\ref{sec3:alg:bootstrap}.
\For{each $j\in\Sset$}
  \State Set $N=\Sset\setminus\{j\}$, $R=I_K-\Wt_N\Wt_N^+$, $z=R\Zt_{:,j}$, $d_{j,\Sset}=z^\top\Wt_{:,j}$.
  \If{the prescribed observed rank or denominator checks fail}
    \State Mark IV inference unavailable; report $I^{\mathrm{BLA}}_{j,\Sset}=\R$, $p^{\mathrm{BLA}}_{j,\Sset}=1$; continue.
  \EndIf
  \State Set $\widehat c=z^\top\ytest/d_{j,\Sset}$ and $e=\|z\|_2/|d_{j,\Sset}|$.
  \State Report $I^{\mathrm{IV}}_{j,\Sset}=[\widehat c-z_{1-\alpha/2}e,\widehat c+z_{1-\alpha/2}e]$ and
  $p^{\mathrm{IV}}_{j,\Sset}(c_0)=2\Phi(-|\widehat c-c_0|/e)$.
  \If{the gap scale is available}
    \State Set $\widehat s=\sqrt{e^2+\widehat t_{j,\Sset}^{\,2}}$.
    \State Report $I^{\mathrm{BLA}}_{j,\Sset}=[\widehat c-z_{1-\alpha/2}\widehat s,\widehat c+z_{1-\alpha/2}\widehat s]$ and
    $p^{\mathrm{BLA}}_{j,\Sset}(c_0)=2\Phi(-|\widehat c-c_0|/\widehat s)$.
  \Else
    \State Report $I^{\mathrm{BLA}}_{j,\Sset}=\R$ and $p^{\mathrm{BLA}}_{j,\Sset}=1$.
  \EndIf
\EndFor
\end{algorithmic}
\end{algorithm}

\begin{algorithm}[!htbp]
\caption[Bootstrap estimation of the IV--BLA gap scale]{\phantomsection\label{sec3:alg:bootstrap}Bootstrap estimation of the IV--BLA gap scale}
\small
\begin{algorithmic}[1]
\Require Inference views {$V_2$, $V_3$, and $V_4$} at times $t_1,\ldots,t_M$; polynomial library $\theta_j(u)=u^{j-1}$, $j\in[P]$; fixed selected support $\Sset$; fractions $a_i$ and standard deviations $s_3,s_4$; maps $B_M,C_M$; prespecified local-cubic bandwidth $h_M$; draw count $\Bboot$; fixed minimum retained count $n_{\min}\ge2$; rank and denominator thresholds.
\Ensure Nonnegative gap scales $\widehat t_{j,\Sset}$, or an unavailable indicator.
\If{$\Sset=\{1\}$ contains only the constant feature}
  \If{$B_M\mathbf1_M\ne0$}
    \State \Return $\widehat t_{1,\{1\}}=0$.
  \Else
    \State Mark the gap scale unavailable and \textbf{return}.
  \EndIf
\EndIf
\State Set $\overline V=(a_2V_2+a_3V_3+a_4V_4)/(1-a_1)$.
\State Fit the specified local cubic smoother to $\overline V$; obtain $\xpilot$ and set $\widehat{\bar{\boldsymbol y}}=C_M\xpilot$.
\State Form the pilot design without added measurement noise, $\widehat{\bar W}=B_M\Theta(\xpilot)$, where $\Theta(\xpilot)_{mj}=\widehat x_m^{\,j-1}$.
\State \textbf{if} $\widehat{\bar W}_{\Sset}$ fails the prescribed rank check, mark all gap scales unavailable and \textbf{return}.
\State Initialize an empty list $\mathcal L_j$ for each $j\in\Sset$.
\For{$\ell=1,\ldots,\Bboot$}
  \State Independently draw $V_3^\boot\sim N(\xpilot,s_3^2I_M)$ and $V_4^\boot\sim N(\xpilot,s_4^2I_M)$.
  \State Correct each degree using $\mathcal C_{0,s}=1$, $\mathcal C_{1,s}(v)=v$ and
  $\mathcal C_{d+1,s}(v)=v\mathcal C_{d,s}(v)-ds^2\mathcal C_{d-1,s}(v)$.
  \State Form $\Wt^\boot=B_M\mathcal C_{s_3}(V_3^\boot)$ and
  $\Zt^\boot=B_M\mathcal C_{s_4}(V_4^\boot)$, with feature $j$ using degree $j-1$.
  \For{each $j\in\Sset$}
    \State Set $N=\Sset\setminus\{j\}$, $R^\boot=I_K-\Wt_N^\boot(\Wt_N^\boot)^+$,
    $z^\boot=R^\boot\Zt^\boot_{:,j}$, $d_{j,\Sset}^\boot=(z^\boot)^\top\Wt^\boot_{:,j}$.
    \If{the selected and nuisance rank checks pass and $|d_{j,\Sset}^\boot|>\delta_M$}
      \State Append $(z^\boot)^\top\widehat{\bar{\boldsymbol y}}/d_{j,\Sset}^\boot$ to $\mathcal L_j$.
    \EndIf
  \EndFor
\EndFor
\For{each $j\in\Sset$}
  \If{$|\mathcal L_j|\ge n_{\min}$}
    \State Set $\widehat t_{j,\Sset}=(\widehat Q_j(3/4)-\widehat Q_j(1/4))/(2z_{3/4})$, where $\widehat Q_j(p)$ is the linearly interpolated empirical quantile at position $1+(n-1)p$ in the sorted retained sample, with $n=|\mathcal L_j|$; retain a zero value.
  \Else
    \State Mark the gap scale unavailable.
  \EndIf
\EndFor
\end{algorithmic}
\end{algorithm}
For an empty nuisance set, the column-space projection is zero and the residualizer is $I_K$; its rank check passes. Observed IV rank and denominator checks use only the corrected design and instrument. The rank checks require full column rank of $\Wt_\Sset$ and $[\Wt_N,\Zt_{:,j}]$ for each $j\in\Sset$, with $N=\Sset\setminus\{j\}$. For a prespecified tolerance $0\le\eta_M<1$, the observed denominator check requires finite $d_{j,\Sset}$, a finite positive product $\|\zp\|_2\|\wperp\|_2$, and $|d_{j,\Sset}|>\eta_M\|\zp\|_2\|\wperp\|_2$. Pilot and bootstrap checks may use the inference views, but failure of a gap-scale check does not invalidate an otherwise available IV interval. The asymptotic threshold conditions are stated in Appendix~\ref{sec3:app:bootstrap}. All bootstrap draws use fresh randomness independent of selection. Algorithm~\ref{sec3:alg:inference} returns coefficientwise intervals and tests; simultaneous inference requires the separate adjustment in Section~\ref{sec3:bla}.
\subsection{Calibration conditional on feature inclusion}
\label{full:app:inclusion}
Let $\mathfrak M$ be a fixed finite family containing every possible selected support, fix a feature index $j$, and set $\mathfrak M_j=\{\Sset\in\mathfrak M:j\in \Sset\}$ and $\mathcal J_{M,j}=\{j\in\Shat\}$. Each support retains its own target $\xi^\star_{j,\Sset}$; pooling does not make those coefficients numerically identical.
\begin{appendixproposition}[Inclusion-conditional calibration]
\label{full:prop:inclusion}
Assume $\Prb(\mathcal J_{M,j})>0$. No positive lower bound on the inclusion probability is required. If every support in $\mathfrak M_j$ has almost-sure IV identification, the true-value IV $p$-value in \eqref{sec3:eq:ivci} is uniform conditional on $\mathcal J_{M,j}$. Assigning $p=1$ when numerical checks make IV inference unavailable gives a superuniform law instead. Conditional also on IV availability, the law remains exactly uniform whenever that conditioning event has positive probability. If the fixed-support BLA conditions of Theorem~\ref{sec3:thm:bla} hold for every support in this finite family, then the corresponding BLA interval has limiting coverage $1-\alpha$ and, for $0<u<1$,
\[
\Prb\{p^{\mathrm{BLA}}_{j,\Shat}(\xi^\star_{j,\Shat})\leq u\mid\mathcal J_{M,j}\}\longrightarrow u.
\]
\end{appendixproposition}
\begin{proof}
The weights $w_M(\Sset)=\Prb(\Shat=\Sset\mid\mathcal J_{M,j})$ sum to one. Integrating the exact IV law within each support and then summing with these weights gives $\sum_{\Sset\in\mathfrak M_j}w_M(\Sset)u=u$. Write $\mathcal E_{M,j}$ for IV availability. Since $\mathcal E_{M,j}\in\mathcal H$, the reported true-value $p$-value satisfies, for $0<u<1$,
\[
\Prb\{p^{\mathrm{IV,rep}}_{j,\Shat}\le u\mid\mathcal J_{M,j}\}
=u\Prb(\mathcal E_{M,j}\mid\mathcal J_{M,j})\le u.
\]
Integrating the conditional uniform law over $\mathcal J_{M,j}\cap\mathcal E_{M,j}$ proves the availability-conditional assertion.

For BLA inference, construct the hypothetical statistic for each deterministic support using only the inference views and independent bootstrap randomness. It is independent of the selection information. Its true-value rejection probability is $u+\epsilon_M(\Sset)$, where $\epsilon_M(\Sset)\to0$ by the fixed-support theorem. For every support event of positive probability, conditioning on that event leaves this law unchanged. Therefore the absolute error of the pooled rejection probability is bounded by
\[
\sum_{\Sset\in\mathfrak M_j}w_M(\Sset)|\epsilon_M(\Sset)|
\leq\max_{\Sset\in\mathfrak M_j}|\epsilon_M(\Sset)|\longrightarrow0.
\]
The last step uses finiteness of the family. The same maximum argument makes the probabilities of unavailable BLA inference vanish. Test inversion gives the interval statement.
\end{proof}
For a pooled reference-null experiment, the specified value must equal the current true BLA separately for every possible selected support. Testing another trajectory's reference coefficient measures power when that reference value differs from the current support-specific BLA. Proposition~\ref{full:prop:inclusion} justifies pooling under its assumptions; pooled empirical calibration alone cannot establish calibration in a rare support.

\subsection{Comparator calibration and model-count Bonferroni inference}
\label{full:app:comparator}
The comparator in Section~\ref{sec4:design} replaces the testing response by $Y^{(1)}=C_MV_1$. Fix a deterministic support $\Sset$ and coefficient $j\in\Sset$, set $N=\Sset\setminus\{j\}$, and define $z=(I_K-\Wt_N\Wt_N^+)\Zt_{:,j}$, $d=z^\top\Wt_{:,j}$ and $\boldsymbol\ell=z/d$ on identified realizations. The response $Y^{(1)}$ is Gaussian and independent of $\mathcal F=\sigma(V_3,V_4)$. Conditional on $\mathcal F$, the contrast is fixed, and \[ \widehat c^{\mathrm N}_{j,\Sset}=\boldsymbol\ell^\top Y^{(1)},\qquad c^\star_{j,\Sset}=\boldsymbol\ell^\top\ybar,\qquad e^{\mathrm N}_{j,\Sset}=(s_1/s_2)\|\boldsymbol\ell\|_2. \]
The factor $s_1/s_2$ follows from $s_2^2C_MC_M^\top=I_K$. The naive comparator of \citet[Section~6]{matsukawa2025} likewise uses a classical Gaussian test without conditioning on selection.

\begin{appendixproposition}[Fixed-support distributions of the comparator statistics]
\label{full:prop:comparator}
For deterministic $\Sset$, $(\widehat c^{\mathrm N}_{j,\Sset}-c^\star_{j,\Sset})/e^{\mathrm N}_{j,\Sset}$ is standard normal conditional on $\mathcal F$ on identified realizations. Under the fixed-support bridge and bootstrap assumptions,
\[
\frac{\widehat c^{\mathrm N}_{j,\Sset}-\xi^\star_{j,\Sset}}
{\sqrt{(e^{\mathrm N}_{j,\Sset})^2+\widehat t_{j,\Sset}^2}}\Rightarrow N(0,1).
\]
These statements do not condition on the event $\Shat=\Sset$.
\end{appendixproposition}
\begin{proof}
Independence of $V_1$ and $(V_3,V_4)$ gives the first law by applying the fixed contrast $\boldsymbol\ell$ to $Y^{(1)}\sim N(\ybar,(s_1/s_2)^2I_K)$. For the second, decompose the scaled estimation error into the response error and $G_M=\sqrt M(c^\star_{j,\Sset}-\xi^\star_{j,\Sset})$. Conditioning first on $\mathcal F$ and then integrating gives the characteristic-function argument in Appendix~\ref{sec3:app:coverage}, with $M(e^{\mathrm N}_{j,\Sset})^2\to_p(s_1/s_2)^2/q$. The bridge contributes limiting variance $\tau_{j,\Sset}^2$. Scale consistency and Slutsky's theorem complete the proof, even though the estimated scale can depend on the other views. Under the stated assumptions, identification failure has vanishing probability. Failures of the numerical checks also have vanishing probability when the checks satisfy the conditions in Appendix~\ref{sec3:app:bootstrap}.
\end{proof}
For equal allocations, $e^{\mathrm N}_{j,\Sset}=\ee$ and both methods use the same BLA scale. Their estimates differ through the response view.

For the primary quadratic comparison, the Naive estimator therefore has the ratio form
\[
\widehat c^{\mathrm N}_{3,\Sset}
=\frac{(\widetilde z^\perp_{3,\Sset})^\top C_MV_1}{d_{3,\Sset}}.
\]
It retains the corrected design and instrument, and equal thinning fractions give the selection and testing responses the same covariance. For target type $a\in\{\mathrm{IV},\mathrm{BLA}\}$, its two-sided $p$-value at null value $b\in\R$ is
\begin{equation}
p^{\mathrm N,a}_{\Sset}(b)=2\Phi\!\left(-\frac{|\widehat c^{\mathrm N}_{3,\Sset}-b|}{s^a_{3,\Sset}}\right),
\label{sec4:eq:naive}
\end{equation}
where $s^{\mathrm{IV}}_{3,\Sset}=e_{3,\Sset}$ and $s^{\mathrm{BLA}}_{3,\Sset}=\widehat s_{3,\Sset}$.

The argument below applies these fixed-support laws to the finite family of possible selected models. The resulting model-count protection averages over selection. It differs from Bonferroni or Holm adjustment of several PSI-SINDy hypotheses within one selected support, as proved in Appendix~\ref{sec3:app:coverage}.

\paragraph{Error control of the model-count Bonferroni procedure.}
\label{sec4:app:bon}

Model-count Bonferroni uses
\begin{equation}
p^{\mathrm{Bon},a}_{\Sset}(b)=\min\{1,4p^{\mathrm N,a}_{\Sset}(b)\},
\label{sec4:eq:bon}
\end{equation}
counting the four supports containing the intercept and quadratic, with the linear and cubic terms optional. Appendix~\ref{full:app:comparison} also assesses multiplier eight.

For this comparison, the support family is
\[
\mathfrak M_3=\bigl\{\{1,3\},\{1,2,3\},\{1,3,4\},\{1,2,3,4\}\bigr\}.
\]
For each deterministic $\Sset\in\mathfrak M_3$, Proposition~\ref{full:prop:comparator} gives $\Prb\{p^{\mathrm N,\mathrm{IV}}_\Sset(c^\star_{3,\Sset})\le u\mid\mathcal F\}=u$ on identified realizations, where $\mathcal F=\sigma(V_3,V_4)$ and $0<u<1$.
For an unidentified hypothetical model, the rejection event is false without evaluating an undefined target. With this convention, its conditional rejection probability is at most $u$.

Although the selected support depends on $V_1$, a rejection at its true realized target is contained in the union of the four fixed-model rejection events. Consequently,
\begin{equation}
\begin{split}
&\Prb\!\left\{3\in\Shat,\
p^{\mathrm{Bon},\mathrm{IV}}_{\Shat}(c^\star_{3,\Shat})\le\alpha\mid\mathcal F\right\}\\
&\quad\le\sum_{\Sset\in\mathfrak M_3}
\Prb\!\left\{p^{\mathrm N,\mathrm{IV}}_\Sset(c^\star_{3,\Sset})\le\alpha/4\mid\mathcal F\right\}
\le\alpha.
\end{split}
\label{sec4:eq:bonproof}
\end{equation}
This bound averages over selection and response noise conditional on $\mathcal F$. It does not divide by the probability of selecting a particular support or the quadratic term. Thus it does not imply a bound of $\alpha$ for the conditional rates plotted in Section~4.

For the BLA comparison, the corresponding statement is asymptotic. If the fixed-support rank, pilot, and scale conditions of Theorem~\ref{sec3:thm:bridge} and Proposition~\ref{sec3:prop:scale} hold for every model in $\mathfrak M_3$, then each fixed-model Naive statistic, centered at its BLA and divided by its total scale, converges to a standard normal distribution. Proposition~\ref{full:prop:comparator} combines the independent Gaussian response error with the bridge limit and scale consistency. Summing the four fixed-model error probabilities gives
\begin{equation}
\limsup_{M\to\infty}\Prb\!\left\{3\in\Shat,\
p^{\mathrm{Bon},\mathrm{BLA}}_{\Shat}(\xi^\star_{3,\Shat})\le\alpha\right\}\le\alpha.
\label{sec4:eq:bonbla}
\end{equation}
This asymptotic bound averages over selection. It provides neither support-conditional BLA control nor a power ordering relative to PSI-SINDy.

\section{Simulation protocols and results}
\label{sec4:appendix}\label{sec4:app:protocol}
This appendix documents the primary selected-BLA study, supporting realized-IV calibration, and a separate comparison with coefficient-bootstrap uncertainty for Section~\ref{sec4:experiments}; Appendix~\ref{full:app:simulations} gives the bridge approximation diagnostics.
The primary study, support-conditioned study, and additional simulation studies use separate cohorts. The following sections specify the assessment events, numerical conventions and Monte Carlo uncertainty for the reported results.

\subsection{Primary simulation protocol}
The primary study generates 5,000 independent datasets in each of eight cells: the reference trajectory at $M\in\{200,800,1600,3200\}$ and the four positive signals of Section~\ref{sec4:design}. The scalar library is $(1,x,x^2,x^3)$, the interval is $[0,8]$, and the observation grid is $t_m=8(m-1)/(M-1)$. Twelve weak weights are
\begin{equation}
\psi_k(t)=\left[1-\{(t-c_k)/0.75\}^2\right]_+^4,
\qquad c_k=0.8+(k-1)(7.2-0.8)/11.
\label{sec4:eq:weights}
\end{equation}
Here $[u]_+=\max(u,0)$. Whitening uses the known response covariance, so $s_2^2C_MC_M^\top=I_K$. All four thinning allocations are equal. Selection uses only $V_1$: least-squares refitting on retained corrected columns, threshold $\lambda_M=0.02\sqrt{800/M}$ in original coefficient units, mandatory intercept, and at most 20 iterations. A separate selection-only pilot used 1,000 null datasets at each $M$ and inspected support counts only. The threshold allows assessment of quadratic-term selection under the reference trajectory; we do not claim that it is optimal for model recovery.

The trajectory pilot fits local cubic polynomials to $(V_2+V_3+V_4)/3$, using windows of 21, 61, 105 and 183 observations at the four observation counts. Boundary values use the first or last full window. Bootstrap budgets are 799, 1,199, 1,599 and 2,399; all positive-signal cells use 1,199. Each draw holds the selected support fixed, regenerates corrected design and instrument, and recomputes the nuisance projector. The estimated gap scale is the interquartile range of the bootstrap IV ratios divided by $2\Phi^{-1}(3/4)$.

The smallest-to-largest eigenvalue ratio of the response covariance must exceed $10^{-12}$, and matrix rank is assessed with relative singular-value tolerance $10^{-11}$. The primary implementation checks nuisance rank, a nonzero residualized instrument, and a finite nonzero denominator. The absolute bootstrap denominators exceed $10^{-12}/s_2^2$ in fully whitened units; at least 100 usable draws are required. The equivalent covariance-shape whitening uses cutoff $10^{-12}$ and retains $s_2$ in the response scale; coefficients, scales and ratios are invariant to that common rescaling. All 30,437 selected primary focal cases were identified, had positive estimated gap scales and retained every requested draw. In the primary simulations, IV inference is available whenever the observed design and instrument checks pass. BLA inference additionally requires a finite positive estimated gap scale; otherwise, the confidence interval is $\mathbb R$ and the $p$-value is $1$. Algorithm~\ref{sec3:alg:inference} permits a zero estimated gap scale and checks full selected rank, augmented instrument rank and a relative denominator tolerance, as specified in Appendix~\ref{sec3:app:algorithms}.

The accompanying code supplement provides a reproducibility manifest mapping each reported figure and table to its implementation, configuration, random-number settings, and regeneration command. Its selected-support inference routine separates IV availability from gap-scale availability and permits a zero gap scale, as in Algorithm~\ref{sec3:alg:inference}.

\subsection{Assessment denominators and Monte Carlo uncertainty}
\label{sec4:app:counts}\label{scope:app:records}
In tables and method subscripts, PSI denotes PSI-SINDy and Bon. (or Bon) denotes Bonferroni.
For replication $r$, let $\Shat_r$ denote the selected support, $A_r$ indicate focal inclusion and $R_{r,v}$ rejection by method $v$. Set $R_{r,v}=0$ when the focal term is absent. With $N_A=\sum_r A_r$ and $R_{\rm MC}$ generated datasets,
\begin{equation}
\widehat p_{\rm cond,v}=\frac{\sum_r A_rR_{r,v}}{N_A},\qquad
\widehat p_{\rm joint,v}=\frac{\sum_r A_rR_{r,v}}{R_{\rm MC}}.
\label{sec4:eq:rates}
\end{equation}
Conditional rates are undefined at $N_A=0$. Summaries conditional on a selected support instead use $N_{\Sset}=\sum_r\mathbf1_{\{\Shat_r=\Sset\}}$. Calibration evaluates the test at each support's true BLA coefficient, whereas reference-null rejection tests the reference trajectory's coefficient within that same support. The calibration and reference-null values agree only when the reference null is true. IV checks evaluate each identified realized target separately. Unavailable BLA inference returns no rejection and the confidence interval $\R$; such intervals remain in the coverage assessment when the target is assessable. For conditional IV rates, the sum and denominator are restricted to focal selections with available IV inference; joint rates retain $R_{\rm MC}$ as their denominator. Nonunique latent BLAs are excluded from BLA assessments, and excluded cases are counted separately. Finite-width medians exclude whole-line intervals, whose frequency must also be reported.

Rates are reported with pointwise 95\% Wilson intervals \citep{wilson1927}. These quantify Monte Carlo error, not coefficient uncertainty or simultaneous error over the displayed grid. For paired method differences, set $D_r=R_{r,\rm PSI}-R_{r,\rm Bon}$ on the selected datasets. For $N_A>1$, define
\[
\overline D=\frac{1}{N_A}\sum_{r:A_r=1}D_r,\qquad
s_D^2=\frac{1}{N_A-1}\sum_{r:A_r=1}(D_r-\overline D)^2.
\]
With $z_{\rm MC}=\Phi^{-1}(0.975)$, the approximate paired interval is $\overline D\pm z_{\rm MC}s_D/\sqrt{N_A}$; it is undefined for $N_A\le1$. Replications, not coefficients within a replication, are the independent sampling units.

\subsection{Primary numerical results}
\label{sec4:iv}\label{sec4:main-calibration}
Table~\ref{emp:tab:primary} reports primary calibration at each target and median coefficient-interval widths. Rejection at the true target and coverage are complementary events; their Wilson intervals transform accordingly. For PSI-SINDy, all pointwise Monte Carlo intervals for IV and BLA coverage in these four primary cells contain 95\%. This evaluates finite-sample calibration; the exact-IV and asymptotic BLA guarantees are established separately in Section~\ref{sec3:method}. Table~\ref{sec4:tab:power} reports selected-BLA reference-null power and the corresponding joint selection-and-rejection frequencies.

\begin{table}[!htb]\setlength{\tabcolsep}{4pt}
\caption{Primary quadratic calibration, conditional on selection. Both panels report rejection percentages with pointwise 95\% Wilson intervals at the stated true target. Widths are median coefficient-interval widths. Bonferroni uses multiplier four.}
\label{emp:tab:primary}\label{sec4:tab:ivcal}\label{sec4:tab:blacal}\label{sec4:tab:coverage}
\begin{center}\small
\begin{tabular}{rrccc}\toprule
& & \multicolumn{3}{c}{Selected-BLA target}\\
\cmidrule(lr){3-5}
$M$ & $N_A$ & Naive & Bonferroni & PSI-SINDy\\\midrule
200&3425&7.36 [6.53,8.28]&1.84 [1.44,2.35]&4.41 [3.77,5.15]\\
800&3360&7.86 [6.99,8.82]&2.02 [1.60,2.56]&5.39 [4.67,6.20]\\
1600&3423&7.36 [6.53,8.29]&1.61 [1.24,2.09]&4.91 [4.23,5.68]\\
3200&3452&8.69 [7.80,9.68]&1.74 [1.35,2.23]&5.27 [4.58,6.07]\\
\midrule
& & \multicolumn{3}{c}{Realized-IV target}\\
\cmidrule(lr){3-5}
$M$ & $N_A$ & Naive & Bonferroni & PSI-SINDy\\\midrule
200&3425&7.85 [7.00,8.80]&1.64 [1.26,2.12]&5.05 [4.37,5.84]\\
800&3360&7.98 [7.11,8.94]&2.05 [1.63,2.59]&5.57 [4.84,6.39]\\
1600&3423&7.04 [6.23,7.95]&1.75 [1.36,2.25]&4.94 [4.26,5.71]\\
3200&3452&7.94 [7.08,8.89]&1.51 [1.15,1.97]&4.92 [4.25,5.70]\\
\bottomrule\end{tabular}
\par\smallskip
\begin{tabular}{rcccc}\toprule
$M$&200&800&1600&3200\\\midrule
PSI median IV width&0.3690&0.1841&0.1301&0.0921\\
PSI median BLA width&0.3846&0.1913&0.1352&0.0957\\\bottomrule
\end{tabular}
\end{center}
\end{table}

\begin{table}[!htbp]
\caption{Selected-BLA power. The Bonferroni and PSI-SINDy cells give rejection counts, conditional power percentages and pointwise 95\% Monte Carlo intervals, with denominator $N_A$. The last two columns give joint selection-and-rejection percentages with denominator 5,000.}\label{sec4:tab:power}
\begin{center}\small
\setlength{\tabcolsep}{5pt}
\renewcommand{\arraystretch}{1.12}
\begin{tabular}{rrcccc}
\toprule
$\gamma$ & $N_A$ & Bonferroni & PSI-SINDy & Bon. joint & PSI joint \\
\midrule
0.02 & 3540 & \shortstack{101\\2.85 [2.35, 3.45]} & \shortstack{239\\6.75 [5.97, 7.63]} & 2.02 & 4.78 \\
0.04 & 3901 & \shortstack{210\\5.38 [4.72, 6.14]} & \shortstack{460\\11.79 [10.82, 12.84]} & 4.20 & 9.20 \\
0.08 & 4440 & \shortstack{718\\16.17 [15.12, 17.28]} & \shortstack{1280\\28.83 [27.52, 30.18]} & 14.36 & 25.60 \\
0.16 & 4896 & \shortstack{2306\\47.10 [45.70, 48.50]} & \shortstack{3175\\64.85 [63.50, 66.17]} & 46.12 & 63.50 \\
\bottomrule
\end{tabular}
\end{center}
\end{table}

Primary PSI-minus-Bonferroni paired power differences are 3.90, 6.41, 12.66 and 17.75 percentage points, respectively. These differences use the selected denominators; the marginal guarantee for model-count protection is stated in Appendix~\ref{sec4:app:bon}. At $M=200$ under the reference trajectory, the two selections of $\Sset=\{1,3,4\}$ have quadratic BLA approximately $0.13977$, not zero; the reference-null definition retains their correct targets. The pooled results predominantly assess supports containing the linear term and do not establish precise calibration for rare supports.

\subsection{Comparison with coefficient-bootstrap methods}
\label{app:external-comparison}

\paragraph{Purpose and setup.}
This supplementary comparison examines whether coefficient-bootstrap intervals provide calibrated tests and coverage for a dynamical term reported after selection. The evaluation target is the structural quadratic coefficient $\gamma$ in \eqref{sec4:eq:ode}, rather than the selected finite-grid BLA used in the primary study. We use $\gamma=\delta=0$, $M=800$, and independent Gaussian measurement errors with known $\sigma_X=0.06$. The library is $(1,x,x^2,x^3)$, and PSI-SINDy retains the primary study's weak weights and equal thinning allocations. Unlike the primary selection protocol, PSI selection always retains $1$ and $x$, with threshold 0.02 for the remaining terms. Thus the PSI-selected support contains all nonzero dynamical terms in this benchmark.

\paragraph{Compared procedures.}
The Ensemble-SINDy data-bagging adapters use cubic Savitzky--Golay smoothing, finite-difference derivatives, paired time-row resampling, and sparse least-squares refits \citep{fasel2022}. SG5 uses the five-point smoothing window from the official uncertainty-quantification example; the longer-SG adapter uses a prespecified 61-point window. Both use a sparse-fit threshold of 0.02 and the full candidate library without mandatory terms, ridge regularization, or library bagging. Their nominal 95\% coefficient intervals are empirical percentile intervals of all 999 bootstrap draws, including zero coefficients.

The residual-bootstrap adapter uses the longer-SG preprocessing and the PSI-selected support. Motivated by \citet{gao2023}, it constructs studentized residual-bootstrap intervals from 999 draws, with centered residuals, resampling-pool scaling $\sqrt{n/(n-q)}$, and residual-variance divisor $n-q$, where $n=M$ is the number of time rows and $q$ is the fitted dimension. These preprocessing, support, and interval conventions are explicit adaptations. External adapters receive the original noisy trajectory. PSI-SINDy uses its BLA interval with 1,199 bootstrap draws to estimate the IV--BLA gap scale. This structural-target assessment is empirical: in this benchmark, finite-grid quadrature error can make the selected BLA coefficient differ from $\gamma$.

\paragraph{Assessment.}
All methods are evaluated on the same 1,382 of 2,000 paired datasets for which PSI selects $x^2$. Ensemble fits need not select the same support as PSI. We reject the structural null $H_0:\gamma=0$ when the nominal 95\% coefficient interval excludes zero, and assess coverage by whether that interval contains $\gamma=0$. Both rates use the common selected denominator. At this null, coverage is therefore the complement of rejection. All intervals in this comparison are finite. Pointwise 95\% Wilson intervals quantify Monte Carlo uncertainty in the rejection rates.

\begin{table}[!htbp]
\setlength{\tabcolsep}{5pt}
\caption{Focal structural-null comparison at $M=800$, $\sigma_X=0.06$, and $\gamma=\delta=0$. All methods are assessed on the same 1,382 datasets, out of 2,000, in which PSI selects the quadratic term. Type~I error, structural-coefficient coverage, and pointwise 95\% Wilson Monte Carlo intervals are expressed as percentages. Monte Carlo intervals refer to Type~I error. Width is the median finite coefficient-interval width.}
\label{tab:external-focal}
\begin{center}\small
\begin{tabular}{@{}lrrrr@{}}\toprule
Method & Type~I error & \shortstack{95\% MC interval\\for Type~I error} & \shortstack{Structural\\coverage} & \shortstack{Median\\width}\\\midrule
PSI-SINDy & 4.70 & [3.71, 5.95] & 95.30 & 0.1914\\
Ensemble, SG5 & 0.00 & [0.00, 0.28] & 100.00 & 2.9573\\
Ensemble, longer SG & 8.10 & [6.78, 9.66] & 91.90 & 0.1675\\
Residual, matched support & 20.48 & [18.43, 22.69] & 79.52 & 0.1676\\
\bottomrule\end{tabular}
\end{center}
\end{table}

\paragraph{Interpretation.}
In Table~\ref{tab:external-focal}, PSI-SINDy's observed rejection is compatible with the nominal 5\% level, while its intervals are substantially shorter than those of the conservative SG5 ensemble. The longer-SG ensemble and matched residual adapter produce narrower intervals but undercover in this setting. These results illustrate why coefficient variability should be assessed through both calibration and interval width.

Performance varies across settings. For example, at $M=200$ and $\sigma_X=0.12$, the longer-SG ensemble has 5.00\% null rejection and median width 0.5707, compared with 3.45\% and 0.8086 for PSI-SINDy. The accompanying code reproduces the complete null-grid results and exact-BLA sensitivity analyses. The comparison concerns the specified application adapters and structural target; it neither establishes a uniform ordering nor tests the cited methods' theoretical guarantees outside their assumptions.
\FloatBarrier

\section{Bridge diagnostics and supplementary inference results}
\label{full:app:simulations}\label{sec4:app:bridge}
\subsection{Support-conditioned calibration and familywise error}
\label{app:support-calibration}
The supplementary cohort has 28,000 independent datasets, 2,000 per setting, under
$\dot x=-0.25-0.35x+\gamma x^2+\delta\sin(2x)$, $x(0)=1.2$, on $[0,8]$.
Eight calibration cells use $\gamma=0.20$, $\sigma_X=0.06$, $\delta\in\{0,0.05\}$ and $M\in\{200,800,1600,3200\}$. The system labels Polynomial and Sine-augmented denote $\delta=0$ and $\delta=0.05$, respectively. Five further signals are $\gamma\in\{0,0.05,0.10,0.15,0.30\}$ at $M=800$, $\delta=0$, and one additional ablation cell has $\sigma_X=0.12$. The $\gamma=0.20,M=800$ cell is reused for power and ablation, not counted twice. Selection retains the intercept and uses STLSQ threshold 0.18 in original coefficient units. Weak weights and pilot windows match the primary protocol; bootstrap counts are 399, 799, 1,199 and 1,599. In fourteen datasets the selected model was constant-only; these use the exact zero gap scale. All 72,807 selected-coordinate assessments are retained, with no BLA fallback, bootstrap rank failure or discarded draw.

The prespecified focal groups assess the slope: $j=2$, $\Sset=\{1,2,3\}$ for Polynomial and $\Sset=\{1,2\}$ for Sine-augmented. Table~\ref{tab:expandedprimary} retains all eight groups. To assess familywise error, we test every selected coordinate at its own true BLA and record whether any null hypothesis is rejected in each dataset. In this all-true-null evaluation, Holm and Bonferroni have identical any-rejection events; simultaneous Bonferroni coverage is their complement. These are within-selected-family adjustments, distinct from the model-count comparator.

Table~\ref{tab:expandedprimary} reports selected-BLA slope coverage and median interval width for these exact-support events. Polynomial coverage is 92.92\%--94.70\%, and sine-augmented coverage is 94.50\%--95.65\%. At $M=200$, the polynomial group has 420/452 coverage, with pointwise MC interval [90.18\%, 94.94\%]; the other seven intervals contain 95\%. Coverage need not improve monotonically over a finite grid, and these results do not establish the asymptotic limit. From $M=200$ to $3200$, median widths decrease from 0.11091 to 0.02769 for the polynomial group and from 0.07656 to 0.01908 for the sine-augmented group. Width comparisons across the groups concern different selected targets. For the sine-augmented group at $M=200$, the familywise rejection rate is 6.98\% (47/673; pointwise 95\% Monte Carlo interval [5.29\%, 9.16\%]), indicating a finite-sample departure from the nominal 5\% level. These intervals are not adjusted for comparisons across the displayed groups.

\begin{table}[!htb]\setlength{\tabcolsep}{4pt}
\caption{All prespecified support-conditioned groups. Coverage and FWER entries are percentages with pointwise 95\% Wilson intervals. $n=N_\Sset$; median width is in coefficient units. Simultaneous Bonferroni coverage is $100-\mathrm{FWER}$.}
\label{tab:expandedprimary}\label{presentation:tab:coverage-details}\label{tab:expandedfamilies}
\begin{center}\small
\begin{tabular}{llrcrrc}\toprule
System&$M$&$n$&Slope BLA coverage&Width&Family events&FWER\\\midrule
Polynomial&200&452&92.92 [90.18,94.94]&0.11091&14&3.10 [1.85,5.13]\\
Polynomial&800&812&94.70 [92.94,96.04]&0.05536&32&3.94 [2.81,5.51]\\
Polynomial&1600&1046&94.55 [93.00,95.77]&0.03917&45&4.30 [3.23,5.71]\\
Polynomial&3200&1234&94.17 [92.72,95.34]&0.02769&40&3.24 [2.39,4.38]\\
Sine-augmented&200&673&94.50 [92.51,95.99]&0.07656&47&6.98 [5.29,9.16]\\
Sine-augmented&800&1096&94.62 [93.12,95.80]&0.03818&54&4.93 [3.80,6.37]\\
Sine-augmented&1600&1219&95.65 [94.36,96.66]&0.02699&50&4.10 [3.12,5.37]\\
Sine-augmented&3200&1290&94.57 [93.20,95.68]&0.01908&69&5.35 [4.25,6.71]\\\bottomrule
\end{tabular}
\end{center}
\end{table}

\subsection{Bridge approximation and bootstrap scale accuracy}
\label{app:bridge-accuracy}
For $A_M=\bar W^\star/\sqrt M$, $b_M=\bar{\boldsymbol y}^\star/\sqrt M$, a full-rank selected design, and $N=\Sset\setminus\{j\}$, set
\begin{equation}
\begin{aligned}
\beta_M&=A_{M,\Sset}^{+}b_M,&r_M&=b_M-A_{M,\Sset}\beta_M,\\
R_M&=I_K-A_{M,N}A_{M,N}^{+},&v_M&=R_MA_{M,:,j},\\
q_M&=v_M^\top v_M,&\boldsymbol\pi_M&=A_{M,N}^{+}A_{M,:,j}.
\end{aligned}\label{sec4:eq:grid}
\end{equation}
Here $\beta_M$ is the selected finite-grid BLA. With realized feature errors $E_3=\widetilde W-\bar W^\star$ and $E_4=\widetilde Z-\bar W^\star$, the first-order scaled gap is
\begin{equation}
L_{M,j,\Sset}=\{r_M^\top(E_{4,:,j}-E_{3,N}\boldsymbol\pi_M)-v_M^\top E_{3,\Sset}\beta_M\}/q_M.
\label{sec4:eq:LM}
\end{equation}
The term involving $E_{3,N}\boldsymbol\pi_M$ retains nuisance-projector uncertainty. We calculate $\tau_{M,j,\Sset}^2=\Var(L_{M,j,\Sset})$ using the known trajectory and exact corrected-feature covariance, solely for assessment. It is not the variance of the nonlinear IV ratio. The normalized remainder and scale ratio are
\begin{equation}
\frac{|\sqrt M(c^\star_{j,\Sset}-\xi^\star_{j,\Sset})-L_{M,j,\Sset}|}{\tau_{M,j,\Sset}},
\qquad \frac{\sqrt M\widehat t_{j,\Sset}}{\tau_{M,j,\Sset}}.
\label{sec4:eq:diagnostics}
\end{equation}
All oracle scales used in these diagnostics are positive. Figure~\ref{full:fig:bridge} assesses the complete first-order expansion in the primary quadratic cohort using \eqref{sec4:eq:diagnostics}. From $M=200$ to $3200$, the median normalized remainder decreases from 0.0607 to 0.0156 and the median estimated-to-oracle scale ratio changes from 1.0195 to 1.0002. The oracle scale is the standard deviation of the scaled first-order term; finite-sample IV-ratio moments are not required. The reported intervals use the estimated inference-view pilot, never the oracle trajectory. These median diagnostics assess approximation and scale accuracy, whereas the coverage results additionally assess tail calibration. Table~\ref{tab:expandedcomponents} gives complementary support-conditioned diagnostics: a scale ratio near one does not alone establish tail accuracy, absence of bias, or nominal coverage.

\begin{figure}[!htb]
\centering\includegraphics[width=0.64\linewidth]{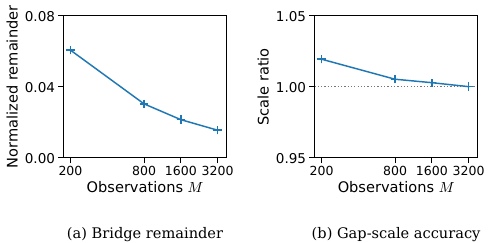}
\caption{Primary IV--BLA bridge accuracy conditional on selecting $x^2$, with $\gamma=\delta=0$. (a) Median absolute first-order approximation error divided by the oracle scale. (b) Median estimated-to-oracle gap-scale ratio; reference one. In \eqref{sec4:eq:diagnostics}, the oracle scale is the standard deviation of the scaled first-order gap. Both $M$ axes are logarithmic. Table~\ref{tab:expandedcomponents} gives complementary support-conditioned diagnostics.}
\label{full:fig:bridge}
\end{figure}

\begin{table}[!htb]

\caption{IV and bridge diagnostics for the same slope and exact-support groups as Table~\ref{tab:expandedprimary}, with the same assessment counts. Brackets give pointwise 95\% Wilson Monte Carlo intervals for IV rejection. Here $\widehat t=\widehat t_{j,\Sset}$ and $\tau=\tau_{M,j,\Sset}$ denote the estimated gap scale and oracle standard deviation of the scaled first-order gap for the stated coordinate and support. The scale ratio is the median $\sqrt M\widehat t/\tau$; the remainder is the median absolute scaled linearization error divided by $\tau$. SD denotes the empirical standard deviation of the studentized BLA statistic.}\label{tab:expandedcomponents}
\begin{center}\footnotesize
\begin{tabular}{lrrcrr}\toprule
System & $M$ & IV rejection (\%) & Scale ratio & Remainder & SD \\\midrule
Polynomial & 200 & 6.19 [4.32, 8.81] & 0.986 & 0.0267 & 1.005 \\
Polynomial & 800 & 5.42 [4.06, 7.20] & 0.996 & 0.0154 & 1.010 \\
Polynomial & 1600 & 6.12 [4.82, 7.74] & 0.999 & 0.0103 & 1.003 \\
Polynomial & 3200 & 4.86 [3.80, 6.21] & 0.998 & 0.0074 & 1.008 \\
Sine-augmented & 200 & 5.35 [3.89, 7.32] & 0.998 & 0.0145 & 1.005 \\
Sine-augmented & 800 & 5.02 [3.88, 6.48] & 0.999 & 0.0076 & 1.002 \\
Sine-augmented & 1600 & 4.02 [3.05, 5.27] & 1.000 & 0.0055 & 0.971 \\
Sine-augmented & 3200 & 5.66 [4.52, 7.06] & 0.998 & 0.0037 & 1.031 \\
\bottomrule\end{tabular}
\end{center}
\end{table}

Table~\ref{tab:bridge-regimes} compares the primary regime with the two calibration failures highlighted in Section~\ref{sec4:experiments}. Both failures have much larger normalized remainders. The median estimated-to-oracle scale ratio exceeds one for the pendulum, so median scale underestimation does not explain both failures. These oracle diagnostics identify limitations of the first-order approximation; they neither establish a unique cause of undercoverage nor provide a calibration check from observed data alone. The pendulum calculation applies the same algebraic expansion equationwise, without asserting the scalar asymptotic theorem for multistate observations.

\begin{table}[!htb]\setlength{\tabcolsep}{5pt}
\caption{Bridge diagnostics in the primary regime and two cases of undercoverage. Each row conditions on inclusion of the stated term and retains each support's own target. $n$ is the selected count; remainder and scale ratio are medians of \eqref{sec4:eq:diagnostics}. Primary and sine rows use $\sigma_X=0.06$; the sine drift is $\dot x=-2\sin x$. The pendulum row uses its second equation and $\sigma_X=0.05$. Coverage is expressed as a percentage.}
\label{tab:bridge-regimes}
\begin{center}\small
\begin{tabular}{llrrcrr}\toprule
Study & Term & $M$ & $n$ & BLA coverage & Remainder & Scale ratio\\\midrule
Primary & $x^2$ & 200 & 3425 & 95.59 & 0.0607 & 1.0195\\
Primary & $x^2$ & 3200 & 3452 & 94.73 & 0.0156 & 1.0002\\
Sine drift & $x^2$ & 800 & 2989 & 87.82 & 0.8341 & 0.9232\\
Pendulum & $xy^2$ & 1600 & 275 & 85.82 & 0.9562 & 1.4048\\
\bottomrule\end{tabular}
\end{center}
\end{table}

\subsection{Component ablations and sensitivity analyses}
The matched ablation in the main text uses common observations, selected support, views and bootstrap randomness. Table~\ref{tab:paired-ablation-uncertainty} reports all twelve paired component contrasts for the six noise-level--coordinate groups. At $\sigma_X=0.12$, removing feature correction raises intercept rejection by 5.53 percentage points [2.73,8.33]. At $\sigma_X=0.06$, removing the gap scale raises intercept rejection by 2.09 percentage points [1.11,3.08]. In the selected basis $\{1,x,x^2\}$, Gaussian feature correction changes the quadratic feature only by a constant shift, which is absorbed by the included intercept. Thus the slope and quadratic decisions are unchanged in this ablation; this mechanism is specific to the stated basis and does not make feature correction generally irrelevant.

\begin{table}[!htb]
\setlength{\tabcolsep}{3pt}
\caption{Matched component assessment at $M=800$ and $\Sset=\{1,2,3\}$. Here $N_\Sset$ is the number of datasets selecting exactly $\Sset$. The full-method column reports rejection at the true selected BLA (nominal 5\%). The last columns give ablated-minus-full changes in percentage points with approximate paired 95\% Monte Carlo intervals. Shared measurements, selection, views and bootstrap randomness give matched comparisons. A dagger marks unchanged decisions because, for the selected basis $\{1,x,x^2\}$, the intercept absorbs the quadratic feature correction.}\label{tab:paired-ablation-uncertainty}
\begin{center}\small
\begin{tabular}{rrrccc}\toprule
$\sigma_X$ & $j$ & $N_\Sset$ & \shortstack{Full-method\\rejection (\%)} & \shortstack{Without gap\\scale} & \shortstack{Without feature\\correction}\\\midrule
0.06 & 1 & 812 & 5.42 & \shortstack{$+2.09$\\$[1.11,3.08]$} & \shortstack{$+0.86$\\$[-0.72,2.44]$} \\
0.06 & 2 & 812 & 5.30 & \shortstack{$+2.22$\\$[1.20,3.23]$} & $0^{\dagger}$ \\
0.06 & 3 & 812 & 5.42 & \shortstack{$+0.86$\\$[0.23,1.50]$} & $0^{\dagger}$ \\
0.12 & 1 & 452 & 5.53 & \shortstack{$+1.77$\\$[0.55,2.99]$} & \shortstack{$+5.53$\\$[2.73,8.33]$} \\
0.12 & 2 & 452 & 6.86 & \shortstack{$+1.99$\\$[0.70,3.28]$} & $0^{\dagger}$ \\
0.12 & 3 & 452 & 5.53 & \shortstack{$+1.33$\\$[0.27,2.38]$} & $0^{\dagger}$ \\
\bottomrule\end{tabular}
\end{center}
\end{table}

\paragraph{Additional component comparisons.}\label{app:component-comparisons}
A separate six-cell component study uses 30,000 datasets, with $M\in\{200,800\}$, $\sigma_X\in\{0.06,0.12,0.24\}$, and 5,000 datasets per cell. Table~\ref{tab:ext-components} reports intercept results at $M=800$ for the two larger noise scales. The study compares full PSI with variants that omit feature correction, share the response and design views, or omit the separate instrument; raw, whitened, corrected and selection-view ordinary least squares (OLS); and response reuse with and without model-count adjustments. Arms share observations but can select different supports. Multiplier-four and multiplier-eight $p$-values are capped at one; the support-count variant uses eight for the mandatory intercept and four for other coordinates. The no-gap intervals retain the full method's coefficient estimate and response scale. No pointwise MC interval for full-method BLA coverage lies entirely below 95\%; no-gap coefficient intervals undercover in several cells. Oracle bridge summaries are not assigned to shared-design or no-instrument arms.
Raw OLS uses an unweighted projection target, whereas whitened methods use the covariance-weighted BLA. Shared response--design views violate response independence; using the corrected design as instrument changes the bridge assumptions. These comparisons therefore assess each arm at its own target and are distinct from the matched ablation above.

\begin{table}[!htb]\setlength{\tabcolsep}{5pt}
\caption{Additional intercept assessment at $M=800$, pooled over selected supports, $n=5000$ per noise level. Entries are percentages with pointwise 95\% Wilson intervals. OLS denotes ordinary least squares after whitening. The full and no-gap intervals target the same selected BLA.}
\label{tab:ext-components}\label{tab:ext-gap-coverage}
\begin{center}\footnotesize
\begin{tabular}{rccc}\toprule
&\multicolumn{3}{c}{Rejection at the true selected BLA}\\
\cmidrule(lr){2-4}
$\sigma_X$&Full&Uncorrected&OLS\\\midrule
0.12&4.76 [4.20,5.39]&6.74 [6.08,7.47]&9.64 [8.85,10.49]\\
0.24&4.46 [3.92,5.07]&5.82 [5.20,6.50]&11.98 [11.11,12.91]\\
\bottomrule\end{tabular}
\par\smallskip
\begin{tabular}{rcc}\toprule
&\multicolumn{2}{c}{Selected-BLA coverage}\\
\cmidrule(lr){2-3}
$\sigma_X$&Full&No gap\\\midrule
0.12&95.24 [94.61,95.80]&92.74 [91.99,93.43]\\
0.24&95.54 [94.93,96.08]&92.00 [91.22,92.72]\\
\bottomrule\end{tabular}
\end{center}
\end{table}

\paragraph{Comparator sensitivity.}\label{full:app:comparison}\label{sec4:app:additional}
The fixed-threshold comparator analysis reuses 26,000 supplementary datasets, excluding the high-noise ablation cell; it is not an independent cohort. With threshold 0.18 and multiplier eight, Bonferroni has higher observed reference-null power at four of five positive signals. At $\gamma=0.10$, its count is $165/177$ (93.22\% [88.52,96.08]) versus PSI's $84/177$ (47.46\% [40.23,54.79]). At zero signal no quadratic is selected among 2,000 datasets, so conditional reference-null calibration is undefined, not zero. These comparisons neither establish equal conditional Type~I error nor isolate the changes in threshold, signal grid and multiplier. Replacing multiplier four by eight on the primary $p$-values lowers Bonferroni power from 2.85--47.10\% to 1.47--38.64\%; PSI-SINDy still has higher observed power in the primary comparison.

\paragraph{Support-specific calibration departures.}\label{app:port-support-departures}
Across the eight prespecified calibration cells, we assessed 110 support--coordinate groups and 38 selected families. Six pointwise MC intervals for empirical coverage have upper endpoints below 95\%, including the prespecified polynomial $M=200$ slope and five further groups. For example, the polynomial $M=1600$, $\Sset=\{1,2\}$ intercept has $755/810$ coverage, 93.21\% [91.27,94.75]; the rare $\Sset=\{1,2,4\}$ cubic has $69/77$, 89.61\% [80.82,94.64]. The scan is descriptive and not multiplicity-adjusted.

\section{Inference with larger candidate libraries}\label{ext:app:largerp}
This paired study distinguishes support recovery, selected-coefficient calibration, and interval informativeness. The library size $P$ counts candidate scalar functions, not state variables. The results do not establish a growing-$P$ theorem or uniform weak-identification robustness for the BLA target.
\paragraph{Paired simulation design.}
The study generates 3,000 independent datasets: 500 for each combination of
$M\in\{800,3200\}$ and the three dynamics
\[
 x'(t)=-0.25-0.35x(t)+\gamma x(t)^d,\qquad x(0)=1.2,\quad t\in[0,8].
\]
The null case has $\gamma=0$, the quadratic case has $(\gamma,d)=(0.2,2)$,
and the quartic case has $(\gamma,d)=(0.2,4)$.
Each dataset is evaluated at $P\in\{4,6,8,10,15,20\}$ with both fixed selectors,
giving 36,000 paired evaluations, rather than 36,000 independent datasets.
The scalar library is $\{1,x,\ldots,x^{P-1}\}$, so increasing $P$ also increases
polynomial degree. Measurement-noise standard deviation is $0.06$; each of the
four equally allocated independent Gaussian views has standard deviation
$s_i=0.12$. The generation seed is 2026092501. The selectors and protocol were fixed before generating these independent evaluation datasets, with no outcome-dependent stopping or switching between selectors.

\paragraph{STLSQ variants and selection criterion.}
The prespecified integrated variant (I) uses a cumulative-integral
design constructed from a fixed cubic Savitzky--Golay smoother of $V_1$.
The smoothing windows are 61 and 183 observations at $M=800$ and $3200$,
respectively, with endpoint interpolation. Hermite correction uses the exactly
propagated marginal smoothing variance. This correction does not remove bias
in the smoothed mean. Cumulative trapezoidal integration gives the feature
matrix $F$; centering $F$ and $V_1$ accounts for an unpenalized trajectory offset
during candidate generation. The matched weak-design variant (W) uses the
corrected weak $V_1$ design and weak response, with $K=12$.

For each variant, write its selection design and response as $W_1$ and $Y_1$,
and let $A_1$ be the design after normalizing each column to unit Euclidean norm.
On the current support $\Sset$, STLSQ refits
\[
 \widehat{\boldsymbol b}^{\mathrm{sel}}_{\Sset}(\lambda_{\rm ridge})
 \in\mathop{\rm argmin}_{\boldsymbol b}\bigl\{
 \|Y_1-A_{1,\Sset}\boldsymbol b\|_2^2+\lambda_{\rm ridge}\|\boldsymbol b\|_2^2\bigr\},
 \qquad
 \lambda_{\rm ridge}\in\{0\}\cup\{10^{-10+12k/64}:k=0,\ldots,64\}.
\]
It removes coefficients whose normalized magnitude is below a common threshold $\lambda_{\rm sel}$,
then refits until the support stabilizes. The drift constant is penalized in
ridge fitting but is never removed. Every computed nonnegative scalar-threshold
path interval is examined, with the same fixed ridge grid at every $P$.
There is no cardinality cap. When $P>K$, the zero-ridge full weak fit uses the
minimum-norm least-squares solution; positive ridge gives a unique refit.

Both variants choose among their own STLSQ-generated supports using
\[
 \operatorname{Score}(\Sset)=\min_{b_0,\boldsymbol b}
 \frac{\|V_1-b_0\mathbf1_M-F_{\Sset}\boldsymbol b\|_2^2}{s_1^2}
 +( |\Sset|+1 )\log M.
\]
The offset $b_0$ is distinct from the drift constant. This is a plug-in selection
score, not a likelihood-calibrated Bayesian information criterion (BIC), because $F$ is constructed from $V_1$.
Candidates are ordered by score, support size, lexicographic support, ridge strength, and the first threshold-path value, in that order. Numerically rank-deficient score designs
are assigned infinite score at relative singular-value tolerance $10^{-10}$.

\paragraph{Inference and numerical implementation.}
The other three independent views provide the testing response, corrected design
and instrument. No ridge penalty is used in the IV ratio, its uncertainty or the
selected-BLA target. The weak maps remain at $K=12$ and the existing pilot and
IQR bootstrap scale are retained, with 1,199 draws at $M=800$ and 2,399 at $M=3200$.
For numerical stability, the implementation normalizes columns before QR factorization, residualizes twice, and screens dimensionless moment denominators using only the realized design and instrument. Its relative rank tolerance is $10^{-11}$.
These numerical changes preserve the estimands in exact arithmetic, although
the implementation differs from the primary implementation.
For assessment, the noiseless selected weak-design BLA is solved by
column-normalized least squares, with relative rank tolerance $10^{-12}$.
Unassessable targets are excluded from coverage denominators and counted separately. Unavailable BLA inference returns the whole-real-line interval, which remains in coverage assessment when the target is assessable.

\paragraph{Assessment denominators.}
Exact support recovery and focal inclusion use all 500 generated datasets.
The focal term is $x^2$ in the null and quadratic cases and $x^4$ in the quartic
case. Focal coverage is conditional on its selection and on an assessable
noiseless selected-BLA target, pooled across the selected supports. It does not
assert separate calibration for every support. IV coverage is conditional on
focal selection and an identified realized-IV target. BLA calibration tests use the true finite-grid selected-BLA coefficient, without substituting a physical coefficient or an assumed zero; IV checks use the realized-IV coefficient.
Reported Monte Carlo intervals use the Wilson method described in Appendix~\ref{sec4:app:counts}. These denominators separate focal inclusion, target availability and coefficient calibration.

Table~\ref{app:tab:largerp-compact} reports the complete integrated-selector signal grids together with their selected denominators and support sizes.

\begin{table}[!htb]
\setlength{\tabcolsep}{2.5pt}
\caption{Complete quadratic and quartic grids for the integrated ridge-STLSQ selector. Size is the selected-support mean/maximum over all 500 datasets; $n$ counts focal selection ($x^2$ for quadratic dynamics and $x^4$ for quartic dynamics). Coverage is a percentage with its pointwise 95\% Wilson interval, conditional on focal selection and pooled across supports. Width is the median selected-BLA interval width on the same assessed datasets. The quartic focal term is absent from the library at $P=4$, so its coverage and width are undefined. There are no unavailable intervals or unassessable selected targets in these cells.}
\label{app:tab:largerp-compact}
\begin{center}\footnotesize
\begin{tabular}{@{}lrcrcrcrcr@{}}\toprule
& & \multicolumn{4}{c}{$M=800$} & \multicolumn{4}{c}{$M=3200$}\\
\cmidrule(lr){3-6}\cmidrule(l){7-10}
Dynamics & $P$ & Size & $n/500$ & Coverage [MC] & Width & Size & $n/500$ & Coverage [MC] & Width\\\midrule
Quadratic & 4 & 3.010/4 & 473/500 & 95.77 [93.56,97.25] & 0.1615 & 3.004/4 & 500/500 & 95.40 [93.19,96.92] & 0.0806\\
 & 6 & 3.034/4 & 434/500 & 95.85 [93.54,97.36] & 0.1614 & 3.002/4 & 495/500 & 95.56 [93.36,97.05] & 0.0806\\
 & 8 & 3.052/5 & 426/500 & 95.54 [93.14,97.13] & 0.1614 & 3.018/4 & 487/500 & 95.48 [93.26,97.00] & 0.0806\\
 & 10 & 3.054/5 & 426/500 & 95.77 [93.42,97.31] & 0.1614 & 3.014/4 & 489/500 & 95.50 [93.28,97.01] & 0.0806\\
 & 15 & 3.056/5 & 424/500 & 95.75 [93.39,97.30] & 0.1614 & 3.016/4 & 488/500 & 95.49 [93.27,97.00] & 0.0806\\
 & 20 & 3.054/5 & 426/500 & 95.54 [93.14,97.13] & 0.1614 & 3.016/4 & 488/500 & 95.49 [93.27,97.00] & 0.0806\\
\midrule
Quartic & 4 & 3.000/3 & 0/500 & --- & --- & 3.366/4 & 0/500 & --- & ---\\
 & 6 & 3.118/4 & 264/500 & 95.45 [92.22,97.38] & 0.2147 & 3.014/4 & 435/500 & 94.02 [91.39,95.89] & 0.1069\\
 & 8 & 3.080/5 & 249/500 & 95.58 [92.26,97.52] & 0.2154 & 3.010/4 & 415/500 & 94.46 [91.82,96.28] & 0.1069\\
 & 10 & 3.090/4 & 248/500 & 94.76 [91.24,96.91] & 0.2155 & 3.010/4 & 418/500 & 94.50 [91.88,96.31] & 0.1069\\
 & 15 & 3.066/5 & 237/500 & 94.09 [90.33,96.45] & 0.2154 & 3.010/4 & 415/500 & 94.46 [91.82,96.28] & 0.1069\\
 & 20 & 3.064/5 & 236/500 & 94.49 [90.80,96.75] & 0.2147 & 3.012/4 & 410/500 & 94.39 [91.72,96.23] & 0.1068\\
\bottomrule\end{tabular}
\end{center}
\end{table}

\begin{figure}[!htbp]
\centering
\includegraphics[width=.94\textwidth]{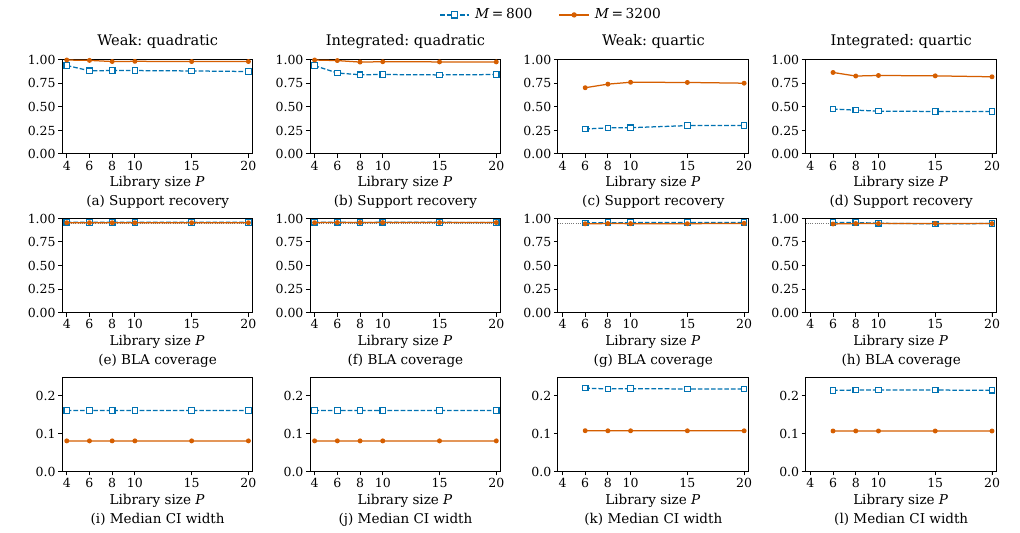}
\caption{Columns compare weak and integrated ridge-STLSQ selectors for quadratic and quartic dynamics; rows show exact support recovery, focal selected-BLA coverage, and median interval width. The focal term is $x^2$ for quadratic dynamics and $x^4$ for quartic dynamics. Both selectors and all library sizes reuse 500 datasets per dynamics/$M$ cell. Recovery uses all datasets; coverage and width condition on focal selection and pool selected supports. Dotted lines mark nominal 0.95 coverage. Quartic $P=4$ points are omitted: recovery is zero and focal inference is undefined because $x^4$ is absent. Table~\ref{app:tab:largerp-compact} gives integrated-selector counts and Monte Carlo intervals.}
\label{fig:larger-library-comparison}\label{fig:larger-library}
\end{figure}

\paragraph{Support recovery and inferential performance.}
At $M=3200$, integrated-selector quartic recovery is 83.2\%, 82.8\% and 81.8\%
at $P=10,15,20$, respectively. The corresponding focal BLA coverage is
94.5\%, 94.5\% and 94.4\%; at $P=20$ this is $387/410$, with Wilson interval
$[91.7\%,96.2\%]$ and median interval width 0.1068.
Weak-design recovery in these cells is 76.0\%, 75.8\% and 75.0\%.
Quartic support recovery remains difficult at the smaller sample size: at $M=800,P=20$, recovery is
30.0\% (W) and 45.0\% (I), whereas integrated-selector focal coverage is
$223/236=94.5\%$. Thus coverage on selected coordinates does not measure
successful discovery across all datasets. The integrated variant is not
uniformly better: quadratic recovery at $M=3200,P=20$ is 98.0\% (W) and 97.6\% (I).
For quadratic dynamics, each selector has its own $P=4$ baseline in Figure~\ref{fig:larger-library-comparison}.

There are no selection failures, unavailable selected-coordinate intervals,
unassessable selected targets or incomplete bootstrap coordinates in this cohort.
Accordingly, finite-interval and all-assessable-target coverage coincide here.
The largest selected support has six terms, although no size cap is imposed; at $P=20$ the maximum is five.
Every $P=20$ dataset has an STLSQ candidate with finite selection score
containing a term of degree 10 or above. The full candidate library can be ill-conditioned or, for the weak
variant at $P>12$, underdetermined; the selected supports need not share that
full-library rank defect.
These observations do not establish noninferiority as $P$ increases.
Across the null cells and both selectors, $x^2$ was selected in only 1--5 of the 500 datasets per cell. The corresponding zero-rejection estimates are therefore too
imprecise to establish conditional 5\% Type~I calibration.

\paragraph{Separate study with unregularized STLSQ.}\label{app:port-larger-p}
A separate study uses the dynamics in \eqref{sec4:eq:ode} with $\gamma=0.20$, $\delta=0$, $x(0)=1.2$ and $t\in[0,8]$, and uses the unregularized weak-design selector at $P\in\{4,5,6,7,8,10\}$, $M\in\{800,3200\}$, $K=12$ and $\sigma_X=0.06$, with 3,000 datasets per cell. For $M=800$ and $3200$, the thresholds are 0.02 and 0.01, and the bootstrap draw counts are 1,199 and 2,399, respectively. These $P$ values are all below $K$. Every selected target was assessable; no invalid coordinate or whole-line interval was recorded. No pointwise MC interval for full-method BLA coverage lay entirely below 95\%, but interval widths increased: at $M=800$, the quadratic median BLA width rose from 0.305 at $P=4$ to 18.3 at $P=10$, while rejection of zero fell from $2251/2974$ to $81/2988$. Selected supports approached the full library and condition numbers grew. This is distinct from exclusions caused by $|\Sset|>K$ or deficient selected rank.

\FloatBarrier
\section{Coordinatewise experiments on multistate systems}\label{ext:app:systems}
\subsection{Construction and simulation settings}
The equations below are evaluated separately with a shared candidate-function family for each system. Noise is independent across observation times and state coordinates, with Gaussian standard deviation $\sigma_X\in\{0.01,0.05\}$ in every coordinate. Four views have equal allocations $a_i=1/4$. Let $d_{\mathrm{state}}$ denote the number of state coordinates and $\mathbb N_0=\{0,1,\ldots\}$ the nonnegative integers. Let $\boldsymbol\nu\in\mathbb N_0^{d_{\mathrm{state}}}$ be a monomial exponent vector, and let $V_{i,\ell,m}$ denote the entry of view $i$ for state coordinate $\ell$ at time $t_m$. The corrected monomial feature in view $i$ is
\[
 \prod_{\ell=1}^{d_{\mathrm{state}}}
 \mathcal C_{\nu_\ell,s_i}(V_{i,\ell,m}),
 \qquad s_i^2=4\sigma_X^2,
\]
where $\mathcal C_{k,s_i}$ is the univariate corrected polynomial at the noise scale of view $i$. Independent coordinate errors imply that its expectation is the corresponding noise-free monomial. This product correction specifies the empirical multistate construction; it does not assert a multistate version of the scalar dense-grid BLA theorem.

For each equation, STLSQ uses only the first view and retains the constant. The other three views provide the testing response, corrected design, and instrument. The selected-model residualization and the combined response-plus-gap scale are applied equationwise. The observation-noise, thinning, and bootstrap random streams are equation-specific, so the experiment does not assess simultaneous inference across different equations of one noisy vector observation. The assessments below concern selected coefficients within individual equations.

All systems use $M\in\{400,800,1600\}$ equally spaced observations, with 2,000 datasets per cell and master seed 20260918. The bootstrap draw counts are 999, 1199, and 1599, respectively. The library contains every monomial up to the stated total degree, including the constant. Weak-weight radii and selection thresholds are fixed by system. The local cubic pilot uses windows of 35, 61 and 105 observations for $M=400,800,1600$, respectively. On each observation interval $[t_1,t_M]$, the $K$ weak-weight centers are equally spaced from $t_1+R+0.05$ to $t_M-R-0.05$, where $R$ is the radius listed below.

\begin{center}\small
\begin{tabular}{lrrrrrr}
\toprule
System & States & Degree & $P$ & $K$ & Radius & Threshold\\
\midrule
Van der Pol & 2 & 3 & 10 & 20 & 0.60 & 0.10\\
Lotka--Volterra & 2 & 2 & 6 & 24 & 0.75 & 0.02\\
Lorenz & 3 & 2 & 10 & 30 & 0.25 & 0.20\\
Pendulum & 2 & 3 & 10 & 30 & 0.35 & 0.10\\
\bottomrule
\end{tabular}
\end{center}
The dynamics are \begin{align*} \text{Van der Pol:}\quad&\dot x=y,\quad\dot y=-x+(1-x^2)y;\\ \text{Lotka--Volterra:}\quad&\dot x=x-0.1xy,\quad\dot y=-1.5y+0.075xy;\\ \text{Lorenz:}\quad&\dot x=10(y-x),\quad\dot y=28x-y-xz,\quad\dot z=xy-\tfrac83z;\\ \text{Pendulum:}\quad&\dot x=y,\quad\dot y=-9.81\sin x. \end{align*}
The initial states and final times are $(2,0)$ and $T=10$ for Van der Pol, $(10,5)$ and $T=15$ for Lotka--Volterra, $(-8,7,27)$ and $T=5$ for Lorenz, and $(2,0)$ and $T=10$ for the pendulum; each interval starts at $t=0$.
The pendulum's second equation is outside its finite polynomial library. Its inferential target is therefore the selected BLA, not a coefficient of an exact polynomial governing equation.

\subsection{Metrics and assessment denominators}\label{app:systems-metrics}
For each equation and term, $n$ counts datasets selecting that term and is the denominator for the rates in Table~\ref{ext:tab:systems-main}. Both full and no-gap intervals are assessed against the selected-BLA coefficient for each realized support; IV rejection is evaluated at the realized-IV target. Rejection at the true selected BLA is the complement of coverage.

Rates are reported with pointwise 95\% Wilson intervals, including rarely selected coordinates; no minimum selection count is imposed. Never-selected terms have zero selection frequency and undefined conditional rates. All 328,110 recorded selected-coordinate assessments were available and returned finite intervals: invalid, unassessable, unavailable-IV, and whole-line counts are zero. Diagnostics were recorded for all 108,000 equation-level fits from 48,000 simulation replications. Every selected coordinate retained every requested draw, with no bootstrap rank failure or invalid fit. In particular, pendulum $xy^2$ coverage is 236/275 (85.82\%) at $M=1600$ and $\sigma_X=0.05$, despite complete retention. Discarded bootstrap draws therefore do not explain this undercoverage; Table~\ref{tab:bridge-regimes} instead documents a large linearization remainder.

\begin{table}[!htb]
\setlength{\tabcolsep}{3pt}
\caption{Multistate assessment at $M=1600$, $\sigma_X=0.05$. BLA and No gap give selected-BLA coverage; IV gives rejection at the realized-IV target. Entries are percentages with pointwise 95\% Wilson intervals. $n$ counts selected datasets. VdP: Van der Pol; LV: Lotka--Volterra; parentheses identify the equation.}
\label{ext:tab:systems-main}
\begin{center}\footnotesize
\begin{tabular}{@{}llrccc@{}}\toprule
Equation & Term & $n$ & BLA & No gap & IV\\\midrule
VdP (2)&$1$&2000&94.85 [93.79,95.74]&13.85 [12.41,15.43]&4.25 [3.45,5.23]\\
VdP (2)&$x^2y$&2000&95.30 [94.28,96.14]&75.45 [73.52,77.29]&5.25 [4.36,6.32]\\
LV (1)&$xy$&2000&94.65 [93.58,95.55]&56.55 [54.37,58.71]&4.45 [3.63,5.44]\\
Lorenz (2)&$1$&2000&94.00 [92.87,94.96]&15.00 [13.50,16.63]&4.80 [3.95,5.83]\\
Lorenz (2)&$xz$&2000&98.00 [97.29,98.53]&55.50 [53.31,57.67]&4.90 [4.04,5.94]\\
Pendulum (2)&$x$&2000&90.50 [89.14,91.71]&78.75 [76.90,80.49]&4.50 [3.68,5.50]\\
Pendulum (2)&$xy^2$&275&85.82 [81.20,89.45]&60.36 [54.48,65.96]&6.18 [3.90,9.68]\\
\bottomrule\end{tabular}
\end{center}
\end{table}

\paragraph{Calibration results and limitations.}
Table~\ref{ext:tab:systems-main} reports $M=1600$, $\sigma_X=0.05$; the following comparisons summarize the six sample-size/noise cells for each system. Gap adjustment improves coverage in the displayed coordinates, but does not uniformly yield nominal BLA coverage. Across all six Lorenz cells, rejection at the true selected BLA for the second-equation intercept ranges from 6.00\% to 9.05\%, with all six pointwise Wilson intervals above 5\%; the $xz$ coefficient instead has conservative BLA coverage. For the misspecified pendulum equation, the $x$ coefficient has rejection between 6.50\% and 11.50\%, again with all six intervals above 5\%. At $M=1600$, $\sigma_X=0.05$, the pendulum $xy^2$ term has BLA coverage 85.82\% among 275 selections, compared with 60.36\% without the gap scale. Higher coverage than the no-gap variant does not establish nominal calibration.

\section{Sensitivity analyses and inference availability}
\label{app:port-experiments}\label{ext:app:port}
Table~\ref{emp:tab:grids} lists the sensitivity grids, each evaluated at level 0.05 with seed 20260918. Unless varied, the scalar library has $P=4$, $K=12$, weak radius 0.75, equal thinning allocations, and drift $\dot x=-0.25-0.35x+0.2x^2$, $x(0)=1.2$ on $[0,8]$. Coordinatewise rates use the number of datasets selecting that coordinate, pooled over supports, as denominator. Familywise rates use datasets for which all selected coefficients are assessable; all generated datasets satisfied this condition in these grids. No setting was chosen by inspecting calibration outcomes. Pointwise BLA-undercoverage and IV flags are descriptive, not multiplicity-adjusted tests over the grid.

\begin{table}[!htb]
\caption{Complete sensitivity grids. $R_{\rm MC}$ is the number of independent datasets per cell, $\sigma_{\rm work}$ is the supplied working noise standard deviation, and $c$ multiplies the drift in the signal sweep. Except for tuning, the threshold is $0.02\sqrt{800/M}$. For $M=200,800,1600,3200$, bootstrap counts are 799,1199,1599,2399, respectively.}
\label{emp:tab:grids}
\begin{center}\small
\begin{tabularx}{\linewidth}{@{}lrrX@{}}\toprule
Study&Cells&$R_{\rm MC}$&Settings varied\\\midrule
Tuning&13&4000&Threshold $0.01,0.02,0.05,0.10,0.18$; $K=8,12,16,24$; radius $0.5,0.75,1$; allocations specified below. $M=800$, $\sigma_X=0.06$.\\
Noise scale&8&4000&$\sigma_{\rm work}/\sigma_X=0.5,0.75,0.9,1,1.1,1.25,1.5,2$; $M=800$, $\sigma_X=0.06$.\\
Misspecification&27&4000&Sine drift, exponential, logistic; $M=200,800,1600$ and $\sigma_X=0.02,0.06,0.12$.\\
Misspecified signal&12&3000&Sine and exponential drifts, $c=0.25,0.5,0.75,1,1.5,2$; $M=800$, $\sigma_X=0.06$.\\
Components&6&5000&$M=200,800$ and $\sigma_X=0.06,0.12,0.24$; all arms retained.\\
Identification&12&3000&$M=200,800,3200$ and $\sigma_X=0.03,0.06,0.12,0.24$.\\\bottomrule
\end{tabularx}
\end{center}
\end{table}

\paragraph{Sensitivity to procedure settings.}\label{app:port-tuning}
One setting changes at a time from threshold 0.02, $K=12$, radius 0.75 and equal allocation. The three unequal allocations are $(0.4,0.2,0.2,0.2)$, $(0.1,0.3,0.3,0.3)$ and $(0.25,0.45,0.15,0.15)$. All cells use 1,199 bootstrap draws. No pointwise MC interval for full-method BLA coverage lies entirely below 95\%.

\paragraph{Sensitivity to the supplied noise scale.}\label{app:port-sigma-sensitivity}
An incorrect noise scale affects thinning, feature correction, selection and both uncertainty components; it does not merely resize the final interval. For twelve full-method coordinate cells, MC intervals for rejection at the true selected BLA lie above 5\%. The largest rate is 15.51\% [14.32,16.79] for $x^3$ at scale ratio 0.5 ($n=3301$). The corresponding IV rejection rate is 15.24\% [14.05,16.50]. These observations concern violation of the assumption of a known, correct noise scale and do not isolate bridge error.

\paragraph{Model misspecification and signal strength.}\label{app:port-misspec}\label{app:port-misspec-power}
Sine and exponential drifts are $\dot x=-c\sin x$, $x(0)=2$, and $\dot x=-0.5cxe^{-x}$, $x(0)=1.5$. Neither belongs to the cubic library. The logistic control, $\dot x=0.8x(1-x/2)$, $x(0)=0.2$, belongs to the library. Calibration fixes $c=1$; the signal sweep varies $c$ on the grid above. The sine results in Table~\ref{tab:bridge-regimes} are from the signal sweep at $c=2$. The calibration rates reported for both grids evaluate rejection at each selected model's true BLA coefficient; IV checks use the realized-IV target. Seven coordinate cells in the calibration grid have MC intervals for BLA rejection above 5\%, including sine, exponential and logistic cells. The largest is 6.43\% [5.70,7.24] for $x^2$ in the sine cell $M=1600$, $\sigma_X=0.12$ ($n=3922$). Five signal-sweep coordinate cells have MC intervals for BLA rejection above 5\%; the largest sine signal gives $x^2$ rejection 12.18\% [11.05,13.40] ($n=2989$), while its realized-IV check is 5.15\% [4.42,6.00]. These results show that near-nominal auxiliary-IV calibration does not ensure finite-sample selected-BLA calibration.

\paragraph{Component study.}\label{app:port-naive-ablation}
The six-cell component study is reported with its protocol and results in Appendix~\ref{app:component-comparisons}.

\paragraph{Calibration across sample sizes and noise levels.}\label{app:port-instrument}\label{ext:app:instrument}
This sample-size and noise-level grid assesses calibration and inference availability; it does not establish uniform robustness to weak identification. One grid coordinate has an MC interval for BLA rejection entirely above 5\%: $x^2$ at $M=200$, $\sigma_X=0.03$, with $6.14\%$ [5.33,7.06] ($n=2950$). Its IV rejection rate is 5.86\% [5.07,6.77], with the pointwise MC interval also above 5\%. These isolated comparisons, without adjustment across the grid, do not by themselves establish systematic miscalibration.

\paragraph{Inference availability and bootstrap diagnostics.}\label{app:port-availability}
These scalar grids and the unregularized larger-$P$ grid report no invalid coordinate, unavailable IV target, undefined selected BLA or whole-line interval. For tuning, supplied scale, misspecification, misspecified signal, identification and unregularized larger $P$, no bootstrap rank failure or dataset with fewer than 100 retained draws was recorded. This does not assert that every requested coordinate draw was retained. For multistate systems, retention diagnostics cover every equation-level fit and show no discarded draws (Appendix~\ref{app:systems-metrics}).

\section{Supercapacitor discharge analysis and noise sensitivity}
\label{full:app:real}\label{real:sec:preprocessing}\label{scope:real:estimates}
\subsection{Measurements, preprocessing and inference}
The public version-1.0 supercapacitor dataset of \citet{hanschek2026,hanschekdata} supplies all eight Vishay devices in the 50\,F group, Method 1A, Class 3 (Creative Commons Attribution 4.0; CC BY 4.0). We use time and voltage only. Measurement metadata record a discharge current of 0.60\,A. Time zero is the first supplied observation. A selected model approximates discharge rate as a function of voltage over the fitting interval. A monotone trace cannot separate voltage from unobserved history dependence, so voltage is not asserted to be a complete physical state. Nominal capacitance does not identify a known BLA coefficient.

We average ten consecutive 100\,Hz observations within $[20,180)$\,s, assigning their mean times, to obtain $M=1600$ nonoverlapping observations. The latent state sequence consists of the corresponding ten-point averages of the noise-free voltage, which need not equal that voltage evaluated at the mean times. Averaging was chosen after diagnostic inspection and the application is exploratory. For device 1, the detrended $[260,280)$\,s periodogram has a peak at about 23.25\,Hz, with 73.2\% of total power within 0.5\,Hz of that peak. Ten-point averaging attenuates this component, with amplitude response approximately 0.128 at the peak, but does not eliminate aliasing or establish independent Gaussian errors.

For fitting-window residuals $r_k$ about a 101-point local cubic smooth, let $d_k=r_{k+1}-r_k$. The working noise scale is
\[
\widehat\sigma_{\rm work}=
\frac{\operatorname{median}_k|d_k-\operatorname{median}_\ell d_\ell|}{\Phi^{-1}(0.75)\sqrt2}.
\]
The normalization corresponds to differences of independent Gaussian errors before smoothing; here smoothing, dependence and reuse of the fitting observations make it a working scale. Thus the exact-IV and BLA theorems do not establish coverage for these measured data. Let $V$ denote the processed voltage. We use scaled coordinates $u=V/(3\,\mathrm V)$ and $t_{\rm sc}=(t-20\,\mathrm s)/(100\,\mathrm s)$. Table~\ref{real:tab:protocol} summarizes the preprocessing and inference settings.

\begin{table}[!htb]

\caption{Common protocol for the averaged-data coefficient analysis. All intervals have a working-model interpretation.}\label{real:tab:protocol}
\begin{center}\small
\begin{tabularx}{\linewidth}{lX}
\toprule
Quantity & Setting\\\midrule
Processed observations & Ten-point averages assigned their mean times; 10\,Hz\\
Fitting window & $[20,180)$\,s; $M=1600$ per device\\
Library and weak equations & $(1,u,u^2,u^3)$; $K=20$ prespecified weights\\
Thinning allocation & Four equal fractions, using the estimated working scale\\
Selection & STLSQ on view 1, mandatory constant, at most 10 iterations\\
Threshold candidates & $0,0.005,0.01,0.02,0.05,0.1,0.2,0.5$\\
Pilot & Local cubic smooth of the mean of views 2--4; 105-point window\\
Bootstrap & 1,599 draws; selected support fixed; projector recomputed\\
Numerical checks & Relative rank tolerance $10^{-11}$; at least 100 usable draws\\
Coefficient intervals & Combined response and gap scale; nominal level 95\%\\
\bottomrule
\end{tabularx}
\end{center}
\end{table}

Twenty compact weak weights have centers equally spaced from 7\% to 93\% of the processed scaled interval and radius 6.5\% of its length. Equal-allocation thinning uses $\widehat\sigma_u=\widehat\sigma_{\rm work}/(3\,\mathrm V)$ and view scale $2\widehat\sigma_u$. Among supports generated on $V_1$ by the threshold grid, selection minimizes $K\log\{\operatorname{RSS}(\Sset)/K\}+|\Sset|\log K$, where $\operatorname{RSS}(\Sset)$ is the residual sum of squares of the selection-view fit. Ties are broken by smaller support, then smaller threshold. This is a selection score, not a calibrated noisy-design likelihood. All 1,599 draws were retained for every selected coordinate. The implementation uses the covariance-shape maps
\[
Q_{\rm sh}=(D_MD_M^\top)^{-1/2},\qquad B=Q_{\rm sh}H_M,\qquad C=Q_{\rm sh}D_M.
\]
Thus $CC^\top=I_K$, with $B=s_2B_M$ and $C=s_2C_M$ relative to full working-covariance whitening. The denominator cutoff is $10^{-12}$ in these coordinates, and the response scale retains $s_2=2\widehat\sigma_u$.

For degree $d=j-1$, the physical-unit coefficient and interval are obtained by multiplying dimensionless values by $(3\,\mathrm V)^{1-d}/(100\,\mathrm s)$. Its units are $\mathrm V^{1-d}\mathrm s^{-1}$. A missing term has no selected-coefficient interval and its absence is not a structural-zero test. Table~\ref{real:tab:processedcoef} illustrates the coefficient analysis for device 1. The eight device fits contain 25 selected coefficients in total. Devices are not repetitions with a known common coefficient, and the pointwise intervals are not simultaneous across terms or devices.

\begin{table}[!htb]
\setlength{\tabcolsep}{5pt}
\caption{Device 1 coefficients in scaled coordinates. Scale columns give $10^4e_j$, $10^4\widehat t_j$ and $10^4\widehat s_j$, suppressing the common selected-support subscript $\Sset$; estimates and intervals are not multiplied by $10^4$. Intervals have nominal level 95\% under the estimated working noise model, without established coverage for these observations.}
\label{real:tab:processedcoef}\label{sec4:tab:discharge}
\begin{center}\footnotesize
\begin{tabular}{@{}lrrrrr@{}}\toprule
Feature & Estimate & Response & Gap & Total & Working interval\\\midrule
$1$ & -0.520512 & 0.7403 & 0.0197 & 0.7406 & $[-0.520657,-0.520367]$\\
$u$ & 0.289392 & 1.8457 & 0.0454 & 1.8463 & $[0.289030,0.289754]$\\
$u^3$ & -0.139550 & 1.5289 & 0.0339 & 1.5292 & $[-0.139850,-0.139250]$\\
\bottomrule\end{tabular}
\end{center}
\end{table}

\subsection{Effects of covariance misspecification}\label{real:sec:covariancescope}
In this subsection, views and measurement covariances are expressed in the dimensionless state units $u=V/(3\,\mathrm V)$. With true Gaussian measurement covariance $\Sigma_{\rm obs}\in\R^{M\times M}$ and a fixed thinning covariance $\Sigma_{\rm work}\in\R^{M\times M}$, equal-allocation views satisfy
\[
\Cov(V_i,V_j)=\Sigma_{\rm obs}-\Sigma_{\rm work}\ (i\ne j),\qquad
\Var(V_i)=\Sigma_{\rm obs}+3\Sigma_{\rm work}.
\]
They are independent when $\Sigma_{\rm work}=\Sigma_{\rm obs}$. Estimating $\Sigma_{\rm work}$ from the same observations does not automatically restore the conditional Gaussian model. Let $\widetilde W^{\rm sh}$ and $\widetilde Z^{\rm sh}$ denote the corrected design and instrument formed using the covariance-shape map $B$. For a selected term $j$, set $N=\Sset\setminus\{j\}$ and define
\[
R^{\rm sh}=I_K-\widetilde W_N^{\rm sh}(\widetilde W_N^{\rm sh})^+,\qquad
z^{\rm sh}=R^{\rm sh}\widetilde Z_{:,j}^{\rm sh},\qquad
d^{\rm sh}=(z^{\rm sh})^\top\widetilde W_{:,j}^{\rm sh}.
\]
When $d^{\rm sh}\ne0$, the fixed-contrast working response variance using the response map $C$ is
\begin{equation}
(e_j^{\rm work})^2=4(z^{\rm sh})^\top C\Sigma_{\rm work}C^\top z^{\rm sh}/(d^{\rm sh})^2.
\label{real:eq:correlatedse}
\end{equation}
A plug-in covariance supplies this variance calculation, not the required view independence or BLA coverage.

\subsection{Discharge-template simulation}\label{real:sec:coverage}
For each device, let $x^{\rm ref}\in\R^{1600}$ be the fixed 101-point local cubic smooth of the processed scaled voltage $u=V/(3\,\mathrm V)$. Let $\sigma_{\rm ref}=\widehat\sigma_{\rm work}/(3\,\mathrm V)$, using the original device's working scale, and declare it to be the simulation's true marginal standard deviation. Both the reference trajectory and this noise scale are in dimensionless state units; the weak maps $B,C$ use the scaled time $t_{\rm sc}$. Each independently generated dataset has a Gaussian distribution with mean $x^{\rm ref}$ and covariance $\sigma_{\rm ref}^2I$ or $\sigma_{\rm ref}^2R_{\rm corr}$. The correlated model uses a stationary covariance from the centered residuals $r_m^{(u)}$ of the processed scaled voltage about $x^{\rm ref}$,
\[
(R_{\rm corr})_{mn}=\varrho_{|m-n|},\quad
\varrho_\ell=\left(1-\frac{\ell}{21}\right)
\frac{\sum_{m=1}^{1600-\ell}r_m^{(u)}r_{m+\ell}^{(u)}}{\sum_{m=1}^{1600}(r_m^{(u)})^2}\quad(0\le\ell\le20),
\quad \varrho_\ell=0\ (\ell>20).
\]
The taper defines a simulation model, not a verified physical sensor covariance. On a synthetic covariance example, implementation checks compared generated autocovariances with their prescribed values and the response-covariance calculation with direct matrix multiplication.

Each noise model is crossed with supplied marginal scale and scale re-estimated in each replication. All four conditions use covariance weighting based on iid working noise; supplying the marginal scale does not supply $R_{\rm corr}$. Selection, thinning, pilot fitting and 1,599 bootstrap draws are repeated. Matched scale variants share observations and seed streams, but changing the scale input affects the entire procedure. The correlated-noise supplied-scale condition is a post hoc sensitivity analysis, evaluated using the same templates, distributions and observations as its matched estimated-scale condition. Each device has 1,000 replications per noise model, giving 16,000 unique datasets and 32,000 fits across the eight devices.

With $W^{\rm ref}=B\Theta(x^{\rm ref})$, $y^{\rm ref}=Cx^{\rm ref}$, each selected support is assessed against
\[
\xi^{\rm ref}_{\Sset}=
\{(W^{\rm ref}_{\Sset})^\top W^{\rm ref}_{\Sset}\}^{-1}
(W^{\rm ref}_{\Sset})^\top y^{\rm ref}.
\]
For each device, all eight possible intercept-containing reference supports have full column rank. This is the selected projection under the common iid working metric, not a projection under the correlated generating covariance. The reference trajectory and target coefficients enter generation and assessment only.

For replication $r$, let $I_{j,r}$ denote the reported BLA interval for term $j$ in the selected support $\Shat_r$, and set
$F_r=|\Shat_r|^{-1}\sum_{j\in\Shat_r}\mathbf1_{\{\xi^{\rm ref}_{j,\Shat_r}\in I_{j,r}\}}$.
Table~\ref{real:tab:semisynthetic} reports its mean and $\operatorname{sd}(F_r)/\sqrt{1000}$, where $\operatorname{sd}$ is the sample standard deviation across replications; dependence among coefficients is thereby retained. Equal-device means use weight $1/8$ and Monte Carlo standard error (MCSE) $\{\sum_ds_{F,d}^2/1000\}^{1/2}/8$, where $s_{F,d}$ is the sample standard deviation of $F_r$ for device $d$. Paired differences use the same formula on within-replication differences. This averaged coverage is not simultaneous coverage, nor does it establish coverage for every selected support.

\begin{table}[!htb]
\setlength{\tabcolsep}{5pt}
\caption{Coverage of nominal 95\% intervals in the discharge-template experiment. All four conditions use iid working covariance weighting. Supplied scale means the true marginal standard deviation, without supplying temporal correlation. Entries are mean percentages of selected coefficient intervals containing their own finite-grid BLA targets over 1,000 replications per device, with Monte Carlo standard errors in parentheses, in percentage points. The equal-device row uses eight fixed templates. Replication-level proportions account for dependence among coefficients; this is not a simultaneous coverage summary.}
\label{real:tab:semisynthetic}
\begin{center}\small
\begin{tabular}{rcccc}\toprule
 & \multicolumn{2}{c}{Independent noise} & \multicolumn{2}{c}{Correlated noise}\\
\cmidrule(lr){2-3}\cmidrule(lr){4-5}
Device & Supplied scale & Estimated scale & Supplied scale & Estimated scale\\\midrule
1 & 94.87 (0.66) & 94.80 (0.67) & 79.90 (1.20) & 61.10 (1.47)\\
2 & 95.33 (0.65) & 95.33 (0.64) & 83.40 (1.12) & 67.43 (1.40)\\
3 & 95.67 (0.60) & 95.67 (0.60) & 82.37 (1.16) & 67.87 (1.41)\\
4 & 94.57 (0.68) & 94.57 (0.68) & 80.63 (1.19) & 62.57 (1.46)\\
5 & 95.85 (0.60) & 95.83 (0.61) & 84.30 (1.11) & 74.20 (1.34)\\
6 & 95.17 (0.65) & 95.23 (0.65) & 85.20 (1.09) & 70.70 (1.39)\\
7 & 95.50 (0.63) & 95.40 (0.64) & 80.57 (1.21) & 61.93 (1.46)\\
8 & 94.87 (0.65) & 94.77 (0.65) & 78.50 (1.24) & 62.30 (1.47)\\
\midrule
Equal-device & 95.23 (0.23) & 95.20 (0.23) & 81.86 (0.41) & 66.01 (0.50)\\
\bottomrule\end{tabular}
\end{center}
\end{table}

Independent-noise coverage estimates are compatible with the nominal level: across all 25 observed support--coordinate groups, they range from 94.2\% to 96.1\% with supplied scale and 94.2\% to 96.0\% with estimated scale, with every pointwise Wilson interval containing 95\%. Under correlated noise, all 25 upper endpoints are below 95\% for both scale inputs. The paired supplied-minus-estimated coverage difference under correlated noise is 15.85 percentage points, with a 95\% Monte Carlo interval of [15.13,16.56]. Estimated-to-true scale ratios range from 0.59 to 0.72 under correlation, versus 0.999--1.002 under independence. Supplied marginal scale leaves cross-view covariance $\sigma_{\rm ref}^2(R_{\rm corr}-I)$. All intervals were available. Within each device, selection was constant across replications and conditions; coverage therefore assesses stable selection.

\section{Computational resources and runtime}
\label{app:runtime}
We measured the computational cost of the primary simulation implementation on a MacBook Pro with an Apple M4 Pro processor (14 CPU cores), 48\,GiB of memory and macOS~27.0. The benchmark used Python~3.12.14, NumPy~2.5.3, SciPy~1.18.1 and threadpoolctl~3.7.0, with one worker and numerical-library thread limits set to one. Computation used the CPU only.

We generated 30 independent datasets for each of the five settings in Table~\ref{tab:runtime}, using the primary protocol in Appendix~\ref{sec4:appendix}: $P=4$, $K=12$, $\sigma=0.06$, equal thinning allocations and the primary study's bootstrap budgets. The benchmark assessed inference for the quadratic term and used a separate random seed, 2026100401. Before timing each setting, warm-up datasets were generated until one selected case returned a valid bootstrap; all warm-up attempts were excluded.

Pipeline time includes thinning, construction of the corrected selection design, term selection and, when the quadratic term is selected, the inference routine. The latter includes the trajectory pilot, bootstrap gap scales for the selected support and the auxiliary Naive estimate computed by the primary implementation. Timings exclude deterministic problem construction, observation generation, latent-target assessment, final test and confidence-interval arithmetic, imports and file I/O. Inference time is part of pipeline time; the two columns are not additive. The code supplement provides a standalone benchmark that generates hardware and software records and individual timing records.

\begin{table}[!htb]

\caption{Wall-clock runtime in seconds, reported as median [25th, 75th percentile]. Each setting uses 30 generated datasets; $N_A$ counts those selecting the quadratic term. Pipeline summaries include all datasets, including those that stop after selection; inference summaries use only the $N_A$ selected datasets. These early exits can produce small pipeline lower quartiles, and the different populations allow a pipeline median below the inference median.}
\label{tab:runtime}
\begin{center}\small
\begin{tabular}{rrrrll}
\toprule
$M$ & $\gamma$ & $\Bboot$ & $N_A$ & Pipeline & Inference \\
\midrule
200 & 0 & 799 & 21 & 0.0082 [0.0001, 0.0090] & 0.0088 [0.0076, 0.0091] \\
800 & 0 & 1199 & 20 & 0.0382 [0.0001, 0.0396] & 0.0391 [0.0382, 0.0398] \\
1600 & 0 & 1599 & 24 & 0.1010 [0.0987, 0.1038] & 0.1013 [0.1000, 0.1045] \\
3200 & 0 & 2399 & 19 & 0.2673 [0.0002, 0.2708] & 0.2701 [0.2680, 0.2709] \\
800 & 0.16 & 1199 & 30 & 0.0379 [0.0374, 0.0382] & 0.0378 [0.0373, 0.0381] \\
\bottomrule
\end{tabular}
\end{center}
\end{table}

All 114 selected cases returned valid inference and retained every requested bootstrap draw; the remaining 36 datasets did not select the quadratic term. Across the four reference-trajectory settings, median inference time increased from 0.0088 to 0.2701 seconds as both $M$ and $\Bboot$ increased. These measurements describe the empirical cost of the primary focal-coefficient implementation, rather than a complexity bound or the runtime of the complete experimental study.

The full simulation study used a Linux x86\_64 environment with eight workers, Python~3.12.14, NumPy~2.3.5 and SciPy~1.17.0. Table~\ref{tab:runtime} reports only the MacBook benchmark.

%% file: ai_use_statement.tex
\subsection*{AI Use Statement}

During manuscript preparation, generative AI tools assisted with paper organization, suggestions of potentially relevant references, language editing, and discussion of schematic illustration design. Generative AI tools also assisted with drafting and revising technical passages, reviewing mathematical arguments, providing feedback on methodology and experimental analyses, interpreting results, reviewing and refining existing code, and developing numerical verification scripts. The authors checked all cited references against the original sources and independently obtained the bibliographic information.

The authors reviewed and verified all AI-assisted work and take full responsibility for the final manuscript, including its text, claims, code, analyses and proofs.

%% file: manuscript.bbl
\begin{thebibliography}{39}
\providecommand{\natexlab}[1]{#1}
\providecommand{\url}[1]{\texttt{#1}}
\expandafter\ifx\csname urlstyle\endcsname\relax
  \providecommand{\doi}[1]{doi: #1}\else
  \providecommand{\doi}{doi: \begingroup \urlstyle{rm}\Url}\fi

\bibitem[Berk et~al.(2013)Berk, Brown, Buja, Zhang, and Zhao]{berk2013}
Richard Berk, Lawrence Brown, Andreas Buja, Kai Zhang, and Linda Zhao.
\newblock Valid post-selection inference.
\newblock \emph{The Annals of Statistics}, 41\penalty0 (2):\penalty0 802--837,
  2013.
\newblock \doi{10.1214/12-AOS1077}.

\bibitem[Bortz et~al.(2023)Bortz, Messenger, and Dukic]{wendy2023}
David~M. Bortz, Daniel~A. Messenger, and Vanja Dukic.
\newblock Direct estimation of parameters in {ODE} models using {WENDy}:
  Weak-form estimation of nonlinear dynamics.
\newblock \emph{Bulletin of Mathematical Biology}, 85\penalty0 (11):\penalty0
  110, 2023.
\newblock \doi{10.1007/s11538-023-01208-6}.

\bibitem[Brunel et~al.(2014)Brunel, Clairon, and {d'Alch{\'e}-Buc}]{brunel2014}
Nicolas J.-B. Brunel, Quentin Clairon, and Florence {d'Alch{\'e}-Buc}.
\newblock Parametric estimation of ordinary differential equations with
  orthogonality conditions.
\newblock \emph{Journal of the American Statistical Association}, 109\penalty0
  (505):\penalty0 173--185, 2014.
\newblock \doi{10.1080/01621459.2013.841583}.
\newblock URL \url{https://arxiv.org/abs/1410.7566}.

\bibitem[Brunton et~al.(2016)Brunton, Proctor, and Kutz]{brunton2016}
Steven~L. Brunton, Joshua~L. Proctor, and J.~Nathan Kutz.
\newblock Discovering governing equations from data by sparse identification of
  nonlinear dynamical systems.
\newblock \emph{Proceedings of the National Academy of Sciences}, 113\penalty0
  (15):\penalty0 3932--3937, 2016.
\newblock \doi{10.1073/pnas.1517384113}.

\bibitem[Chawla et~al.(2026)Chawla, Bortz, and Dukic]{chawla2026}
Abhi Chawla, David~M. Bortz, and Vanja Dukic.
\newblock Bias and coverage properties of the {WENDy-IRLS} algorithm.
\newblock In Bharath Sriraman (ed.), \emph{Handbook of Visual, Experimental and
  Computational Mathematics}, pp.\  1639--1746. Springer, Cham, 2026.
\newblock \doi{10.1007/978-3-032-16368-4_86}.
\newblock URL \url{https://link.springer.com/rwe/10.1007/978-3-032-16368-4_86}.

\bibitem[Chen \& Witten(2023)Chen and Witten]{chen2023clustering}
Yiqun~T. Chen and Daniela~M. Witten.
\newblock Selective inference for k-means clustering.
\newblock \emph{Journal of Machine Learning Research}, 24\penalty0
  (152):\penalty0 1--41, 2023.
\newblock URL \url{https://jmlr.org/papers/v24/22-0371.html}.

\bibitem[Davison \& Hinkley(1997)Davison and Hinkley]{davison1997}
A.~C. Davison and D.~V. Hinkley.
\newblock \emph{Bootstrap Methods and their Application}.
\newblock Cambridge University Press, 1997.
\newblock \doi{10.1017/cbo9780511802843}.

\bibitem[Duy \& Takeuchi(2022)Duy and Takeuchi]{le2022more}
Vo~Nguyen~Le Duy and Ichiro Takeuchi.
\newblock More powerful conditional selective inference for generalized lasso
  by parametric programming.
\newblock \emph{Journal of Machine Learning Research}, 23\penalty0
  (300):\penalty0 1--37, 2022.
\newblock URL \url{https://jmlr.org/papers/v23/21-0494.html}.

\bibitem[Fasel(2025)]{fasel2025}
Urban Fasel.
\newblock Sparse identification of nonlinear dynamics with conformal
  prediction.
\newblock arXiv preprint arXiv:2507.11739, 2025.
\newblock URL \url{https://arxiv.org/abs/2507.11739}.

\bibitem[Fasel et~al.(2022)Fasel, Kutz, Brunton, and Brunton]{fasel2022}
Urban Fasel, J.~Nathan Kutz, Bingni~W. Brunton, and Steven~L. Brunton.
\newblock Ensemble-{SINDy}: Robust sparse model discovery in the low-data,
  high-noise limit, with active learning and control.
\newblock \emph{Proceedings of the Royal Society A: Mathematical, Physical and
  Engineering Sciences}, 478\penalty0 (2260):\penalty0 20210904, 2022.
\newblock \doi{10.1098/rspa.2021.0904}.

\bibitem[Fuller(1987)]{fuller1987}
Wayne~A. Fuller.
\newblock \emph{Measurement Error Models}.
\newblock John Wiley \& Sons, 1987.
\newblock \doi{10.1002/9780470316665}.

\bibitem[Gao et~al.(2023)Gao, Fasel, Brunton, and Kutz]{gao2023}
L.~Mars Gao, Urban Fasel, Steven~L. Brunton, and J.~Nathan Kutz.
\newblock Convergence of uncertainty estimates in ensemble and {Bayesian}
  sparse model discovery.
\newblock arXiv preprint arXiv:2301.12649, 2023.
\newblock URL \url{https://arxiv.org/abs/2301.12649}.

\bibitem[Garc{\'i}a~Rasines \& Young(2023)Garc{\'i}a~Rasines and
  Young]{rasines2023}
D.~Garc{\'i}a~Rasines and G.~A. Young.
\newblock Splitting strategies for post-selection inference.
\newblock \emph{Biometrika}, 110\penalty0 (3):\penalty0 597--614, 2023.
\newblock \doi{10.1093/biomet/asac070}.

\bibitem[Golub \& Pereyra(1973)Golub and Pereyra]{golub1973}
G.~H. Golub and V.~Pereyra.
\newblock The differentiation of pseudo-inverses and nonlinear least squares
  problems whose variables separate.
\newblock \emph{SIAM Journal on Numerical Analysis}, 10\penalty0 (2):\penalty0
  413--432, 1973.
\newblock \doi{10.1137/0710036}.

\bibitem[Hanschek et~al.(2026{\natexlab{a}})Hanschek, Dusini, and
  Grbovi{\'c}]{hanschek2026}
A.~Hanschek, S.~Dusini, and P.~Grbovi{\'c}.
\newblock A standard-aligned framework for voltage-dependent supercapacitor
  characterization.
\newblock \emph{IEEE Open Journal of Industry Applications}, 7:\penalty0
  485--499, 2026{\natexlab{a}}.
\newblock \doi{10.1109/OJIA.2026.3675690}.

\bibitem[Hanschek et~al.(2026{\natexlab{b}})Hanschek, Dusini, Tatschl, Horn,
  and Grbovic]{hanschekdata}
A.~J. Hanschek, S.~Dusini, G.~Tatschl, A.~Horn, and P.~Grbovic.
\newblock Supercapacitor discharge measurements {25F} and {50F} {DUT}-sets.
\newblock Zenodo, 2026{\natexlab{b}}.
\newblock URL \url{https://doi.org/10.5281/zenodo.19221698}.
\newblock Version 1.0. Dataset.

\bibitem[Hirsh et~al.(2022)Hirsh, Barajas-Solano, and Kutz]{hirsh2022}
Seth~M. Hirsh, David~A. Barajas-Solano, and J.~Nathan Kutz.
\newblock Sparsifying priors for {Bayesian} uncertainty quantification in model
  discovery.
\newblock \emph{Royal Society Open Science}, 9\penalty0 (2):\penalty0 211823,
  2022.
\newblock \doi{10.1098/rsos.211823}.

\bibitem[Holm(1979)]{holm1979}
Sture Holm.
\newblock A simple sequentially rejective multiple test procedure.
\newblock \emph{Scandinavian Journal of Statistics}, 6\penalty0 (2):\penalty0
  65--70, 1979.
\newblock URL \url{https://www.jstor.org/stable/4615733}.

\bibitem[Kriegeskorte et~al.(2009)Kriegeskorte, Simmons, Bellgowan, and
  Baker]{kriegeskorte2009}
Nikolaus Kriegeskorte, W.~Kyle Simmons, Patrick S.~F. Bellgowan, and Chris~I.
  Baker.
\newblock Circular analysis in systems neuroscience: the dangers of double
  dipping.
\newblock \emph{Nature Neuroscience}, 12:\penalty0 535--540, 2009.
\newblock \doi{10.1038/nn.2303}.

\bibitem[Lee et~al.(2016)Lee, Sun, Sun, and Taylor]{lee2016}
Jason~D. Lee, Dennis~L. Sun, Yuekai Sun, and Jonathan~E. Taylor.
\newblock Exact post-selection inference, with application to the lasso.
\newblock \emph{The Annals of Statistics}, 44\penalty0 (3):\penalty0 907--927,
  2016.
\newblock \doi{10.1214/15-AOS1371}.

\bibitem[Leiner et~al.(2025)Leiner, Duan, Wasserman, and Ramdas]{leiner2025}
James Leiner, Boyan Duan, Larry Wasserman, and Aaditya Ramdas.
\newblock Data fission: Splitting a single data point.
\newblock \emph{Journal of the American Statistical Association}, 120\penalty0
  (549):\penalty0 135--146, 2025.
\newblock \doi{10.1080/01621459.2023.2270748}.

\bibitem[Lovell(1963)]{lovell1963}
Michael~C. Lovell.
\newblock Seasonal adjustment of economic time series and multiple regression
  analysis.
\newblock \emph{Journal of the American Statistical Association}, 58\penalty0
  (304):\penalty0 993--1010, 1963.
\newblock \doi{10.1080/01621459.1963.10480682}.

\bibitem[Matsukawa et~al.(2025)Matsukawa, Shiraishi, Nishino, Katsuoka, and
  Takeuchi]{matsukawa2025}
Tatsuya Matsukawa, Tomohiro Shiraishi, Shuichi Nishino, Teruyuki Katsuoka, and
  Ichiro Takeuchi.
\newblock Statistical test for auto feature engineering by selective inference.
\newblock In \emph{Proceedings of the 28th International Conference on
  Artificial Intelligence and Statistics}, volume 258 of \emph{Proceedings of
  Machine Learning Research}, pp.\  3214--3222. PMLR, 2025.
\newblock URL \url{https://proceedings.mlr.press/v258/matsukawa25a.html}.

\bibitem[Messenger \& Bortz(2021{\natexlab{a}})Messenger and
  Bortz]{messenger2021}
Daniel~A. Messenger and David~M. Bortz.
\newblock Weak {SINDy}: {Galerkin}-based data-driven model selection.
\newblock \emph{Multiscale Modeling \& Simulation}, 19\penalty0 (3):\penalty0
  1474--1497, 2021{\natexlab{a}}.
\newblock \doi{10.1137/20M1343166}.

\bibitem[Messenger \& Bortz(2021{\natexlab{b}})Messenger and
  Bortz]{messengerpde2021}
Daniel~A. Messenger and David~M. Bortz.
\newblock Weak {SINDy} for partial differential equations.
\newblock \emph{Journal of Computational Physics}, 443:\penalty0 110525,
  2021{\natexlab{b}}.
\newblock \doi{10.1016/j.jcp.2021.110525}.

\bibitem[Messenger \& Bortz(2025)Messenger and Bortz]{consistency2025}
Daniel~A. Messenger and David~M. Bortz.
\newblock Asymptotic consistency of the {WSINDy} algorithm in the limit of
  continuum data.
\newblock \emph{IMA Journal of Numerical Analysis}, 45\penalty0 (6):\penalty0
  3264--3312, 2025.
\newblock \doi{10.1093/imanum/drae086}.

\bibitem[Morris et~al.(2019)Morris, White, and Crowther]{morris2019}
Tim~P. Morris, Ian~R. White, and Michael~J. Crowther.
\newblock Using simulation studies to evaluate statistical methods.
\newblock \emph{Statistics in Medicine}, 38\penalty0 (11):\penalty0 2074--2102,
  2019.
\newblock \doi{10.1002/sim.8086}.

\bibitem[Neufeld et~al.(2024)Neufeld, Dharamshi, Gao, and Witten]{neufeld2024}
Anna Neufeld, Ameer Dharamshi, Lucy~L. Gao, and Daniela Witten.
\newblock Data thinning for convolution-closed distributions.
\newblock \emph{Journal of Machine Learning Research}, 25\penalty0
  (57):\penalty0 1--35, 2024.
\newblock URL \url{https://jmlr.org/papers/v25/23-0446.html}.

\bibitem[Rinaldo et~al.(2019)Rinaldo, Wasserman, and G'Sell]{rinaldo2019}
Alessandro Rinaldo, Larry Wasserman, and Max G'Sell.
\newblock Bootstrapping and sample splitting for high-dimensional,
  assumption-lean inference.
\newblock \emph{The Annals of Statistics}, 47\penalty0 (6):\penalty0
  3438--3469, 2019.
\newblock \doi{10.1214/18-aos1784}.

\bibitem[Rudy et~al.(2017)Rudy, Brunton, Proctor, and Kutz]{rudy2017}
Samuel~H. Rudy, Steven~L. Brunton, Joshua~L. Proctor, and J.~Nathan Kutz.
\newblock Data-driven discovery of partial differential equations.
\newblock \emph{Science Advances}, 3\penalty0 (4):\penalty0 e1602614, 2017.
\newblock \doi{10.1126/sciadv.1602614}.

\bibitem[Schaeffer \& McCalla(2017)Schaeffer and McCalla]{schaeffer2017}
Hayden Schaeffer and Scott~G. McCalla.
\newblock Sparse model selection via integral terms.
\newblock \emph{Physical Review E}, 96\penalty0 (2):\penalty0 023302, 2017.
\newblock \doi{10.1103/PhysRevE.96.023302}.

\bibitem[Shiraishi et~al.(2026)Shiraishi, Miwa, Katsuoka, Duy, Nishino, Taji,
  and Takeuchi]{shiraishi2026statistical}
Tomohiro Shiraishi, Daiki Miwa, Teruyuki Katsuoka, Vo~Nguyen~Le Duy, Shuichi
  Nishino, Kouichi Taji, and Ichiro Takeuchi.
\newblock Statistical test for attention in transformers for images and time
  series.
\newblock \emph{Journal of Machine Learning Research}, 27\penalty0
  (119):\penalty0 1--43, 2026.
\newblock URL \url{https://jmlr.org/papers/v27/25-1631.html}.

\bibitem[Stefanski(1989)]{stefanski1989}
Leonard~A. Stefanski.
\newblock Unbiased estimation of a nonlinear function of a normal mean with
  application to measurement-error models.
\newblock \emph{Communications in Statistics - Theory and Methods}, 18\penalty0
  (12):\penalty0 4335--4358, 1989.
\newblock \doi{10.1080/03610928908830159}.

\bibitem[Taylor \& Tibshirani(2015)Taylor and Tibshirani]{taylor2015}
Jonathan Taylor and Robert~J. Tibshirani.
\newblock Statistical learning and selective inference.
\newblock \emph{Proceedings of the National Academy of Sciences}, 112\penalty0
  (25):\penalty0 7629--7634, 2015.
\newblock \doi{10.1073/pnas.1507583112}.

\bibitem[Tian \& Taylor(2018)Tian and Taylor]{tian2018}
Xiaoying Tian and Jonathan Taylor.
\newblock Selective inference with a randomized response.
\newblock \emph{The Annals of Statistics}, 46\penalty0 (2):\penalty0 679--710,
  2018.
\newblock \doi{10.1214/17-aos1564}.

\bibitem[Tibshirani et~al.(2018)Tibshirani, Rinaldo, Tibshirani, and
  Wasserman]{tibshirani2018}
Ryan~J. Tibshirani, Alessandro Rinaldo, Rob Tibshirani, and Larry Wasserman.
\newblock Uniform asymptotic inference and the bootstrap after model selection.
\newblock \emph{The Annals of Statistics}, 46\penalty0 (3):\penalty0
  1255--1287, 2018.
\newblock \doi{10.1214/17-aos1584}.

\bibitem[van~der Vaart(1998)]{vandervaart1998}
A.~W. van~der Vaart.
\newblock \emph{Asymptotic Statistics}.
\newblock Cambridge University Press, 1998.
\newblock \doi{10.1017/cbo9780511802256}.

\bibitem[Wasserman \& Roeder(2009)Wasserman and Roeder]{wasserman2009}
Larry Wasserman and Kathryn Roeder.
\newblock High-dimensional variable selection.
\newblock \emph{The Annals of Statistics}, 37\penalty0 (5A):\penalty0
  2178--2201, 2009.
\newblock \doi{10.1214/08-aos646}.

\bibitem[Wilson(1927)]{wilson1927}
Edwin~B. Wilson.
\newblock Probable inference, the law of succession, and statistical inference.
\newblock \emph{Journal of the American Statistical Association}, 22\penalty0
  (158):\penalty0 209--212, 1927.
\newblock \doi{10.1080/01621459.1927.10502953}.

\end{thebibliography}
